\documentclass{article}
\usepackage[preprint]{tmlr}

\usepackage[utf8]{inputenc}
\usepackage[T1]{fontenc}
\usepackage{amsmath,amssymb,amsthm,amsfonts}
\usepackage{mathtools}
\usepackage{booktabs}
\usepackage{graphicx}
\usepackage{xurl}
\usepackage{hyperref}
\usepackage{xcolor}
\usepackage[expansion=false]{microtype}
\usepackage{subcaption}
\usepackage{algorithm}
\usepackage{algpseudocode}
\usepackage{enumitem}
\usepackage{placeins}

\hypersetup{colorlinks=true, linkcolor=blue!65!black,
            citecolor=blue!65!black, urlcolor=blue!65!black}

\newtheorem{theorem}{Theorem}
\newtheorem{lemma}[theorem]{Lemma}
\newtheorem{proposition}[theorem]{Proposition}
\newtheorem{corollary}[theorem]{Corollary}
\theoremstyle{definition}
\newtheorem{definition}[theorem]{Definition}
\newtheorem{remark}[theorem]{Remark}
\newtheorem{problem}[theorem]{Problem}
\newtheorem{conjecture}[theorem]{Conjecture}

\newcommand{\Otilde}{\widetilde{O}}
\newcommand{\Omegatilde}{\widetilde{\Omega}}
\newcommand{\Thetatilde}{\widetilde{\Theta}}
\newcommand{\R}{\mathbb{R}}
\newcommand{\E}{\mathbb{E}}
\newcommand{\norm}[1]{\left\lVert #1 \right\rVert}
\newcommand{\nucnorm}[1]{\left\lVert #1 \right\rVert_*}
\newcommand{\FrecAS}{F_{\mathrm{rec}}^{\mathrm{AS}}}
\newcommand{\ecomp}{\varepsilon_{\mathrm{comp}}}
\newcommand{\ecorr}{\varepsilon_{\mathrm{corr,max}}}
\newcommand{\code}[1]{\texttt{\small #1}}
\newcommand{\Rgeom}{R_{\mathrm{geom}}}
\newcommand{\Rreadout}{R_{\mathrm{readout}}}

\newcommand{\DurEsixMin}{29.9}
\newcommand{\DurTotalMin}{343.8}
\newcommand{\EnvPython}{3.12.10}
\newcommand{\EnvNumpy}{2.4.4}
\newcommand{\EnvScipy}{1.17.1}
\newcommand{\EnvTorch}{2.13.0+cpu}
\newcommand{\EnvCpuCount}{14}
\newcommand{\EoneMeasurements}{720}
\newcommand{\EoneViolations}{0}
\newcommand{\EoneConfigs}{16}
\newcommand{\EoneTrials}{15}
\newcommand{\EoneTiedRatioMin}{1.063}
\newcommand{\EoneTiedRatioMax}{1.999}
\newcommand{\EoneTopD}{128}
\newcommand{\EoneTopF}{2\,048}
\newcommand{\EoneTopTiedRatio}{1.0662}
\newcommand{\EoneTopPinvRatio}{1.0010}
\newcommand{\EoneGaussianRatioMax}{3.4e+08}
\newcommand{\EoneTightFrameDev}{4e-16}
\newcommand{\EoneTightFrameConfigs}{16}
\newcommand{\EtwoTrials}{200}
\newcommand{\EtwoThetaCount}{3}
\newcommand{\EtwoSixFourSNF}{1}
\newcommand{\EtwoTwoFiveSixSNF}{2}
\newcommand{\EtwoTwoFiveSixLinRms}{0.0884}
\newcommand{\EtwoTwoFiveSixLinRmsFloor}{0.0882}
\newcommand{\EtwoTwoFiveSixLinRmsOverRoot}{1.41}
\newcommand{\EtwoThetaZeroPFourSNF}{1}
\newcommand{\EtwoThetaZeroPFiveSNF}{2}
\newcommand{\EtwoThetaZeroPSixSNF}{4}
\newcommand{\EtwoThetaDmax}{256}
\newcommand{\EthreeConfigs}{24}
\newcommand{\EthreeViolations}{0}
\newcommand{\EthreeRatioMin}{0.9735}
\newcommand{\EthreeRatioMax}{1.0088}
\newcommand{\EthreeMcMaxRelDiff}{0.0013}
\newcommand{\EthreeMcTrials}{20\,000}
\newcommand{\EthreeUniformOverFloorMin}{1.1395}
\newcommand{\EthreeUniformOverFloorMax}{2.2165}
\newcommand{\EfourRuns}{168}
\newcommand{\EfourSteps}{20\,000}
\newcommand{\EfourConfigs}{12}
\newcommand{\EfourFreeRuns}{60}
\newcommand{\EfourFreeRatioMean}{1.0006}
\newcommand{\EfourFreeRatioMax}{1.0176}
\newcommand{\EfourFreeBelowFloor}{0}
\newcommand{\EfourTiedRuns}{36}
\newcommand{\EfourTiedRatioMean}{1.0000}
\newcommand{\EfourTiedRatioMax}{1.0000}
\newcommand{\EfourTiedBelowFloor}{0}
\newcommand{\EfourSoftmaxRuns}{36}
\newcommand{\EfourSoftmaxRatioMean}{1.5050}
\newcommand{\EfourSoftmaxRatioMax}{3.5230}
\newcommand{\EfourSoftmaxBelowFloor}{0}
\newcommand{\EfourTiedFrameResidualMean}{0.0013}
\newcommand{\EfourTiedFrameResidualMax}{0.0016}
\newcommand{\EfourSoftmaxRatioMaxMean}{3.125}
\newcommand{\EfourUncalRawRatioMean}{6.61e-09}
\newcommand{\EfourUncalGainMin}{0}
\newcommand{\EfourUncalRuns}{36}
\newcommand{\EfourUncalDegenerate}{25}
\newcommand{\EfourUncalBelowFloor}{0}
\newcommand{\EfiveD}{50}
\newcommand{\EfiveF}{100}
\newcommand{\EfiveSeeds}{5}
\newcommand{\EfiveSteps}{50\,000}
\newcommand{\EfiveEvalTrials}{400}
\newcommand{\EfivePrimaryP}{0.010}
\newcommand{\EfiveWelchFloor}{0.01010}
\newcommand{\EfiveLfourRatioLsMean}{1.0216}
\newcommand{\EfiveLfourRatioLsMin}{1.0206}
\newcommand{\EfiveLfourRatioLsMax}{1.0223}
\newcommand{\EfiveLfourRatioPinvMean}{1.0006}
\newcommand{\EfiveLfourRatioPinvMin}{1.0005}
\newcommand{\EfiveLfourRatioPinvMax}{1.0006}
\newcommand{\EfiveLfourRatioWoutMean}{1.0260}
\newcommand{\EfiveLfourSmodelMax}{3}
\newcommand{\EfiveLfourNegPre}{48.6}
\newcommand{\EfiveLfourAgreeLinear}{21.8}
\newcommand{\EfiveLfourWoutPinvDist}{0.726}
\newcommand{\EfiveLfourMse}{1.19e-03}
\newcommand{\EfiveLtwoRatioLsMean}{28.8135}
\newcommand{\EfiveLtwoNegPre}{4.6}
\newcommand{\EfiveLtwoAgreeLinear}{99.1}
\newcommand{\EfiveLtwoWoutPinvDist}{0.011}
\newcommand{\EfiveLtwoMse}{8.30e-04}
\newcommand{\EfiveRandRatioLsMean}{1.0550}
\newcommand{\EfiveRandRatioPinvMean}{1.0170}
\newcommand{\EfiveLfourPmodelThree}{1.000}
\newcommand{\EfiveLfourRmsThree}{0.1733}
\newcommand{\EfiveLfourFloorThree}{0.1714}
\newcommand{\EfiveLfourPmodelFour}{0.844}
\newcommand{\EfiveLfourPlinFour}{0.985}
\newcommand{\EfiveLfourPmodelFive}{0.745}
\newcommand{\EfiveLfourPlinFive}{0.294}
\newcommand{\EfiveLfourPmodelSix}{0.426}
\newcommand{\EfiveLfourPlinSix}{0.136}
\newcommand{\EfiveLfourSlinBestMax}{4}
\newcommand{\EfiveLfourSlinAheadSecondary}{3}
\newcommand{\EfiveLtwoPmodelOne}{0.491}
\newcommand{\EfiveRandPlinThree}{0.698}
\newcommand{\EfiveSecondaryP}{0.020}
\newcommand{\EfiveLfourSecRatioLsMean}{1.0217}
\newcommand{\EfiveLfourSecAgreeLinear}{22.3}
\newcommand{\EfiveLtwoSecRatioLsMean}{45.1973}
\newcommand{\EfiveLtwoSecAgreeLinear}{99.7}
\newcommand{\EfiveMwPLfourLtwo}{0.0079}
\newcommand{\EsixDmax}{1\,024}
\newcommand{\EsixFmax}{1\,048\,576}
\newcommand{\EsixOneTwoEightSNF}{1}
\newcommand{\EsixOneZeroTwoFourSNF}{8}
\newcommand{\EsixOneZeroTwoFourTrials}{50}
\newcommand{\EsixOneZeroTwoFourDurationMin}{21.3}
\newcommand{\EsixFitPoints}{4}
\newcommand{\EsixInterpOneTwoEight}{1.18}
\newcommand{\EsixInterpRsqLog}{0.9882}
\newcommand{\EsixInterpRsqLin}{0.9636}
\newcommand{\EsixInterpRsqLogSq}{0.9630}
\newcommand{\EsixInterpC}{0.0563}
\newcommand{\EsixGrowthGrid}{8.00}
\newcommand{\EsixGrowthInterp}{6.89}
\newcommand{\EsixGrowthPredicted}{5.60}
\newcommand{\EsixTopBestWilsonLow}{0.929}
\newcommand{\EsixEnergyTrialsTop}{10}
\newcommand{\GoneLfourRatio}{1.0006}
\newcommand{\GoneLfourRgeom}{1.0006}
\newcommand{\GoneLfourRreadoutPinv}{1.000000}
\newcommand{\GoneLfourRreadoutWout}{1.0254}
\newcommand{\GoneLfourLeverageMin}{0.4801}
\newcommand{\GoneLfourLeverageMax}{0.5192}
\newcommand{\GoneLfourLeverageCv}{0.0169}
\newcommand{\GoneLtwoRatio}{19.0052}
\newcommand{\GoneLtwoRgeom}{19.0052}
\newcommand{\GoneLtwoRreadoutPinv}{1.000000}
\newcommand{\GoneLtwoRreadoutWout}{1.0053}
\newcommand{\GoneLtwoLeverageMin}{0.0092}
\newcommand{\GoneLtwoLeverageMax}{0.9942}
\newcommand{\GoneLtwoLeverageCv}{0.8673}
\newcommand{\GoneRandRatio}{1.0170}
\newcommand{\GoneRandRgeom}{1.0170}
\newcommand{\GoneRandRreadoutPinv}{1.000000}
\newcommand{\GoneRandRreadoutWout}{1.0000}
\newcommand{\GoneRandLeverageMin}{0.3927}
\newcommand{\GoneRandLeverageMax}{0.6263}
\newcommand{\GoneRandLeverageCv}{0.0918}
\newcommand{\GoneLtwoRgeomDenseP}{29.1032}
\newcommand{\SALfourFiftyRgeom}{1.0006}
\newcommand{\SALtwoFiftyRgeom}{18.5635}
\newcommand{\SARandFiftyRgeom}{1.0187}
\newcommand{\SALfourHundredRgeom}{1.0003}
\newcommand{\SALtwoHundredRgeom}{24.9782}
\newcommand{\SARandHundredRgeom}{1.0099}
\newcommand{\SALfourTwoHundredRgeom}{1.0002}
\newcommand{\SALtwoTwoHundredRgeom}{29.8734}
\newcommand{\SARandTwoHundredRgeom}{1.0049}
\newcommand{\SALfourFiftyRreadout}{1.0263}
\newcommand{\PrimLfourHundredPostDiff}{0.35}
\newcommand{\PrimLfourTwoHundredPostDiff}{0.35}
\newcommand{\PrimModels}{120}
\newcommand{\PrimTies}{103}
\newcommand{\PrimNetWins}{15}
\newcommand{\PrimProbeWins}{2}
\newcommand{\PrimLtwoJointFailures}{60}
\newcommand{\PrimLfourModels}{60}
\newcommand{\PrimLfourTies}{43}
\newcommand{\EsevenSeeds}{20}
\newcommand{\AucFiftyDiff}{0.0053}
\newcommand{\AucFiftyCiLow}{0.0038}
\newcommand{\AucFiftyCiHigh}{0.0068}
\newcommand{\AucFiftyNetWins}{19}
\newcommand{\AucHundredDiff}{0.0210}
\newcommand{\AucHundredCiLow}{0.0193}
\newcommand{\AucHundredCiHigh}{0.0226}
\newcommand{\AucHundredNetWins}{20}
\newcommand{\AucTwoHundredDiff}{0.0379}
\newcommand{\AucTwoHundredCiLow}{0.0365}
\newcommand{\AucTwoHundredCiHigh}{0.0394}
\newcommand{\AucTwoHundredNetWins}{20}
\newcommand{\AucFourHundredDiff}{0.0699}
\newcommand{\AucFourHundredCiLow}{0.0670}
\newcommand{\AucFourHundredCiHigh}{0.0726}
\newcommand{\AucFourHundredNetWins}{10}
\newcommand{\PrimRandFiftyDiff}{-1.00}
\newcommand{\PrimRandFiftyAucDiff}{-0.2122}
\newcommand{\PrimRandHundredDiff}{-2.00}
\newcommand{\PrimRandTwoHundredDiff}{-2.30}
\newcommand{\PrimRandTwoHundredProbeWins}{20}
\newcommand{\PrimRandFourHundredDiff}{-2.60}
\newcommand{\PrimRandFourHundredProbeWins}{10}
\newcommand{\PrimRandFourHundredAucDiff}{-0.1782}
\newcommand{\SbSeeds}{10}
\newcommand{\SbLoadHigh}{4}
\newcommand{\SbLoadMatched}{2}
\newcommand{\SbLfourHundredDiff}{0.70}
\newcommand{\SbLfourHundredCiLow}{0.40}
\newcommand{\SbLfourHundredCiHigh}{1.00}
\newcommand{\SbLfourHundredNetWins}{7}
\newcommand{\SbLfourTwoHundredDiff}{1.00}
\newcommand{\SbLfourTwoHundredCiLow}{0.70}
\newcommand{\SbLfourTwoHundredCiHigh}{1.30}
\newcommand{\SbLfourTwoHundredNetWins}{9}
\newcommand{\SbLoadLfourFiftyDiff}{0.00}
\newcommand{\SbLoadLfourTwoHundredDiff}{0.40}
\newcommand{\SbRandHundredDiff}{-1.80}
\newcommand{\SbRandTwoHundredDiff}{-2.00}
\newcommand{\SbLtwoJointFailures}{40}
\newcommand{\SbLtwoCells}{4}
\newcommand{\SbKappaLfourHundred}{4.3072}
\newcommand{\SbKappaLfourTwoHundred}{5.6212}
\newcommand{\SbKappaLoadLtwoTwoHundred}{1.0031}
\newcommand{\SbMatchPoneKappa}{8.0352}
\newcommand{\SbMatchPoneKappaLo}{7.9971}
\newcommand{\SbMatchPoneKappaHi}{8.0787}
\newcommand{\SbMatchPtwoKappa}{7.6479}
\newcommand{\SbMatchPtwoKappaLo}{7.5971}
\newcommand{\SbMatchPtwoKappaHi}{7.7063}
\newcommand{\SbMatchFiftyPoneKappa}{4.4873}
\newcommand{\SbMatchFiftyPtwoKappa}{4.5074}
\newcommand{\PrimLfourFiftyPreDiff}{-0.10}
\newcommand{\PrimLfourFiftyPreCiLow}{-0.25}
\newcommand{\PrimLfourFiftyPreCiHigh}{0.00}
\newcommand{\PrimLfourFiftyPreTies}{18}
\newcommand{\PrimRandFiftyPreDiff}{-2.45}
\newcommand{\PrimLfourHundredPreDiff}{-0.25}
\newcommand{\PrimRandHundredPreDiff}{-4.05}
\newcommand{\PrimLfourTwoHundredPreDiff}{-0.25}
\newcommand{\PrimRandTwoHundredPreDiff}{-6.55}
\newcommand{\PrimLfourFourHundredPreDiff}{0.60}
\newcommand{\PrimLfourFourHundredPreNetWins}{7}
\newcommand{\PrimRandFourHundredPreDiff}{-10.90}
\newcommand{\PrimRandFourHundredPreProbeWins}{10}
\newcommand{\FullLfourFiftyKappaMin}{4.4873}
\newcommand{\FullLtwoFiftyKappaMin}{1.0000}
\newcommand{\FullRandFiftyKappaMin}{3.1612}
\newcommand{\FullLfourHundredKappaMin}{6.0155}
\newcommand{\FullLtwoHundredKappaMin}{1.0016}
\newcommand{\FullRandHundredKappaMin}{4.4289}
\newcommand{\FullLfourTwoHundredKappaMin}{8.0352}
\newcommand{\FullLtwoTwoHundredKappaMin}{1.0146}
\newcommand{\FullRandTwoHundredKappaMin}{6.2377}
\newcommand{\FullLfourTwoHundredSubsetOver}{0.0206}
\newcommand{\FullLfourTwoHundredArgminFound}{6}
\newcommand{\FullLfourTwoHundredArgminRankMed}{38}
\newcommand{\ReluGapSmax}{12}
\newcommand{\ReluGapDraws}{1\,500}
\newcommand{\ReluGapFeatures}{30}
\newcommand{\ReluLooseMin}{0.00}
\newcommand{\ReluLooseMax}{7.93}
\newcommand{\ReluLfourFiftyGap}{0.87}
\newcommand{\ReluLfourFiftyCiLow}{0.63}
\newcommand{\ReluLfourFiftyCiHigh}{1.10}
\newcommand{\ReluLfourFiftyPositive}{22}
\newcommand{\ReluRandFiftyGap}{2.03}
\newcommand{\ReluRandFiftyCiLow}{1.77}
\newcommand{\ReluRandFiftyCiHigh}{2.27}
\newcommand{\ReluRandFiftyPositive}{29}
\newcommand{\ReluRandHundredGap}{4.33}
\newcommand{\ReluRandHundredCiLow}{4.07}
\newcommand{\ReluRandHundredCiHigh}{4.57}
\newcommand{\ReluRandHundredPositive}{30}
\newcommand{\ReluCensoredCells}{3}
\newcommand{\KdThreeMaxLevMin}{0.4885}
\newcommand{\KdThreeModels}{180}
\newcommand{\FrontierLpSecondsTwoHundred}{0.32}
\newcommand{\CeilLfourFiftyValGap}{0.0039}
\newcommand{\CeilLfourFiftyValCiLow}{0.0025}
\newcommand{\CeilLfourFiftyValCiHigh}{0.0055}
\newcommand{\CeilLfourHundredValGap}{0.0222}
\newcommand{\CeilLfourHundredValCiLow}{0.0195}
\newcommand{\CeilLfourHundredValCiHigh}{0.0249}
\newcommand{\CeilLfourTwoHundredValGap}{0.0368}
\newcommand{\CeilLfourTwoHundredValCiLow}{0.0351}
\newcommand{\CeilLfourTwoHundredValCiHigh}{0.0386}
\newcommand{\CeilLfourFourHundredValGap}{0.0715}
\newcommand{\CeilLfourFourHundredValCiLow}{0.0677}
\newcommand{\CeilLfourFourHundredValCiHigh}{0.0747}
\newcommand{\CeilSbLfourHundredValGap}{0.0261}
\newcommand{\CeilSbLfourTwoHundredValGap}{0.0520}
\newcommand{\CeilRandFiftyValGap}{-0.2180}
\newcommand{\CeilRandTwoHundredValGap}{-0.1829}
\newcommand{\CeilRandFourHundredValGap}{-0.1813}
\newcommand{\CeilLfourCells}{6}
\newcommand{\CeilLfourModels}{80}
\newcommand{\CeilLfourMissed}{80}
\newcommand{\CeilRandModels}{90}
\newcommand{\CeilRandMissed}{0}
\newcommand{\CeilReproMaxDelta}{0.00}
\newcommand{\CeilReproCells}{22}
\newcommand{\EeightTrainedFiftyRgeom}{1.0006}
\newcommand{\EeightTrainedFiftyKappa}{4.468}
\newcommand{\EeightTrainedFiftyRreadout}{1.0258}
\newcommand{\EeightFrozenFiftyRgeom}{1.0427}
\newcommand{\EeightFrozenFiftyKappa}{2.880}
\newcommand{\EeightFrozenFiftyRreadout}{1.8978}
\newcommand{\EeightRandomFiftyRgeom}{1.0201}
\newcommand{\EeightRandomFiftyKappa}{3.106}
\newcommand{\EeightRandomFiftyRreadout}{1.0000}
\newcommand{\EeightTrainedHundredRgeom}{1.0003}
\newcommand{\EeightTrainedHundredKappa}{6.014}
\newcommand{\EeightTrainedHundredRreadout}{1.0228}
\newcommand{\EeightFrozenHundredRgeom}{1.0209}
\newcommand{\EeightFrozenHundredKappa}{4.122}
\newcommand{\EeightFrozenHundredRreadout}{1.9217}
\newcommand{\EeightRandomHundredRgeom}{1.0100}
\newcommand{\EeightRandomHundredKappa}{4.373}
\newcommand{\EeightRandomHundredRreadout}{1.0000}
\newcommand{\EeightTrainedTwoHundredRgeom}{1.0002}
\newcommand{\EeightTrainedTwoHundredKappa}{8.060}
\newcommand{\EeightTrainedTwoHundredRreadout}{1.0185}
\newcommand{\EeightFrozenTwoHundredRgeom}{1.0097}
\newcommand{\EeightFrozenTwoHundredKappa}{5.874}
\newcommand{\EeightFrozenTwoHundredRreadout}{1.9201}
\newcommand{\EeightRandomTwoHundredRgeom}{1.0048}
\newcommand{\EeightRandomTwoHundredKappa}{6.272}
\newcommand{\EeightRandomTwoHundredRreadout}{1.0000}
\newcommand{\EeightSeeds}{10}
\newcommand{\ScSeeds}{10}
\newcommand{\EsevenSteps}{50\,000}
\newcommand{\ScLfourKappa}{10.5098}
\newcommand{\ScLfourRgeom}{1.0001}
\newcommand{\ScLtwoKappa}{1.0653}
\newcommand{\ScLtwoRgeom}{32.3131}
\newcommand{\ScRandKappa}{8.9069}
\newcommand{\ScRandRgeom}{1.0025}
\newcommand{\ScRegistered}{10.79}
\newcommand{\ScMeasured}{10.510}
\newcommand{\ScRegisteredErrorPct}{-2.60}
\newcommand{\ScFitExponent}{0.4202}
\newcommand{\ScPredicted}{10.759}
\newcommand{\ScPredErrorPct}{-2.31}
\newcommand{\MarginFiftyGap}{1.3260}
\newcommand{\MarginFiftyRatio}{1.4195}
\newcommand{\MarginHundredGap}{1.5866}
\newcommand{\MarginHundredRatio}{1.3582}
\newcommand{\MarginTwoHundredGap}{1.7974}
\newcommand{\MarginTwoHundredRatio}{1.2882}
\newcommand{\MarginFourHundredGap}{1.6029}
\newcommand{\MarginFourHundredRatio}{1.1800}
\newcommand{\MarginDrop}{0.1940}
\newcommand{\MarginDropCiLow}{0.0903}
\newcommand{\MarginDropCiHigh}{0.2818}
\newcommand{\ScExpTrained}{0.4101}
\newcommand{\ScExpRandom}{0.4977}
\newcommand{\ScCrossing}{3\,039}
\newcommand{\SubsetFiftyOver}{0.0025}
\newcommand{\SubsetFiftyRank}{5}
\newcommand{\SubsetHundredOver}{0.0118}
\newcommand{\SubsetHundredRank}{23}
\newcommand{\SubsetTwoHundredOver}{0.0206}
\newcommand{\SubsetTwoHundredRank}{38}
\newcommand{\SubsetFourHundredOver}{0.0398}
\newcommand{\SubsetFourHundredRank}{131}
\newcommand{\SubsetFourHundredFound}{2}
\newcommand{\SaeQwenDeepestRgeom}{1.0178}
\newcommand{\SaeLfmDeepestRgeom}{1.0240}
\newcommand{\SaeLfmClosestRgeom}{1.0083}
\newcommand{\SaeSmolDeepestRgeom}{1.0319}
\newcommand{\SaeSmolClosestRgeom}{1.0105}
\newcommand{\SaeMlpCount}{16}
\newcommand{\SaeDictRgeomMin}{1.0055}
\newcommand{\SaeDictRgeomMax}{1.2018}
\newcommand{\SaeDictCvMin}{0.090}
\newcommand{\SaeDictCvMax}{0.469}
\newcommand{\SaeCtlRgeomMin}{1.0000}
\newcommand{\SaeCtlRgeomMax}{1.0005}
\newcommand{\SaeCtlCvMin}{0.005}
\newcommand{\SaeCtlCvMax}{0.020}
\newcommand{\SaeRandInitRgeomMin}{1.0038}
\newcommand{\SaeRandInitRgeomMax}{1.0065}
\newcommand{\SaeRandInitCvMin}{0.060}
\newcommand{\SaeRandInitCvMax}{0.078}
\newcommand{\SaeDicts}{40}
\newcommand{\SaeFailEarlier}{40}
\newcommand{\SaeFailGapMin}{4}
\newcommand{\SaeFailGapMax}{40}
\newcommand{\SaeFailTrainedMin}{13}
\newcommand{\SaeFailTrainedMax}{45}
\newcommand{\SaeSaeCount}{24}
\newcommand{\SaeControls}{7}
\newcommand{\SaePairedLayers}{10}
\newcommand{\SaeDMin}{576}
\newcommand{\SaeDMax}{4\,096}
\newcommand{\SaeLoadMin}{4.5}
\newcommand{\SaeLoadMax}{85.3}
\newcommand{\SaeDictsAboveControl}{40}
\newcommand{\SaeLoadOneRgeom}{1.0674}
\newcommand{\SaeLoadOneCv}{0.276}
\newcommand{\SaeLoadOneLoad}{5.3}
\newcommand{\SaeLoadTwoRgeom}{1.0749}
\newcommand{\SaeLoadTwoCv}{0.404}
\newcommand{\SaeLoadTwoLoad}{42.7}
\newcommand{\SaeLoadThreeRgeom}{1.0734}
\newcommand{\SaeLoadThreeCv}{0.469}
\newcommand{\SaeLoadThreeLoad}{85.3}
\newcommand{\SaeIdentityGap}{2e-16}
\newcommand{\SaeIdentityRows}{57}
\newcommand{\ScPrimDiff}{2.00}
\newcommand{\ScPrimNetWins}{10}
\newcommand{\ScPrimCiLow}{1.50}
\newcommand{\ScPrimCiHigh}{2.50}
\newcommand{\ScPrimTestMaxDiff}{2.00}
\newcommand{\EeightTrainedFiftyTaskGain}{6.59}
\newcommand{\EeightTrainedHundredTaskGain}{8.40}
\newcommand{\EeightTrainedTwoHundredTaskGain}{9.71}
\newcommand{\EeightFrozenFiftyTaskGain}{4.39}
\newcommand{\EeightFrozenHundredTaskGain}{4.70}
\newcommand{\EeightFrozenTwoHundredTaskGain}{4.50}
\newcommand{\EeightFrozenTwoHundredTaskWins}{10}
\newcommand{\EeightRandomTwoHundredTaskGain}{1.00}

\title{Near-Floor Geometry Is Generic: Leverage Dispersion\\in Trained Overcomplete Codes}

\author{%
  \name Hector Borobia \email hecboar@doctor.upv.es \\
  \addr VRAIN, Universitat Polit\`ecnica de Val\`encia, Spain
  \AND
  \name Elies Segu\'i-Mas \email esegui@upvnet.upv.es \\
  \addr Universitat Polit\`ecnica de Val\`encia, Spain
  \AND
  \name Guillermina Tormo-Carb\'o \email gtormo@upvnet.upv.es \\
  \addr Universitat Polit\`ecnica de Val\`encia, Spain%
}

\def\month{08}
\def\year{2026}
\def\openreview{\url{https://openreview.net/forum?id=XXXXXXXXXX}}

\begin{document}
\maketitle

\begin{abstract}
An overcomplete code packs $F$ features into $d<F$ dimensions, so linear readout of one
feature picks up cross-talk from the others, bounded below by the rank--trace floor
$W(F,d)=(F-d)/(d(F-1))$. Proximity to this floor is read as evidence of an efficient arrangement. It
is not: an i.i.d.\ code of the same shape attains it at every shape and load we measure
(\SaeCtlRgeomMin--\SaeCtlRgeomMax{} over \SaeControls{} matched controls), and the reason is exact.
For the gain-calibrated pseudoinverse the attainment ratio is
$\frac{d}{F(F-d)}(\sum_i h_i^{-1}-F)$ with $\sum_i h_i=d$, so it equals one precisely when the
leverage $h_i$ is equalised across features: distance from the floor \emph{is} leverage dispersion.

On \SaeDicts{} released decoder matrices --- \SaeSaeCount{} sparse-autoencoder decoders plus
\SaeMlpCount{} MLP down-projections of a hybrid model --- all sit \emph{above} their matched control
(\SaeDictRgeomMin--\SaeDictRgeomMax), and all reach guaranteed affine failure \SaeFailGapMin{} to
\SaeFailGapMax{} sparsity levels earlier: where a random code of the same shape is still safe, a
released dictionary already has a feature no affine rule recovers. Autoencoders trained by the same
recipe on a \emph{randomly initialised} copy of the same model return to the floor
(\SaeRandInitRgeomMin--\SaeRandInitRgeomMax) in all \SaePairedLayers{} paired layers, so the excess
reflects what the model learned, not the act of fitting an autoencoder. Under controlled conditions
the loss decides whether an interface exists at all: $L^2$ reaches
$R_{\mathrm{geom}}=\SALtwoFiftyRgeom$ to \ScLtwoRgeom{} with its frontier collapsed to $1$, against
$\SALfourTwoHundredRgeom$ for $L^4$, and that frontier follows a rate we registered before the widest
run (\ScRegistered{} predicted, \ScMeasured{} measured).

Two limits are stated rather than implied: frontier statements are worst cases, not typical-case
claims; and because the probe class contains the network's own decoder, our decoder comparison
bounds supervised fitting rather than affine expressiveness.

\end{abstract}

\section{Introduction}
\label{sec:intro}

An overcomplete code places $F$ feature directions in $d<F$ dimensions. Such codes are the object of
study in work on superposition and computation in superposition, and they are what a sparse
autoencoder returns when it decomposes a language model's activations into a dictionary of feature
directions. Whatever else one wants from such a code, a great deal of downstream practice reads
features out of it \emph{linearly}: linear probes, the decoder of the autoencoder itself, steering
vectors, and any diagnostic built from inner products with feature directions.

Linear readout of an overcomplete code cannot be exact. With $F>d$ the directions cannot be
orthogonal, so recovering one feature's coefficient picks up contributions from the others, and the
average squared size of that cross-talk is bounded below by a quantity that depends only on the two
counts,
\[
  W(F,d)\;=\;\frac{F-d}{d(F-1)} .
\]
Reporting how close a code comes to this rank--trace floor is therefore a natural way to ask whether
it is arranged well for linear readout, and ratios of this kind have been read as evidence that
training discovers efficient geometry.

\paragraph{What a reader should take from this paper.} Proximity to the floor is not evidence of
anything a code learned, because it is what an unstructured code gives for free; the quantity that
distance from the floor actually measures is the \emph{dispersion of leverage} across features; and
measured this way, no released sparse-autoencoder dictionary we examined is near the floor --- every
one carries more cross-talk under the best linear readout it admits than a random code of its own
shape does, while an
autoencoder trained by the same recipe on a randomly initialised model is not. The practical
consequence is that the excess is a property of the represented content rather than of the fitting
procedure, and that it moves with depth but not with dictionary width.

\subsection{Why proximity to the floor says less than it appears to}

Our starting point is an identity, not an experiment. For a code $\Phi\in\R^{d\times F}$ read out by
its gain-calibrated pseudoinverse, the ratio of achieved cross-talk to the floor is
\[
  \Rgeom(\Phi)\;=\;\frac{d}{F(F-d)}\left(\sum_{i=1}^{F}\frac{1}{h_i}-F\right),
  \qquad h_i=\phi_i^{\top}(\Phi\Phi^{\top})^{-1}\phi_i ,
\]
and the leverage scores obey $\sum_i h_i=d$ exactly. Cauchy--Schwarz then gives
$\sum_i h_i^{-1}\ge F^2/d$ with equality precisely when every $h_i$ equals $d/F$, so
$\Rgeom\ge1$ with equality exactly when leverage is equalised across features
(Proposition~\ref{prop:leverage}). Distance from the floor is not a separate property that a code
might or might not have alongside its leverage profile: \emph{it is} that profile's dispersion, in a
specific harmonic sense that we make precise and then use.

Two consequences follow immediately and are worth stating before any measurement. First, a code whose
leverage is nearly uniform sits at the floor whether or not it encodes anything, which is why an
i.i.d.\ Gaussian code attains it. Second, the ratio is driven by $\sum_i h_i^{-1}$, a harmonic
quantity, so the spread of leverage can grow substantially without the ratio moving --- a
dissociation we observe and would otherwise have had to report as an anomaly.

\subsection{What we measure, and on what}

We evaluate the diagnostic on two kinds of object, and the division of labour between them is
deliberate.

\emph{Released dictionaries} (Section~\ref{sec:dictionaries}) test whether any of this transfers
beyond a controlled setting. We take the decoder matrices of \SaeDicts{} published sparse
autoencoders --- five model families and three independent training recipes, with $d$ from \SaeDMin{}
to \SaeDMax{} and feature loads from \SaeLoadMin$\times$ to \SaeLoadMax$\times$ --- and compute
$\Rgeom$, the leverage distribution, and the failure threshold from the decoder alone. No model
weights, no activations, no training, and no linear programme are involved, so the whole section
costs seconds per dictionary and can be re-run by a reader.

\emph{Controlled campaigns} (Section~\ref{sec:campaigns}) supply what released artefacts cannot: an
intervention. We train 400 small networks in which the loss, the width, the feature load and the
training sparsity are set by us, so that the effect of each on the resulting code can be measured
rather than inferred, and we retain every trained model so that later analyses need no retraining.

\subsection{Contributions}

\begin{enumerate}[leftmargin=*,itemsep=3pt]

\item \textbf{Distance from the rank--trace floor is leverage dispersion, exactly.} The identity is
elementary and we claim no novelty for the mathematics; what is new is that it does predictive work.
It explains why unstructured codes attain the floor, why trained dictionaries do not, and why
dispersion can grow while the ratio stays flat. We verify it against the implementation on
\SaeIdentityRows{} measured codes, with a maximum deviation of \SaeIdentityGap{}.

\item \textbf{Near-floor geometry is generic, and we measure the baseline rather than assert it.}
\SaeControls{} i.i.d.\ controls, matched in shape and operating sparsity to each dictionary, land at
$\Rgeom=\SaeCtlRgeomMin$ to \SaeCtlRgeomMax{} across loads from \SaeLoadMin$\times$ to
\SaeLoadMax$\times$. Any claim that a learned code is close to optimal for linear readout needs this
comparison, and it is not present in the literature we build on.

\item \textbf{No released dictionary we measured is at the floor.} All \SaeDictsAboveControl{} of
\SaeDicts{} sit above their own matched control, at $\Rgeom=\SaeDictRgeomMin$ to
\SaeDictRgeomMax{}, with leverage dispersion \SaeDictCvMin{} to \SaeDictCvMax{} against the
controls' \SaeCtlCvMin{} to \SaeCtlCvMax{}. Whatever these dictionaries optimise, it is not the
evenness of their own leverage.

\item \textbf{The same holds in sparsity units, which is the form a practitioner can use.} The
diagnostic also returns the sparsity at which some feature is guaranteed to be unrecoverable by any
affine rule. All \SaeFailEarlier{} dictionaries reach it \SaeFailGapMin{} to \SaeFailGapMax{} levels
earlier than the i.i.d.\ code of their own shape and operating sparsity: where a random code is still
safe, a released dictionary already has a feature no affine rule recovers. This is a worst-case
statement about one feature rather than an average over all of them, so its agreement with the ratio
is not automatic --- though both are functionals of the same leverage distribution, and neither is an
external endpoint (Section~\ref{sec:discussion}).

\item \textbf{A matched randomly-initialised control locates the cause.} Autoencoders trained by an
identical recipe on a randomly initialised copy of the same model return to the floor
($\Rgeom=\SaeRandInitRgeomMin$ to \SaeRandInitRgeomMax) in all \SaePairedLayers{} paired layers. The
excess therefore reflects what the model learned, not the act of fitting a sparse autoencoder --- a
distinction an i.i.d.\ control cannot draw, because it does not share the training procedure.

\item \textbf{The loss decides whether a linear interface exists at all.} Across four widths, two
feature loads and two training sparsities, $L^2$ training reaches $\Rgeom=\SALtwoFiftyRgeom$ to
\ScLtwoRgeom{} with its affine recovery frontier collapsed to $1$, while $L^4$ holds
$\SALfourTwoHundredRgeom$ with a frontier that grows from \FullLfourFiftyKappaMin{} to
\ScMeasured{}. This is the largest effect we report and the one that replicates most widely.

\item \textbf{The frontier follows a rate, tested out of sample, and the advantage it buys stops
growing.} Fitted on three widths the frontier implied \ScRegistered{} at $d=400$; we registered that
number in the run configuration before the width was trained and measured \ScMeasured{}. Over the
same four widths the advantage over an untrained code narrows from \MarginFiftyRatio{} to
\MarginFourHundredRatio{}, which we report as a limitation on the preceding claim rather than
alongside it.

\item \textbf{An artefact that can be checked.} The weights of all 400 trained networks are
released, along with the SHA-256 of every third-party dictionary analysed, a decision log recording
what was registered before each campaign, and a register of nine defects found during this work
together with what each invalidated. Section~\ref{sec:limitations} states which readings we
withdrew and why.

\end{enumerate}

\subsection{Scope}

Three limits apply throughout and we state them here rather than at the end. Every statement about
the affine recovery frontier is a worst case over feature amplitudes: it says what no affine rule can
guarantee, not what fails in typical use. The comparison between a network's own decoder and a fitted
affine probe (Section~\ref{sec:probes}) measures what supervised fitting recovers rather than what an
affine map can express, because the probe class contains the network's decoder --- and we verify
that the class membership is not hypothetical: the selection misses the network's own map, on the
very validation objective it maximises, in every $L^4$ model above the tie width. We report the
comparison for completeness and build nothing on it. And the controlled campaigns use one task family at widths up
to $d=400$, so they establish the mechanism, not its magnitude in a deployed model --- which is
precisely why the dictionary measurements are the other half of the paper.

\section{Problem setting and notation}
\label{sec:setting}

We write $\Otilde(\cdot),\Omegatilde(\cdot),\Thetatilde(\cdot)$ to hide factors polynomial in
$\log d$. Table~\ref{tab:notation} fixes the notation used throughout.

\begin{table}[t]
\centering
\caption{Notation.}
\label{tab:notation}
\small
\begin{tabular}{ll}
\toprule
symbol & meaning \\
\midrule
$d$ & activation width (number of neurons) \\
$F$ & number of features, $F>d$ \\
$\Phi\in\R^{d\times F}$ & code; column $\phi_i$ is the direction of feature $i$ \\
$G\in\R^{F\times d}$ & linear readout \\
$M=G\Phi\in\R^{F\times F}$ & readout interface, $\operatorname{rank}(M)\le d$ \\
$D=\operatorname{diag}(M_{11},\dots,M_{FF})$ & diagonal gains; $D^{-1}M$ is the calibrated interface \\
$A=M-I_F$ & cross-talk matrix (zero diagonal after calibration) \\
$b\in\{0,1\}^F$ & sparse Boolean state; $S=\operatorname{supp}(b)$, $s=|S|$ \\
$\mu=\max_{i\ne j}|\langle\phi_i,\phi_j\rangle|$ & coherence of a unit-norm code \\
$W(F,d)=\frac{F-d}{d(F-1)}$ & mean-square Welch-type floor \\
$\widehat W$ & measured mean squared off-diagonal cross-talk \\
$\theta$ & decoding threshold (default $1/2$) \\
$p_{\mathrm{rec}}(s)$ & probability of exact threshold recovery at sparsity $s$ \\
$s_{95}$ & largest $s$ such that $p_{\mathrm{rec}}\ge0.95$ for \emph{every} tested $s'\le s$ \\
$\Sigma=\Phi\Phi^{\top}$ & frame operator (invertible iff $\Phi$ has full row rank) \\
$h_i=\phi_i^{\top}\Sigma^{-1}\phi_i$ & leverage of feature $i$; $\sum_ih_i=d$ \\
$W_{\star}(\Phi)$ & the code's own floor: least mean-square cross-talk it admits \\
$R_{\mathrm{geom}},R_{\mathrm{readout}}$ & $W_{\star}(\Phi)/W(F,d)$ and $W(G,\Phi)/W_{\star}(\Phi)$ \\
$\Phi_{-i}$ & $\Phi$ with column $i$ deleted \\
$\alpha\in(0,1]$ & smallest detectable active amplitude \\
$\kappa_i(\alpha)$ & collision frontier of feature $i$; separable at $s$ iff $\kappa_i>s$ \\
$\delta_i(s),\gamma_i(s)$ & distance between the two hulls, and the margin $\delta_i(s)/2$ \\
$m,m'$ & feature counts \emph{inside imported statements}; $m'$ plays the role of $F$ \\
\bottomrule
\end{tabular}
\end{table}

Two decoding questions are compared throughout. \emph{Analog reconstruction} asks that the
scores $z=Mb$ approximate $b$ in a norm; \emph{support recovery} asks only that
$\mathbf 1\{z_i\ge\theta\}=b_i$ for every $i$. The first is a statement about magnitudes and is
never exactly satisfiable when $F>d$; the second is a statement about a decision boundary and
can be satisfied exactly --- by a thresholded linear score as much as by anything else, which
is why that rule is measured here too rather than treating support recovery as the preserve of
nonlinearity. Everything below is an elaboration of that asymmetry.


\section{The linear-readout floor}
\label{sec:floor}

\begin{theorem}[Unit-diagonal cross-Gramian Welch floor]
\label{thm:floor}
Let $F\ge2$, let $\Phi\in\R^{d\times F}$ and $G\in\R^{F\times d}$, and set $M=G\Phi$. Suppose
$M_{ii}=1$ for all $i\in[F]$. Then
\begin{equation}
  \sum_{i\neq j}|M_{ij}|^2\;\ge\;\frac{F(F-d)}{d}.
  \label{eq:floorsum}
\end{equation}
Consequently, if $F>d$,
\begin{equation}
  \frac{1}{F(F-1)}\sum_{i\neq j}|M_{ij}|^2\;\ge\;W(F,d):=\frac{F-d}{d(F-1)},
  \label{eq:floor}
\end{equation}
and hence $\max_{i\neq j}|M_{ij}|\ge\sqrt{W(F,d)}$. In particular, if $F/d\to\infty$ then
$\max_{i\neq j}|M_{ij}|\ge(1-o(1))\,d^{-1/2}$.
\end{theorem}

\begin{proof}
Since $M=G\Phi$ with $G\in\R^{F\times d}$ and $\Phi\in\R^{d\times F}$ we have
$\operatorname{rank}(M)\le d$, and the unit diagonal gives $\operatorname{tr}(M)=F$. For any
matrix of rank at most $d$, $|\operatorname{tr}(M)|\le\nucnorm{M}\le\sqrt d\,\norm{M}_F$
(Appendix~\ref{app:nuclear}), so $F^2\le d\,\norm{M}_F^2$. Expanding
$\norm{M}_F^2=F+\sum_{i\neq j}|M_{ij}|^2$ gives~\eqref{eq:floorsum}, and dividing by $F(F-1)$
gives~\eqref{eq:floor}. For the maximum, note that
$\max_{i\neq j}|M_{ij}|^2\ge\frac{1}{F(F-1)}\sum_{i\neq j}|M_{ij}|^2$, and take square roots.
\end{proof}

The argument is the standard rank--trace mechanism, cf.~\cite{datta2012geometry}; we give it
in full only for self-containedness. The next corollary is what the diagnostic actually uses,
because it removes the normalisation hypothesis.

\begin{corollary}[Gain-normalised form]
\label{cor:calibfree}
Let $F>d$ and let $M=G\Phi$ satisfy $\operatorname{rank}(M)\le d$ and $M_{ii}\neq0$ for all
$i$. Then
\[
  \max_{i\neq j}\frac{|M_{ij}|}{|M_{ii}|}\;\ge\;\sqrt{\frac{F-d}{d(F-1)}}.
\]
\end{corollary}

\begin{proof}
Apply Theorem~\ref{thm:floor} to $D^{-1}M$ with $D=\operatorname{diag}(M_{11},\dots,M_{FF})$:
this matrix has rank at most $d$, unit diagonal, and entries $M_{ij}/M_{ii}$.
\end{proof}

\begin{remark}[Why the gain normalisation is essential]
\label{rem:calib}
Without a diagonal normalisation the absolute off-diagonal entries of $G\Phi$ can be made
arbitrarily small by shrinking $G$, and the floor appears to fail. A concrete instance: take
$\Phi$ to be the three unit vectors at $120^{\circ}$ in $\R^2$ --- an equiangular tight frame
--- and $G=D\Phi^{\top}$ with $D=\operatorname{diag}(1/10,1/20,1/20)$. Then $G\Phi$ has
diagonal $(1/10,1/20,1/20)$ and maximum off-diagonal entry $1/20$, ten times \emph{below}
$\sqrt{(F-d)/(d(F-1))}=1/2$. There is no contradiction: this interface reads out its own
coordinates with gains far below one, and after recalibrating each row to unit gain the
maximum off-diagonal entry becomes exactly $1/2$ --- the floor, attained with equality.
Corollary~\ref{cor:calibfree} isolates the scale-invariant content: what is bounded below is
the interference-to-gain ratio. This example is verified in the test suite, and
Section~\ref{sec:res-e4} shows that a gradient optimiser left free of the constraint exploits
exactly this degree of freedom.
\end{remark}

\begin{proposition}[Equality within the tied family]
\label{prop:tight}
Let $\Phi\in\R^{d\times F}$ have unit-norm columns and let $G=\Phi^{\top}$, so that
$M=\Phi^{\top}\Phi$ has unit diagonal. Then $\sum_{i\neq j}|M_{ij}|^2\ge F(F-d)/d$, with
equality \emph{if and only if} $\Phi\Phi^{\top}=(F/d)I_d$, that is, iff $\Phi$ is a unit-norm
tight frame.
\end{proposition}

\begin{proof}
The diagonal entries are $\norm{\phi_i}^2=1$. By cyclicity,
$\norm{M}_F^2=\operatorname{tr}((\Phi\Phi^{\top})^2)=\norm{\Sigma}_F^2$ with
$\Sigma=\Phi\Phi^{\top}\succeq0$ and $\operatorname{tr}\Sigma=F$. Cauchy--Schwarz on the
eigenvalues of $\Sigma$ gives $\norm{\Sigma}_F^2\ge(\operatorname{tr}\Sigma)^2/d=F^2/d$, with
equality iff all $d$ eigenvalues coincide, i.e.\ $\Sigma=(F/d)I_d$. Subtracting the diagonal
contribution $F$ gives the claim.
\end{proof}

Within this family --- unit-norm columns read out by their own transpose --- equality therefore
characterises tight frames, consistent
with~\cite{waldron2003generalized,datta2012geometry}. We stress the scope: it is a statement
about the tied Gram matrix, not about every rank-$d$ unit-diagonal interface. Other readouts
attain the same floor under different conditions, and the next proposition identifies the
condition for the pseudoinverse. Whether learned codes find such configurations is an empirical
question, answered in Section~\ref{sec:res-e4}.

\begin{proposition}[The pseudoinverse ratio is a leverage statistic]
\label{prop:leverage}
Let $\Phi\in\R^{d\times F}$ have full row rank, let $G=\Phi^{+}$, and write
$P=\Phi^{+}\Phi$ for the orthogonal projector onto the row space, with leverage scores
$h_i=P_{ii}\in(0,1]$. Then the gain-normalised interface $\widetilde M=D^{-1}P$ satisfies
\[
  \sum_{i\neq j}\widetilde M_{ij}^{2}\;=\;\sum_{i=1}^{F}\frac1{h_i}-F ,
  \qquad\text{so}\qquad
  r=\frac{d}{F(F-d)}\left(\sum_{i=1}^{F}\frac1{h_i}-F\right).
\]
Since $\sum_i h_i=\operatorname{tr}P=d$, Cauchy--Schwarz gives $\sum_i h_i^{-1}\ge F^2/d$ with
equality iff $h_i=d/F$ for every $i$. Hence $r\ge1$, and $r=1$ exactly when leverage is
equalised across features.
\end{proposition}

\begin{proof}
$P$ is symmetric idempotent, so row $i$ obeys $\sum_j P_{ij}^2=P_{ii}=h_i$ and $D_{ii}=h_i$.
Dividing row $i$ by $h_i$ and removing the unit diagonal entry leaves
$(h_i-h_i^2)/h_i^2=h_i^{-1}-1$; summing over $i$ gives the identity. The inequality is
Cauchy--Schwarz applied to $\sum_i h_i\cdot\sum_i h_i^{-1}\ge F^2$, and dividing by $F(F-1)$
and by $W(F,d)=(F-d)/(d(F-1))$ turns it into $r\ge1$.
\end{proof}

This is worth stating because it deflates a natural over-reading of a ratio close to one under
the pseudoinverse readout. Attaining the floor with $G=\Phi^{+}$ says that the code spreads its
features evenly over the row space --- nothing more. It is a statement about the geometry of the
code, not about whether the code is a good representation for any task; a code can have
perfectly equalised leverage and still be useless. We use this to scope the empirical claims of
Section~\ref{sec:res-e5}.

\subsection{The floor of the code itself}
\label{sec:codefloor}

Theorem~\ref{thm:floor} bounds every rank-$d$ unit-diagonal interface, so it is a statement
about the \emph{ensemble} of codes: no particular $\Phi$ need be able to attain it. That makes a
measured ratio hard to read. If a trained code sits $0.06\%$ above $W(F,d)$, is the readout
excellent, or is the floor simply unreachable for this code and the readout merely adequate? The
question is answered by computing the floor of the code itself, which is available in closed
form.

\begin{theorem}[Code-specific analog optimum]
\label{thm:codefloor}
Let $\Phi\in\R^{d\times F}$ have full row rank, let $\Sigma=\Phi\Phi^{\top}\succ0$ be its frame
operator, and let $h_i=\phi_i^{\top}\Sigma^{-1}\phi_i$ be the leverage of feature $i$. Then for
every $i$
\begin{equation}
  \min_{g\in\R^d}\Big\{\sum_{j\neq i}(g^{\top}\phi_j)^2 \;:\; g^{\top}\phi_i=1\Big\}
  \;=\;\frac{1}{h_i}-1,
  \label{eq:capon}
\end{equation}
attained uniquely at $g_i^{\star}=\Sigma^{-1}\phi_i/h_i$. Stacking the $F$ minimisers as rows gives
$G^{\star}=D_h^{-1}\Phi^{+}$ with $D_h=\operatorname{diag}(h_1,\dots,h_F)$: the row-wise optimum is
the Moore--Penrose pseudoinverse \emph{after} gain calibration, and coincides with $\Phi^{+}$ itself
only when every $h_i=1$. Consequently the smallest mean-square off-diagonal cross-talk this code
admits is
\begin{equation}
  W_{\star}(\Phi)\;:=\;\frac{1}{F(F-1)}\sum_{i=1}^{F}\Big(\frac1{h_i}-1\Big)
  \;=\;\frac{\sum_i h_i^{-1}-F}{F(F-1)}\;\ge\;W(F,d),
  \label{eq:codefloor}
\end{equation}
with equality if and only if $h_i=d/F$ for every $i$.
\end{theorem}

\begin{proof}
Since $\sum_{j}(g^{\top}\phi_j)^2=g^{\top}\Sigma g$, the constraint $g^{\top}\phi_i=1$ turns the
objective into $g^{\top}\Sigma g-1$. Minimising a positive-definite quadratic under one linear
equality is classical: the Lagrange conditions give $g=\lambda\Sigma^{-1}\phi_i$, the constraint
fixes $\lambda=1/h_i$, and the value is $\phi_i^{\top}\Sigma^{-1}\phi_i/h_i^2=1/h_i$. Subtracting
the constrained direction's contribution, which is exactly $1$, gives~\eqref{eq:capon}; strict
convexity gives uniqueness. Summing over $i$ and dividing by $F(F-1)$
gives~\eqref{eq:codefloor}. Finally $\sum_ih_i=\operatorname{tr}(\Sigma^{-1}\Phi\Phi^{\top})=d$,
so Cauchy--Schwarz gives $\sum_ih_i^{-1}\ge F^2/d$ with equality iff the $h_i$ are all equal,
and $(F^2/d-F)/(F(F-1))=W(F,d)$.
\end{proof}

\begin{remark}[Provenance]
\label{rem:capon}
Equation~\eqref{eq:capon} is the minimum-variance distortionless-response (Capon) beamformer with
the frame operator in place of a covariance~\cite{capon1969}, and $g_i^{\star}$ is the canonical
dual frame vector rescaled to unit gain; both are classical. We claim no
novelty in the optimisation. What the diagnostic uses is the reading of $h_i$ as a leverage
score, the resulting per-code floor~\eqref{eq:codefloor}, and the exact attribution of the gap
in Proposition~\ref{prop:excess}.
\end{remark}

\begin{proposition}[The excess is leverage dispersion, exactly]
\label{prop:excess}
With $c=d/F$ the mean leverage,
\[
  \sum_{i=1}^{F}\frac1{h_i}-\frac{F^2}{d}\;=\;\sum_{i=1}^{F}\frac{(h_i-c)^2}{c^{2}h_i}\;\ge\;0 .
\]
Hence $W_{\star}(\Phi)$ exceeds $W(F,d)$ by a weighted dispersion of the leverage scores, which
vanishes precisely when leverage is uniform.
\end{proposition}

\begin{proof}
Expand the right-hand side: $c^{-2}\sum_i(h_i^2-2ch_i+c^2)/h_i
=c^{-2}\sum_ih_i-2c^{-1}F+\sum_ih_i^{-1}$. Substituting $\sum_ih_i=d$ and $c=d/F$ turns the
first two terms into $F^2/d-2F^2/d=-F^2/d$.
\end{proof}

This separates two questions that a single ratio against $W(F,d)$ conflates. Writing
$W(G,\Phi)$ for the mean-square off-diagonal cross-talk actually achieved by a readout $G$,
\begin{equation}
  \underbrace{\frac{W(G,\Phi)}{W(F,d)}}_{\text{what is usually reported}}
  \;=\;\underbrace{\frac{W_{\star}(\Phi)}{W(F,d)}}_{R_{\mathrm{geom}}:\ \text{the code}}
  \times
  \underbrace{\frac{W(G,\Phi)}{W_{\star}(\Phi)}}_{R_{\mathrm{readout}}:\ \text{the decoder}} ,
  \label{eq:decomposition}
\end{equation}
and the two factors answer different questions. $R_{\mathrm{geom}}$ asks whether the code's
leverage is spread evenly; $R_{\mathrm{readout}}\ge1$ asks whether the decoder is the best one
\emph{for that code}. In particular Proposition~\ref{prop:leverage} says that the calibrated
pseudoinverse has $R_{\mathrm{readout}}=1$ identically, so a published attainment ratio computed
with $G=\Phi^{+}$ measures $R_{\mathrm{geom}}$ and nothing else. Section~\ref{sec:res-e5} reports
both factors.

Finally, the analog level reads a state distribution through one object only, which is what lets
Section~\ref{sec:hierarchy} compare it with the affine level under a single state model.

\begin{proposition}[Distribution-weighted reconstruction error]
\label{prop:secondmoment}
Let $M=G\Phi$ have unit diagonal, let $A=M-I_F$, and let the state $b\in\R^F$ be drawn from any
law with finite second moment $C=\E[bb^{\top}]$. Then
$\E\norm{Ab}_2^2=\operatorname{tr}(A^{\top}A\,C)$. The analog level therefore depends on the
state law only through $C$: two laws with the same second moment impose the same reconstruction
error on every readout.
\end{proposition}

\begin{proof}
$\E\norm{Ab}_2^2=\E\operatorname{tr}(A^{\top}Abb^{\top})=\operatorname{tr}(A^{\top}A\,C)$ by
linearity.
\end{proof}

\begin{corollary}[No better-than-floor $\varepsilon$-linear readout]
\label{cor:impossible}
Let $F>d$ and suppose $G$ satisfies $\norm{G y_i-e_i}_\infty\le\delta$ for singleton states
$y_i$, with $\delta<1/2$. Then $\delta/(1-\delta)\ge\sqrt{(F-d)/(d(F-1))}$; in particular
$\delta=\Omega(d^{-1/2})$ when $F/d\to\infty$.
\end{corollary}

\begin{proof}
Let $M=G[y_1,\dots,y_F]$. The hypothesis gives $|M_{ii}-1|\le\delta$ and $|M_{ji}|\le\delta$
for $j\neq i$. Calibrating, $|\widetilde M_{ij}|\le\delta/(1-\delta)$, and
Theorem~\ref{thm:floor} applies to $\widetilde M$.
\end{proof}


\section{The affine level: a per-feature collision frontier}
\label{sec:frontier}

Sections~\ref{sec:floor} and~\ref{sec:threshold} bound two criteria on random ensembles. Neither
answers the question an interpretability practitioner actually asks about a \emph{particular}
trained code: for this feature, in this network, up to what sparsity can a linear probe read the
feature's presence off the representation at all? This section gives that quantity exactly, as
the value of one linear programme per feature.

\begin{definition}[Affine separability of a feature]
\label{def:affsep}
Fix $\Phi\in\R^{d\times F}$, a feature $i$ and a sparsity $s$. Write $\mathcal A_i(s)$ for the
set of states $\Phi b$ with $b\in\{0,1\}^F$, $b_i=1$ and $|\operatorname{supp}(b)|\le s$, and
$\mathcal B_i(s)$ for the same with $b_i=0$. Feature $i$ is \emph{affinely separable at sparsity
$s$} if there exist $w\in\R^d$ and $\theta\in\R$ with $w^{\top}x>\theta$ for every
$x\in\mathcal A_i(s)$ and $w^{\top}x<\theta$ for every $x\in\mathcal B_i(s)$.
\end{definition}

The states are Boolean here, and that is not a simplification we could drop. Allowing an active
amplitude to be any positive number makes ``$i$ is active'' an open condition, and an active
state with $b_i\to0$ converges to an inactive one; the two hulls then touch for every code and
every sparsity, so nothing would ever be separable. A positive floor on the active amplitude is
what makes the question well posed, and Corollary~\ref{cor:kappa-alpha} is the graded statement
with that floor made explicit.

An affine functional separates two sets exactly when it separates their convex hulls, so the
question is whether those two hulls meet. The programme below decides it.

\begin{theorem}[The collision frontier]
\label{thm:frontier}
For $z\in\R^{F-1}$, indexed by the features other than $i$, let $z^{+}_j=\max(z_j,0)$ and
$z^{-}_j=\max(-z_j,0)$, so that both parts are non-negative and $z=z^{+}-z^{-}$. Let
$\Phi_{-i}$ be $\Phi$ with column $i$ deleted. Put
\begin{equation}
  \kappa_i\;:=\;\min_{z}\Big\{\max\big(\textstyle\sum_j z^{+}_j,\;1+\sum_j z^{-}_j\big)
  \;:\;\Phi_{-i}z=\phi_i,\;\norm{z}_\infty\le1\Big\}.
  \label{eq:kappa}
\end{equation}
with $\kappa_i:=+\infty$ when the constraints admit no $z$ at all. Then feature $i$ is affinely
separable at every sparsity $s'\le s$ if and only if $\kappa_i>s$. When $\kappa_i$ is finite the
largest sparsity at which feature $i$ is separable is $\lceil\kappa_i\rceil-1$; when
$\kappa_i=+\infty$ the feature is separable at every sparsity, and no collision exists to bound.
\end{theorem}

\begin{proof}
Separability fails exactly when $\operatorname{conv}\mathcal A_i(s)$ meets
$\operatorname{conv}\mathcal B_i(s)$. By Proposition~\ref{prop:knapsack} below, applied to the
coordinates other than $i$ with budget $s-1$ on the active side and $s$ on the inactive side,
those hulls are $\Phi\{e_i+u:u\in[0,1]^{F-1},\;\mathbf 1^{\top}u\le s-1\}$ and
$\Phi\{v:v\in[0,1]^{F-1},\;\mathbf 1^{\top}v\le s\}$ respectively --- the relaxation from
$0/1$ coordinates to the box is exact, because the polytope is integral. So a collision is a
pair $u,v$ in those two boxes with $\Phi(e_i+u)=\Phi v$. We may take $u$ and $v$ to have
disjoint supports, since a common component cancels from both sides and only relaxes the two sum
constraints. Then
$\Phi(e_i+u)=\Phi v$ reads $\Phi_{-i}(v-u)=\phi_i$, so with $z=v-u$ we have $z^{+}=v$ and
$z^{-}=u$. The budget on the inactive state is $\sum_j v_j\le s$, that on the active state is
$1+\sum_j u_j\le s$, and the entries lie in $[0,1]$, giving $\norm{z}_\infty\le1$. The two
budgets are exactly $\sum_jz^{+}_j\le s$ and $1+\sum_jz^{-}_j\le s$, so a collision at some
sparsity at most $s$ exists if and only if the objective in~\eqref{eq:kappa} can be brought to
$s$ or below, that is, if and only if $\kappa_i\le s$.
\end{proof}

Three things improve on the corresponding exactly-$s$ statement, which we record because the
earlier formulation of this work used it. The threshold is $s$ itself rather than
$2\min(s,F-s)-1$; the admissible sets nest in $s$, so monotonicity of the frontier holds by
construction rather than needing a proof and a restriction to $s\le F/2$; and $\kappa_i=7.4$ now
reads directly as ``separable up to $7$, lost at $8$'', which is what a capacity ought to mean.
Computationally~\eqref{eq:kappa} is a single linear programme, stated as
Algorithm~\ref{alg:frontier}: the $\max$ of two linear forms is handled by one auxiliary scalar,
and one solve settles every sparsity at once rather than one feasibility problem per $s$.

\begin{algorithm}[t]
\caption{\textsc{CollisionFrontier} --- the largest affinely separable sparsity, per feature}
\label{alg:frontier}
\begin{algorithmic}[1]
\Require code $\Phi\in\R^{d\times F}$; feature set $\mathcal I\subseteq[F]$; minimum amplitude
  $\alpha\in(0,1]$
\Ensure $\kappa_i(\alpha)$ and the frontier $\lceil\kappa_i(\alpha)\rceil-1$ for each
  $i\in\mathcal I$
\For{$i\in\mathcal I$}
  \State solve, in variables $u,v\in\R^{F-1}_{\ge0}$ and $t\in\R$:
  \Statex \qquad $\min\;t$ \quad s.t.\quad $\Phi_{-i}(u-v)=\phi_i$,\quad
    $\alpha\mathbf 1^{\top}u\le t$,\quad $1+\alpha\mathbf 1^{\top}v\le t$,\quad
    $u,v\le1/\alpha$
  \State $\kappa_i(\alpha)\gets t^{\star}$;\quad
    $s_i^{\max}\gets\lceil\kappa_i(\alpha)\rceil-1$
\EndFor
\State \Return $\{(\kappa_i(\alpha),s_i^{\max})\}_{i\in\mathcal I}$
\end{algorithmic}
\end{algorithm}

\begin{remark}[Why splitting $z$ into two non-negative vectors loses nothing]
\label{rem:split}
Algorithm~\ref{alg:frontier} does not constrain $u$ and $v$ to have disjoint supports, so a
feasible point need not have $u=z^{+}$ and $v=z^{-}$. It is still exact. If
$u_j,v_j>0$ for some $j$, subtracting $\min(u_j,v_j)$ from both leaves $u-v$ unchanged, keeps
feasibility, and strictly decreases both $\mathbf 1^{\top}u$ and $\mathbf 1^{\top}v$, hence does
not increase $t$. Every optimum therefore admits a canonical split, on which the two budget
constraints are exactly those of~\eqref{eq:kappa} and the bounds $u,v\le1/\alpha$ are exactly
$\norm{z}_\infty\le1/\alpha$.
\end{remark}

The cost is one linear programme in $2(F-1)+1$ variables with $d$ equality constraints per
feature, which is affordable for every feature at $d=50$ and not at $d=400$; how we choose a
subset there, and why the resulting minimum is an upper bound, is stated in
Section~\ref{sec:limitations}.

\subsection{What the state distribution contributes}
\label{sec:frontier-amplitudes}

Definition~\ref{def:affsep} allowed graded activations in $[0,1]$. It is natural to ask what
happens under a genuine state law, and the answer is unusually clean: almost nothing survives.

\begin{proposition}[Amplitudes wash out of the hulls]
\label{prop:knapsack}
For $\alpha\in(0,1]$ and $k\in\mathbb N$ let
$\mathcal U_k^{\alpha}=\{u\in\R_{\ge0}^{F}:u_j\in\{0\}\cup[\alpha,1],\;
|\operatorname{supp}(u)|\le k\}$. Then
$\operatorname{conv}\mathcal U_k^{\alpha}=\{u\in[0,1]^F:\mathbf 1^{\top}u\le k\}$ for every
$\alpha$.
\end{proposition}

\begin{proof}
The right-hand side is the unit-weight knapsack polytope, which is integral: its vertices are the
$0/1$ vectors with at most $k$ ones. Each such vertex lies in $\mathcal U_k^{\alpha}$ for every
$\alpha\le1$, so the right-hand side is contained in the hull; the reverse inclusion is immediate
since $\mathcal U_k^{\alpha}$ satisfies both constraints.
\end{proof}

\begin{corollary}[The frontier under a minimum amplitude]
\label{cor:kappa-alpha}
If active amplitudes lie in $[\alpha,1]$ then feature $i$ is separable at every sparsity
$s'\le s$ iff $\kappa_i(\alpha)>s$, where
\[
  \kappa_i(\alpha)=\min_z\Big\{\max\big(\alpha\textstyle\sum_jz^{+}_j,\;
  1+\alpha\sum_jz^{-}_j\big)\;:\;\Phi_{-i}z=\phi_i,\;\norm{z}_\infty\le1/\alpha\Big\},
\]
and $\kappa_i(\alpha)$ is non-decreasing in $\alpha$. As $\alpha\to0$ an active state converges
to an inactive one, the hulls touch, and no uniform margin exists.
\end{corollary}

\begin{remark}[The affine level is a support-level worst case, not a weighted average]
\label{rem:notdistributional}
It would be convenient to call Corollary~\ref{cor:kappa-alpha} a distribution-aware frontier. It
is not, and the distinction matters. By Proposition~\ref{prop:knapsack} two state laws with the
same support family and the same minimum amplitude give the \emph{same} $\kappa_i(\alpha)$,
however differently they weight the states in between. The affine level therefore reads exactly
one scalar off the state law --- the smallest detectable amplitude --- whereas by
Proposition~\ref{prop:secondmoment} the analog level reads the full second moment $C$. That
asymmetry is a property of the two criteria, not an artefact of our formulation: separability is
a statement about a worst case over a support family, and a worst case cannot be reweighted. A
genuinely probabilistic version, guaranteeing separability for all but a $\delta$ fraction of
states, is possible but needs machinery we do not open here.
\end{remark}

\subsection{Margin, certificate, and a dimensional ceiling}
\label{sec:frontier-margin}

The verdict of Theorem~\ref{thm:frontier} is binary. Two further quantities say how comfortably
it holds and how it fails, and one says that no code can do well beyond a certain sparsity. The
first two are stated for the exactly-$s$ parameterisation in which they were derived and
verified, where the admissible set carries the affine constraint $\mathbf 1^{\top}z=1$ and an
$\ell_1$ budget; we flag that rather than restate them without proof.

\begin{proposition}[The robust margin is half the hull distance]
\label{prop:margin}
Let $\delta_i(s)=\operatorname{dist}(\mathcal A_i(s),\mathcal B_i(s))$ be the Euclidean distance
between the two hulls. The maximum-margin affine decoder for feature $i$ has
$\norm{w^{\star}}=1/\delta_i(s)$ and margin $\gamma_i(s)=1/(2\norm{w^{\star}})=\delta_i(s)/2$.
Consequently, with the optimal unit-norm $w$ and its centred bias, every perturbed input
$x+\eta$ with $\norm{\eta}_2<\gamma_i(s)$ still receives the correct label on every admissible
support.
\end{proposition}

\begin{proof}
The margin programme is $\min\norm{w}$ subject to $\min_{u\in D}w^{\top}u\ge1$ with
$D=\mathcal A_i(s)-\mathcal B_i(s)$. Support-function duality gives
$\max_{\norm{w}\le1}\min_{u\in D}w^{\top}u=\operatorname{dist}(0,D)=\delta_i(s)$, whence
$\norm{w^{\star}}=1/\delta_i(s)$. For the perturbation claim, the score moves by at most
$\norm{w}\norm{\eta}$.
\end{proof}

The margin is therefore a certified noise tolerance in the units of the representation, not
merely a scale, and $\delta_i(s)=0$ holds exactly when the feature is not separable --- the
verdict and the margin are two readings of one object.

\begin{proposition}[Failure comes with a finite, independently checkable certificate]
\label{prop:farkas}
If feature $i$ is not separable at sparsity $s$, there are states
$x^{+}_1,\dots,x^{+}_m\in\mathcal A_i(s)$, states $x^{-}_1,\dots,x^{-}_n\in\mathcal B_i(s)$ and
convex weights $\lambda,\mu$ with
\begin{equation}
  \sum_{k}\lambda_kx^{+}_k\;=\;\sum_{l}\mu_lx^{-}_l .
  \label{eq:farkas}
\end{equation}
No affine rule can label all of them correctly, since affinity would force the same point to lie
strictly on both sides of $\theta$. By Radon's theorem the obstruction can be compressed to at
most $d+2$ states.
\end{proposition}

\begin{remark}[What the certificate is, and what it is not]
\label{rem:cert}
Equation~\eqref{eq:farkas} is a convex combination of active states equal to a convex combination
of inactive ones. It is \emph{not} a single pair of colliding supports, and describing it as one
would be wrong: the obstruction is generically spread over several states on each side. This
matters practically, because the certificate is what makes the negative verdict auditable. ``The
solver reported infeasible'' is not evidence; a list of supports from which any third party can
rebuild the states from $\Phi$ and recompute~\eqref{eq:farkas} is. Our implementation stores the
supports rather than the cut vectors for exactly that reason, and re-verifies them in a process
that never sees the solver.
\end{remark}

\begin{proposition}[A dimensional ceiling on the affine level]
\label{prop:circuit}
Call a set $C$ of columns an \emph{affine circuit} if it is minimally affinely dependent:
$\sum_{j\in C}l_j\phi_j=0$ with $\sum_{j\in C}l_j=0$ and every $l_j\neq0$. Let $q$ be the
smallest size of an affine circuit of $\Phi$. Since any $d+2$ points in $\R^d$ are affinely
dependent, $q\le d+2$ whenever $F>d+1$. Moreover the exactly-$s$ collision radius satisfies
$\rho_{\min}\le q-1\le d+1$, and hence no code supports featurewise affine support recovery
beyond sparsity $\lceil d/2\rceil$.
\end{proposition}

\begin{proof}
Given a circuit $C$, pick $i\in C$ with $|l_i|$ largest and set $z_j=-l_j/l_i$ for
$j\in C\setminus\{i\}$ and $z_j=0$ otherwise. Then $\Phi_{-i}z=\phi_i$; the affine constraint
$\mathbf 1^{\top}z=1$ holds automatically, because $\sum_{j\in C}l_j=0$ forces
$\sum_{j\neq i}l_j=-l_i$; $\norm{z}_\infty\le1$ by the choice of $i$; and
$\norm{z}_1\le q-1$. Hence $\rho_i\le q-1$. The sparsity bound follows from the threshold
$2\min(s,F-s)-1$ of the exactly-$s$ rule.
\end{proof}

This is the affine counterpart of Theorem~\ref{thm:floor}: one bounds error from below, the other
bounds sparsity from above, and both follow from the ambient dimension alone.


\section{Released dictionaries are not near the floor}
\label{sec:dictionaries}

A sparse autoencoder's decoder is an overcomplete code in exactly the sense of
Section~\ref{sec:setting}: $F$ unit-norm directions in $d<F$ dimensions, used in practice by taking
inner products with feature directions. It is therefore an object the diagnostic applies to without
adaptation, and it is one that people deploy. This section measures \SaeDicts{} of them.

\paragraph{The gauge, stated before anything is measured.} Leverage is not invariant to rescaling
a column, so a measurement of it is meaningless without a declared normalisation. Every decoder
analysed here, and every control it is compared against, is normalised to unit column norms before
anything is computed (\code{unit\_cols} in \code{scripts/sae\_diagnostic.py}). That is the same
convention the theory of Section~\ref{sec:floor} assumes, and it is what makes the comparison against
a shape-matched control a comparison of directions rather than of scales. Without it the excess we
report could be manufactured: shrinking a handful of columns of an i.i.d.\ code, directions untouched,
moves its ratio further than any dictionary in this section sits above its control.

A residual of the same kind survives the gauge and we do not resolve it. A dead latent --- a decoder
column whose encoder never fires on the data --- still has a direction, still enters the row space,
and still moves the leverage of every other feature, while changing nothing the model computes. So
two functionally identical autoencoders can receive different scores here. Fixing that needs matched
activation data to define the active support, which this section deliberately does not use; we return
to it in Section~\ref{sec:limitations}.

\paragraph{What is and is not computed.} Everything below comes from a decoder matrix. No model
weights are loaded, no activations are collected, no autoencoder is trained, and no linear programme
is solved. The cost is seconds per dictionary, which is why the section can span five model families
rather than one. The price is that these are statements about the geometry a dictionary presents to a
linear reader, not about what its features mean or how well it reconstructs; the reconstruction
metrics its authors report are a different and complementary measurement.

\paragraph{Design.} Four axes, chosen so that each isolates one variable
(Table~\ref{tab:sae}). \emph{Scale}: Qwen-Scope residual dictionaries at $d=2048$ and $d=4096$ with
load and recipe held at $16\times$ and TopK. \emph{Load}: EleutherAI's Pythia-160m dictionaries at
one layer with $k$ fixed at 64 and width varying, giving \SaeLoadOneLoad$\times$ to
\SaeLoadThreeLoad$\times$ at constant $d$. \emph{Causal}: EleutherAI's SmolLM2-135M pair, one arm
trained on the model and one on a randomly initialised copy of it, same recipe, same data, same ten
layers. \emph{Architecture}: the MLP down-projection of LFM2.5-350M, a hybrid convolution/attention
model rather than a standard transformer.

Every code is compared against an i.i.d.\ Gaussian code of \emph{its own} shape, evaluated at
\emph{its own} published operating sparsity. That matching is not cosmetic. Because $\sum_i h_i=d$,
the mean leverage of any code is exactly $d/F$, so absolute quantities like the fraction of features
a bound rules out are properties of the shape and not of the dictionary. Only the departure from a
shape-matched control carries information about the code.

\subsection{Near-floor geometry is generic}

Each control is a single i.i.d.\ draw at its block's shape rather than a Monte Carlo ensemble.
That is deliberate and it is defensible: by Proposition~\ref{prop:leverage} the ratio is a function
of the leverage profile, whose concentration at these dimensions is tight enough that the whole set
of controls spans less than \SaeCtlRgeomMax{} $-$ \SaeCtlRgeomMin{} in $\Rgeom$ across every shape
and load we measure. A reader who wants the sampling spread rather than our word for it can generate
it in seconds from the released script.

The \SaeControls{} matched controls attain the floor: $\Rgeom$ from \SaeCtlRgeomMin{} to
\SaeCtlRgeomMax{} across every shape and every load from \SaeLoadMin$\times$ to
\SaeLoadMax$\times$, with leverage dispersion \SaeCtlCvMin{} to \SaeCtlCvMax{}. This is
Proposition~\ref{prop:leverage} in action rather than a coincidence: an i.i.d.\ code concentrates its
leverage tightly around $d/F$, and a code with equalised leverage sits at the floor by the identity.

The consequence for how such ratios are read is direct. A reported attainment ratio near one, under a
pseudoinverse readout, is consistent with a code that has learned nothing at all. It becomes evidence
about training only when compared against a shape-matched baseline, and that comparison is the first
thing we could not find in the work we build on.

It is worth seeing how little the ratio constrains a code, because the identity makes the extreme
case constructible. Stack $k$ copies of the same orthonormal basis, $\Phi=[I_d\;I_d\;\cdots\;I_d]$
with $F=kd$. Then $\Phi\Phi^{\top}=k I_d$, every leverage score is exactly $1/k=d/F$, and by
Proposition~\ref{prop:leverage} the code attains the floor exactly: $\Rgeom=1$. Yet features $i$ and
$i+d$ have identical directions, so no readout of any kind --- linear, affine or otherwise --- can
tell which of the two is active. A code can therefore sit exactly on the rank--trace floor and be
maximally unidentifiable. Attaining the floor is not a certificate of anything; it is the statement
that the unavoidable cross-talk is spread evenly, and evenness is compatible with duplication. This
cuts both ways, and we take the cut: it is why we report the excess above a matched control as a
measured difference in leverage dispersion, and not as a claim that one code is ``better conditioned''
than another.

\subsection{No released dictionary we measured is at the floor}

All \SaeDictsAboveControl{} of the \SaeDicts{} dictionaries sit above their own control, at $\Rgeom$
from \SaeDictRgeomMin{} to \SaeDictRgeomMax{}. There is no exception across two decoder sites
(residual stream and MLP), three training recipes (TopK, JumpReLU-style and a hybrid model's own
projection), five model families, and a twentyfold range of load.

The direction is worth pausing on, because it is the opposite of the natural guess. One might expect
a dictionary trained to represent a model's features to be at least as well arranged for linear
readout as an unstructured code. It is not: a random code of the same shape carries strictly less
cross-talk under the best linear readout each admits. Whatever these dictionaries optimise --- sparse
reconstruction of activations, under an $\ell_0$ or TopK constraint --- it is not the evenness of
their own leverage, and the two objectives are not aligned.

By the identity, the excess has a single source: leverage dispersion, \SaeDictCvMin{} to
\SaeDictCvMax{} against the controls' \SaeCtlCvMin{} to \SaeCtlCvMax{}. A released dictionary
allocates its directions unevenly with respect to the space it lives in, and a handful of features
with small $h_i$ dominate the harmonic sum that sets the ratio.

\paragraph{The second statistic agrees, in sparsity units.} The diagnostic returns a failure
threshold as well as a ratio: the smallest sparsity at which Corollary~\ref{cor:failthresh}
guarantees that some feature cannot be recovered by any affine rule. Its absolute value is a property
of the shape --- at the operating sparsities these dictionaries are used at, the sufficient condition
rules out essentially every feature of the controls too --- so it carries information only as a
matched difference. Matched, it is unanimous: all \SaeFailEarlier{} of \SaeDicts{} dictionaries reach
guaranteed failure \emph{earlier} than the i.i.d.\ code of their own shape and operating sparsity, by
\SaeFailGapMin{} to \SaeFailGapMax{} sparsity levels, with the trained thresholds running from
\SaeFailTrainedMin{} to \SaeFailTrainedMax{}. This is a different functional of the leverage
distribution from $\Rgeom$ --- a worst-case statement about one feature rather than an average over
all of them --- and it points the same way, which is the main reason we report it.

\paragraph{Depth matters more than width, but not in one clean ordering.} The excess is far from
uniform across a model, and within a suite it varies by more than any width effect we measure below.
In the three Qwen suites the pattern is clean: mid-depth dictionaries sit furthest from the floor and
the deepest layer sampled sits closest --- $\Rgeom=\SaeDictRgeomMax{}$ at layer 17 of Qwen3-8B
against \SaeQwenDeepestRgeom{} at its own last layer. In the other two suites it is not. The deepest
LFM2.5 layer is \SaeLfmDeepestRgeom{} while the closest to the floor is the layer before it
(\SaeLfmClosestRgeom{}), and the deepest SmolLM2 layer is \SaeSmolDeepestRgeom{} against a minimum of
\SaeSmolClosestRgeom{} at layer 3. So the safe statement is the one the data supports in every
suite: the layer a dictionary is trained on changes its linear readability by far more than its
width does, and the direction of that change is not the same in every family.

\subsection{A randomly initialised control shows where the excess comes from}

The obvious objection to the previous subsection is that training a TopK autoencoder might produce
uneven leverage on any input, in which case the excess would say nothing about the model. An i.i.d.\
control cannot answer this, because it does not share the training procedure.

EleutherAI released the pair that does. Two dictionaries trained by an identical recipe on identical
data, one on SmolLM2-135M and one on a \emph{randomly initialised} copy of it, at the same ten layers
and the same $64\times$ expansion. The randomly initialised arm returns to the floor:
$\Rgeom=\SaeRandInitRgeomMin$ to \SaeRandInitRgeomMax{} with dispersion \SaeRandInitCvMin{} to
\SaeRandInitCvMax{}, against the trained arm's excess in all \SaePairedLayers{} paired layers. Its
values are also nearly flat across depth, where the trained arm has clear structure.

Fitting a sparse autoencoder therefore does not, by itself, produce dispersed leverage. The excess
tracks what the model represents. This is the strongest causal statement the section supports, and it
rests on one model pair, which is the only such pair we are aware of --- a limitation we return to in
Section~\ref{sec:limitations}.

\subsection{Load and dispersion come apart}

Holding $d$ and $k$ fixed and varying only width, the ratio is flat while the dispersion grows:
$\Rgeom$ of \SaeLoadOneRgeom, \SaeLoadTwoRgeom{} and \SaeLoadThreeRgeom{} at loads of
\SaeLoadOneLoad$\times$, \SaeLoadTwoLoad$\times$ and \SaeLoadThreeLoad$\times$, while
$\mathrm{cv}(h)$ climbs from \SaeLoadOneCv{} through \SaeLoadTwoCv{} to \SaeLoadThreeCv{}.

Without the identity this reads as an anomaly. With it, it is a prediction: $\Rgeom$ is set by
$\sum_i h_i^{-1}$, a harmonic quantity dominated by the smallest leverage scores, whereas the
coefficient of variation is a second-moment summary sensitive to the upper tail. Extra spread at
higher load can therefore accumulate in a part of the distribution the ratio does not see. Across all
\SaeDicts{} dictionaries the excess is ranked almost perfectly by the harmonic-to-arithmetic gap of
the leverage and only well by the dispersion, which is the same statement measured rather than
derived.

Two things follow for practice. Widening a dictionary at fixed $d$ does not improve --- and does not
much worsen --- how well it can be read linearly. And the coefficient of variation, the obvious
summary to reach for, is the wrong one: it moves when the ratio does not.

\begin{table}[t]
\centering
\caption{The Interface Diagnostic on \SaeDicts{} released dictionaries, grouped by the axis each
block isolates, with every i.i.d.\ control matched to the shape and operating sparsity of the block
it judges. $s$ is the dictionary's own published operating sparsity. The
\emph{trained against randomly initialised} block is a matched intervention rather than an
observation: both arms come from the same recipe, data and layers, and differ only in whether the
underlying model was trained. Generated from \texttt{results/sae/derived/sae\_axes.json}; the SHA-256
of every checkpoint analysed is recorded there.}
\label{tab:sae}
\small
\begin{tabular}{llrrrrrr}
\toprule
code & & $d$ & $F$ & $F/d$ & $s$ & $R_{\mathrm{geom}}$ & $\mathrm{cv}(h)$ \\
\midrule
\multicolumn{8}{l}{\itshape Scale: $d$ varies, load fixed at $16\times$} \\
Qwen3-1.7B layer1 &  & 2048 & 32768 & 16.0 & 50 & 1.1371 & 0.356 \\
Qwen3-1.7B layer18 &  & 2048 & 32768 & 16.0 & 50 & 1.1218 & 0.274 \\
Qwen3-1.7B layer27 &  & 2048 & 32768 & 16.0 & 50 & 1.0055 & 0.102 \\
Qwen3-1.7B layer9 &  & 2048 & 32768 & 16.0 & 50 & 1.1849 & 0.331 \\
Qwen3-8B L17 &  & 4096 & 65536 & 16.0 & 50 & 1.2018 & 0.352 \\
Qwen3-8B L30 &  & 4096 & 65536 & 16.0 & 50 & 1.0178 & 0.126 \\
Qwen3-8B \_L4 &  & 4096 & 65536 & 16.0 & 50 & 1.0812 & 0.262 \\
Qwen3.5-2B layer1 &  & 2048 & 32768 & 16.0 & 50 & 1.0336 & 0.180 \\
Qwen3.5-2B layer16 &  & 2048 & 32768 & 16.0 & 50 & 1.0798 & 0.259 \\
Qwen3.5-2B layer23 &  & 2048 & 32768 & 16.0 & 50 & 1.0306 & 0.187 \\
Qwen3.5-2B layer8 &  & 2048 & 32768 & 16.0 & 50 & 1.1704 & 0.356 \\
i.i.d. 2048x32768 & control & 2048 & 32768 & 16.0 & 50 & 1.0001 & 0.008 \\
i.i.d. 4096x65536 & control & 4096 & 65536 & 16.0 & 50 & 1.0000 & 0.005 \\
\addlinespace[3pt]
\multicolumn{8}{l}{\itshape Load: $F/d$ varies, $d$ and $k$ fixed} \\
Pythia-160m 4k &  & 768 & 4084 & 5.3 & 64 & 1.0674 & 0.276 \\
Pythia-160m 32k &  & 768 & 32768 & 42.7 & 64 & 1.0749 & 0.404 \\
Pythia-160m 65k &  & 768 & 65536 & 85.3 & 64 & 1.0734 & 0.469 \\
i.i.d. 768x4084 & control & 768 & 4084 & 5.3 & 64 & 1.0005 & 0.020 \\
i.i.d. 768x32768 & control & 768 & 32768 & 42.7 & 64 & 1.0001 & 0.008 \\
i.i.d. 768x65536 & control & 768 & 65536 & 85.3 & 64 & 1.0000 & 0.005 \\
\addlinespace[3pt]
\multicolumn{8}{l}{\itshape Trained model against a randomly initialised one} \\
SmolLM2 random L0 &  & 576 & 36864 & 64.0 & 64 & 1.0038 & 0.060 \\
SmolLM2 random L12 &  & 576 & 36864 & 64.0 & 64 & 1.0063 & 0.077 \\
SmolLM2 random L15 &  & 576 & 36864 & 64.0 & 64 & 1.0063 & 0.077 \\
SmolLM2 random L18 &  & 576 & 36864 & 64.0 & 64 & 1.0064 & 0.078 \\
SmolLM2 random L21 &  & 576 & 36864 & 64.0 & 64 & 1.0063 & 0.077 \\
SmolLM2 random L24 &  & 576 & 36864 & 64.0 & 64 & 1.0065 & 0.078 \\
SmolLM2 random L27 &  & 576 & 36864 & 64.0 & 64 & 1.0065 & 0.078 \\
SmolLM2 random L3 &  & 576 & 36864 & 64.0 & 64 & 1.0045 & 0.066 \\
SmolLM2 random L6 &  & 576 & 36864 & 64.0 & 64 & 1.0057 & 0.074 \\
SmolLM2 random L9 &  & 576 & 36864 & 64.0 & 64 & 1.0061 & 0.076 \\
SmolLM2 trained L0 &  & 576 & 36864 & 64.0 & 64 & 1.0512 & 0.307 \\
SmolLM2 trained L12 &  & 576 & 36864 & 64.0 & 64 & 1.0302 & 0.248 \\
SmolLM2 trained L15 &  & 576 & 36864 & 64.0 & 64 & 1.1160 & 0.351 \\
SmolLM2 trained L18 &  & 576 & 36864 & 64.0 & 64 & 1.1351 & 0.397 \\
SmolLM2 trained L21 &  & 576 & 36864 & 64.0 & 64 & 1.0253 & 0.198 \\
SmolLM2 trained L24 &  & 576 & 36864 & 64.0 & 64 & 1.0355 & 0.279 \\
SmolLM2 trained L27 &  & 576 & 36864 & 64.0 & 64 & 1.0319 & 0.211 \\
SmolLM2 trained L3 &  & 576 & 36864 & 64.0 & 64 & 1.0105 & 0.177 \\
SmolLM2 trained L6 &  & 576 & 36864 & 64.0 & 64 & 1.0216 & 0.190 \\
SmolLM2 trained L9 &  & 576 & 36864 & 64.0 & 64 & 1.0207 & 0.198 \\
i.i.d. 576x36864 & control & 576 & 36864 & 64.0 & 64 & 1.0001 & 0.007 \\
\addlinespace[3pt]
\multicolumn{8}{l}{\itshape Hybrid convolution/attention MLP code} \\
LFM2.5-350M L0 (conv) &  & 1024 & 4608 & 4.5 & 64 & 1.0110 & 0.101 \\
LFM2.5-350M L1 (conv) &  & 1024 & 4608 & 4.5 & 64 & 1.0453 & 0.247 \\
LFM2.5-350M L10 (attn) &  & 1024 & 4608 & 4.5 & 64 & 1.0510 & 0.239 \\
LFM2.5-350M L11 (conv) &  & 1024 & 4608 & 4.5 & 64 & 1.0330 & 0.204 \\
LFM2.5-350M L12 (attn) &  & 1024 & 4608 & 4.5 & 64 & 1.0119 & 0.122 \\
LFM2.5-350M L13 (conv) &  & 1024 & 4608 & 4.5 & 64 & 1.0247 & 0.196 \\
LFM2.5-350M L14 (attn) &  & 1024 & 4608 & 4.5 & 64 & 1.0083 & 0.090 \\
LFM2.5-350M L15 (conv) &  & 1024 & 4608 & 4.5 & 64 & 1.0240 & 0.167 \\
LFM2.5-350M L2 (attn) &  & 1024 & 4608 & 4.5 & 64 & 1.0179 & 0.133 \\
LFM2.5-350M L3 (conv) &  & 1024 & 4608 & 4.5 & 64 & 1.0183 & 0.158 \\
LFM2.5-350M L4 (conv) &  & 1024 & 4608 & 4.5 & 64 & 1.0605 & 0.262 \\
LFM2.5-350M L5 (attn) &  & 1024 & 4608 & 4.5 & 64 & 1.0267 & 0.154 \\
LFM2.5-350M L6 (conv) &  & 1024 & 4608 & 4.5 & 64 & 1.0197 & 0.149 \\
LFM2.5-350M L7 (conv) &  & 1024 & 4608 & 4.5 & 64 & 1.0528 & 0.242 \\
LFM2.5-350M L8 (attn) &  & 1024 & 4608 & 4.5 & 64 & 1.0330 & 0.168 \\
LFM2.5-350M L9 (conv) &  & 1024 & 4608 & 4.5 & 64 & 1.0626 & 0.297 \\
i.i.d. 1024x4608 & control & 1024 & 4608 & 4.5 & 64 & 1.0004 & 0.018 \\
\addlinespace[3pt]
\bottomrule
\end{tabular}

\end{table}

\section{What training buys: controlled campaigns}
\label{sec:campaigns}
\label{sec:res-primary}

Released dictionaries tell us what trained codes look like; they cannot tell us what any particular
choice during training does to them, because nothing about them was under our control. This section
supplies the intervention. We train 400 small networks in which the loss, the width, the feature load
and the training sparsity are set by us, measure the same statistics on the resulting codes, and
retain every trained model so that later analyses need no retraining --- three of the results below
were computed months after the runs, from the saved weights alone.

\subsection{The network against an affine probe of its own representation}

The campaigns in Appendix~\ref{app:earlier} compare the network's thresholded output against
\emph{fixed} readouts of its code. That understates the baseline, and the audit that prompted this
campaign said so. Stage A replaces it: 180 models at $d\in\{50,100,200\}$ with $F=2d$, 20 seeds per cell, three code families
($L^4$, $L^2$, and an untrained random code), and an affine probe fitted over three families, two
representations, three threshold policies and a six-point regularisation grid, all selected on
validation with the test split opened once. \code{tests/test\_probes.py} poisons the test set and
asserts that no selected hyperparameter moves, so leakage is excluded structurally rather than by
inspection.

\paragraph{Splits, and where they cannot be disjoint.} Train, validation and test draw
\emph{disjoint supports}, which is a stronger guarantee than disjoint states and is what makes a
fitted probe's test score meaningful. It cannot hold unconditionally: at $s=1$ there are only $F$
supports in total --- 100, 200 and 400 at the three widths --- against the thousands of states each
split requests. The sampler therefore serves the test split first, caps the smallest sparsities at
what the space contains, and distributes the remainder proportionally; it records which sparsities
were exhausted and how many draws were rejected, and it raises rather than silently returning an
empty split. At $F=100$ the $s=1$ test set holds 50 states, not 4096, and the sparsities $s\in\{1,2\}$
are flagged exhausted. Every per-sparsity rate at those sparsities is therefore computed on fewer
states than the nominal budget, and its Wilson interval is correspondingly wider.

\paragraph{The comparison has to be matched, and matching it changes the answer.}
A first pass compared the network at a fixed $\theta=1/2$ against the best of all eighteen probe
configurations. That stacks three advantages on the probe: two thirds of the configurations tune
their threshold on validation while the network does not; two thirds read the \emph{pre}-ReLU state
$\Phi b$ while the network's output layer reads $\mathrm{ReLU}(\Phi b)$; and $\theta=1/2$ is not on a
common scale across a ridge prediction, a logistic probability and a hinge margin. Removing all
three --- probe restricted to the post-ReLU state, one validation-selected global threshold for each
side, differences paired within seed --- gives Table~\ref{tab:primary}.

\begin{table}[t]
\centering
\caption{The matched comparison: network minus best post-ReLU affine probe in $s_{95}$, paired
within seed, with a bootstrap interval over the models of each cell (20 at $d\le200$, \ScSeeds{} at
$d=400$). Both decoders read the same state and each gets one validation-selected global threshold.
The sign is a property of the code and not of the protocol: positive for $L^4$ and growing with
width, exactly zero for $L^2$ --- where both decoders are at $s_{95}=0$ and the zero is a joint
failure rather than an agreement --- and strongly negative for an untrained code, under the identical
procedure. Section~\ref{sec:probes} states what the comparison can and cannot mean. Generated from
\texttt{results/e7/derived/} and \texttt{results/e7\_stageC/derived/}.}
\label{tab:primary}
\small
\begin{tabular}{llcccc}
\toprule
code & $d$ & mean difference & 95\% interval & net / tie / probe & AUC difference \\
\midrule
$L^4$ & 50 & +0.00 & [+0.00, +0.00] & 0 / 20 / 0 & +0.0053 \\
$L^4$ & 100 & +0.35 & [+0.05, +0.60] & 9 / 9 / 2 & +0.0210 \\
$L^4$ & 200 & +0.35 & [+0.10, +0.60] & 6 / 14 / 0 & +0.0379 \\
$L^4$ & 400 & +2.00 & [+1.50, +2.50] & 10 / 0 / 0 & +0.0699 \\
\addlinespace[3pt]
$L^2$ & 50 & +0.00 & [+0.00, +0.00] & 0 / 20 / 0 & +0.0007 \\
$L^2$ & 100 & +0.00 & [+0.00, +0.00] & 0 / 20 / 0 & +0.0004 \\
$L^2$ & 200 & +0.00 & [+0.00, +0.00] & 0 / 20 / 0 & +0.0007 \\
$L^2$ & 400 & +0.00 & [+0.00, +0.00] & 0 / 10 / 0 & +0.0008 \\
\addlinespace[3pt]
untrained & 50 & -1.00 & [-1.00, -1.00] & 0 / 0 / 20 & -0.2122 \\
untrained & 100 & -2.00 & [-2.00, -2.00] & 0 / 0 / 20 & -0.2098 \\
untrained & 200 & -2.30 & [-2.50, -2.10] & 0 / 0 / 20 & -0.1805 \\
untrained & 400 & -2.60 & [-3.10, -2.10] & 0 / 0 / 10 & -0.1782 \\
\bottomrule
\end{tabular}

\end{table}

Across the \PrimModels{} trained models the two tie in \PrimTies, with \PrimNetWins{} models where
the network leads and \PrimProbeWins{} where the probe does. That single count needs breaking down
before it is read, because \PrimLtwoJointFailures{} of the ties are $L^2$ cells in which both
decoders sit at $s_{95}=0$: those are joint failures, not agreements, and counting them as ties
inflates the apparent evidence for equivalence. Among the \PrimLfourModels{} $L^4$ models the split
is \PrimLfourTies{} ties, \PrimNetWins{} to the network and \PrimProbeWins{} to the probe, and the
difference grows with width.

The registered outcome was that a properly trained, validation-selected affine decoder would
\emph{match} the network. Under $s_{95}$ it does at the smallest width and stops doing so as width
grows. Under the other co-primary estimand it does not hold even there: the AUC difference at $d=50$
is \AucFiftyDiff{} ([\AucFiftyCiLow, \AucFiftyCiHigh]) on \AucFiftyNetWins{} of \EsevenSeeds{} seeds,
so the registered prediction fails at every width once the finer statistic is used, and what $s_{95}$
shows at $d=50$ is a difference below its own resolution rather than a match. We
report that, and Section~\ref{sec:probes} states what the comparison can and cannot mean, because
the probe class contains the network's own decoder and the gap therefore measures the fitting
procedure rather than the expressiveness of affine maps.

\paragraph{The sign is set by the code, not by the protocol.} The obvious objection to the paragraph
above is that a protocol giving both sides a validation-selected threshold must end up favouring the
network, which has one output layer to calibrate against the probe's eighteen configurations. The
untrained arm answers it. Run through the identical procedure, an untrained code's own thresholded
readout \emph{loses} to the fitted affine probe by \PrimRandFiftyDiff{}, \PrimRandHundredDiff{},
\PrimRandTwoHundredDiff{} and \PrimRandFourHundredDiff{} sparsity levels at the four widths, on every
one of the \PrimRandTwoHundredProbeWins{} models per cell at $d\le200$ and all
\PrimRandFourHundredProbeWins{} at $d=400$, with AUC differences of \PrimRandFiftyAucDiff{} to
\PrimRandFourHundredAucDiff{} --- two orders of magnitude larger than any $L^4$ cell and in the
opposite direction. A procedure that systematically favoured the network could not produce that.

So across the three code families the same protocol gives three different answers, and which one it
gives is decided by how the code was made: the network leads under $L^4$ and increasingly with width,
neither decoder recovers anything under $L^2$, and the probe leads by a wide margin on a code that
was never trained. Read as a decoder comparison this table says little; read as a property of the
code, the ordering is the result.

\paragraph{The second estimand agrees, and the earlier disagreement was an artefact.} The analysis
plan named two co-primary estimands: $s_{95}$ and the recovery AUC, the area under the exact-recovery
curve over the evaluated sparsities. $s_{95}$ is a floored crossing point and therefore blunt, so a
difference of less than one level is invisible to it; the AUC sees the whole curve. On the same
matched post-ReLU comparison the AUC is positive at every width and grows with it:
\AucFiftyDiff{} (95\% interval [\AucFiftyCiLow, \AucFiftyCiHigh]) on \AucFiftyNetWins{} of
\EsevenSeeds{} seeds at $d=50$; \AucHundredDiff{} ([\AucHundredCiLow, \AucHundredCiHigh]) on
\AucHundredNetWins{} of \EsevenSeeds{} at $d=100$; \AucTwoHundredDiff{}
([\AucTwoHundredCiLow, \AucTwoHundredCiHigh]) on \AucTwoHundredNetWins{} of \EsevenSeeds{} at
$d=200$; and \AucFourHundredDiff{}
([\AucFourHundredCiLow, \AucFourHundredCiHigh]) on \AucFourHundredNetWins{} of \ScSeeds{} at
$d=400$. Every cell is unanimous across seeds and no interval covers zero.

An earlier version of this section reported the opposite: an AUC that changed sign with width, which
we described as the two estimands disagreeing and treated as the finding. That reading was produced
by a defect in how the shared threshold was selected (Appendix~\ref{app:withdrawn}), which cost the
network up to five sparsity levels at the widest setting. With the defect fixed the two estimands
concur, and their concurrence is the more informative outcome: a difference visible to a blunt
statistic and to a fine one, in the same direction, is harder to attribute to the choice of summary.
Figure~\ref{fig:e7primary}(a,b)
shows both estimands at all four widths so that a reader can see the agreement rather than take our
word for it.

\paragraph{The nonlinearity costs almost nothing on a trained code, and a great deal on an
untrained one.} Restoring the pre-ReLU probe --- which reads the preactivation $\Phi b$, a state the
network's output layer never sees --- is how we had previously located the effect, and on the $L^4$
codes it now finds very little. The pre-ReLU probe's advantage is \PrimLfourFiftyPreDiff{} of a
sparsity level at $d=50$ with \PrimLfourFiftyPreTies{} of \EsevenSeeds{} seeds tied,
\PrimLfourHundredPreDiff{} at $d=100$, \PrimLfourTwoHundredPreDiff{} at $d=200$, and at $d=400$ it
reverses sign: the network leads even the pre-ReLU probe by \PrimLfourFourHundredPreDiff{} on
\PrimLfourFourHundredPreNetWins{} of \ScSeeds{} models. An earlier version of this paper read a
full-level gap here as a measurement of what the ReLU discards; that reading was an artefact of the
threshold defect and is withdrawn (Appendix~\ref{app:withdrawn}, item 2).

On the untrained arm the same comparison is an order of magnitude larger and unanimous:
\PrimRandFiftyPreDiff{}, \PrimRandHundredPreDiff{}, \PrimRandTwoHundredPreDiff{} and
\PrimRandFourHundredPreDiff{} levels to the pre-ReLU probe, growing steeply with width. So the
pipeline of ReLU plus a single global threshold throws away a large and growing amount of affinely
accessible information on a code that was never trained, and next to none on one trained under
$L^4$. That is the surviving form of the claim, and it is a statement about the interaction between
the code and the readout rather than about the nonlinearity alone.

\begin{figure}[t]
\centering
\includegraphics[width=\linewidth]{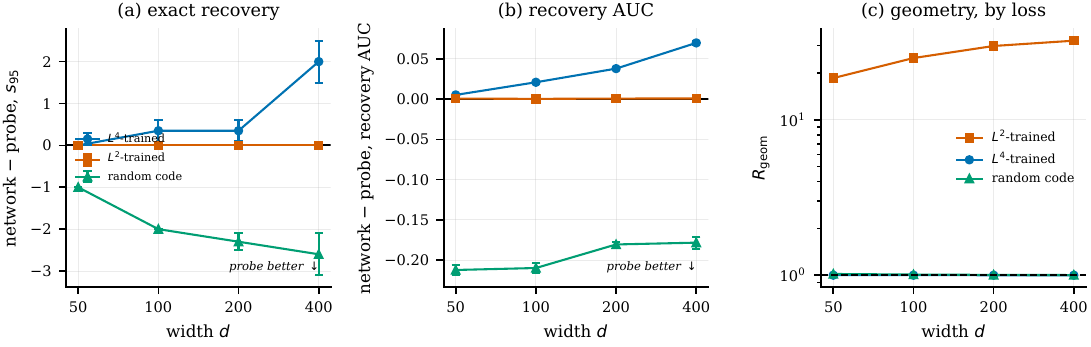}
\caption{\textbf{The matched comparison, both estimands, and the geometry the loss produces.}
Paired within seed; bars are bootstrap intervals over the models of a cell. Negative means the affine
probe decodes better. \textbf{(a)} Under $s_{95}$ the matched post-ReLU comparison separates the three
code families in three directions: the network leads under $L^4$ by an amount that grows with width,
$L^2$ is a flat zero because neither decoder recovers anything, and on an untrained code the probe
leads by one to two and a half levels. The protocol is identical in all three.
\textbf{(b)} The same models under the recovery AUC, which resolves differences smaller than one
sparsity level and agrees with $s_{95}$ at every width --- an earlier version reported a sign reversal
here, which was an artefact of the threshold defect (Appendix~\ref{app:withdrawn}).
\textbf{(c)} $R_{\mathrm{geom}}$ by loss. $L^2$ leaves the interface one to two orders of magnitude
above the rank--trace floor; $L^4$ and an \emph{untrained} random code both sit on it, which is why
proximity to the floor is not by itself evidence of learned structure.}
\label{fig:e7primary}
\end{figure}

\paragraph{The loss decides the interface.} The three code families separate cleanly and by a wide
margin, at every width:

\begin{center}\small
\begin{tabular}{lccc}
\toprule
& $d=50$ & $d=100$ & $d=200$ \\
\midrule
$R_{\mathrm{geom}}$, $L^4$ & \SALfourFiftyRgeom & \SALfourHundredRgeom & \SALfourTwoHundredRgeom \\
$R_{\mathrm{geom}}$, $L^2$ & \SALtwoFiftyRgeom & \SALtwoHundredRgeom & \SALtwoTwoHundredRgeom \\
$R_{\mathrm{geom}}$, random & \SARandFiftyRgeom & \SARandHundredRgeom & \SARandTwoHundredRgeom \\
\midrule
$\kappa_{\min}$, $L^4$ & \FullLfourFiftyKappaMin & \FullLfourHundredKappaMin & \FullLfourTwoHundredKappaMin \\
$\kappa_{\min}$, $L^2$ & \FullLtwoFiftyKappaMin & \FullLtwoHundredKappaMin & \FullLtwoTwoHundredKappaMin \\
$\kappa_{\min}$, random & \FullRandFiftyKappaMin & \FullRandHundredKappaMin & \FullRandTwoHundredKappaMin \\
\bottomrule
\end{tabular}
\end{center}

$L^2$ drives $R_{\mathrm{geom}}$ into the tens and rising with width, and its frontier collapses to
$1$. By Theorem~\ref{thm:frontier} a feature is separable at $s$ only when $\kappa_i>s$, so a
frontier at $1$ means the worst feature is separable at singletons and at nothing beyond them. Only
at $d=50$, where $\kappa_{\min}=\FullLtwoFiftyKappaMin{}$ exactly, does the bound reach the
degenerate case in which two features share a representation; at the larger widths
$\kappa_{\min}$ is \FullLtwoHundredKappaMin{}, \FullLtwoTwoHundredKappaMin{} and \ScLtwoKappa{}, so
every feature is separable at $s=1$ and fails at $s=2$. The mechanism is a spread of column norms: a
feature whose direction is a fraction of a typical length lies inside the convex hull of the others,
and no duplicate is needed. $L^4$ does not do this, and its
$R_{\mathrm{readout}}$ of \SALfourFiftyRreadout{} says its own output layer is near-optimal for the
code it built.

\paragraph{But near-floor geometry is generic.} The untrained random code reaches
$R_{\mathrm{geom}}=\SARandFiftyRgeom{}$ at $d=50$ and its frontier is
\FullRandTwoHundredKappaMin{} at $d=200$, against $L^4$'s \FullLfourTwoHundredKappaMin{}. So the
distance from the rank--trace floor, which earlier work reads as a signature of learned structure,
is what an i.i.d.\ code gives for free at this load; the trained code is better, by one to two
sparsity levels, and that margin is the honest size of the effect.

\paragraph{And the margin comes from the code, not the readout.} The random control differs from a
trained model in two ways at once --- its code is not trained \emph{and} its readout is
$\Phi^{+}$ rather than learned --- so it cannot say which half matters. E8 separates them with three
arms on an identical budget, \EeightSeeds{} seeds each: both trained; a random code held
\emph{frozen} with only $W_{\mathrm{out}}$ trainable; and neither trained.

\begin{center}\small
\begin{tabular}{lccccccccc}
\toprule
& \multicolumn{3}{c}{$d=50$} & \multicolumn{3}{c}{$d=100$} & \multicolumn{3}{c}{$d=200$} \\
\cmidrule(lr){2-4}\cmidrule(lr){5-7}\cmidrule(lr){8-10}
arm & $R_{\mathrm{geom}}$ & $\kappa_{\min}$ & $R_{\mathrm{ro}}$
    & $R_{\mathrm{geom}}$ & $\kappa_{\min}$ & $R_{\mathrm{ro}}$
    & $R_{\mathrm{geom}}$ & $\kappa_{\min}$ & $R_{\mathrm{ro}}$ \\
\midrule
both trained & \EeightTrainedFiftyRgeom & \EeightTrainedFiftyKappa & \EeightTrainedFiftyRreadout
             & \EeightTrainedHundredRgeom & \EeightTrainedHundredKappa & \EeightTrainedHundredRreadout
             & \EeightTrainedTwoHundredRgeom & \EeightTrainedTwoHundredKappa & \EeightTrainedTwoHundredRreadout \\
frozen code  & \EeightFrozenFiftyRgeom & \EeightFrozenFiftyKappa & \EeightFrozenFiftyRreadout
             & \EeightFrozenHundredRgeom & \EeightFrozenHundredKappa & \EeightFrozenHundredRreadout
             & \EeightFrozenTwoHundredRgeom & \EeightFrozenTwoHundredKappa & \EeightFrozenTwoHundredRreadout \\
neither      & \EeightRandomFiftyRgeom & \EeightRandomFiftyKappa & \EeightRandomFiftyRreadout
             & \EeightRandomHundredRgeom & \EeightRandomHundredKappa & \EeightRandomHundredRreadout
             & \EeightRandomTwoHundredRgeom & \EeightRandomTwoHundredKappa & \EeightRandomTwoHundredRreadout \\
\bottomrule
\end{tabular}
\end{center}

\noindent$R_{\mathrm{ro}}$ abbreviates $R_{\mathrm{readout}}$. The frozen arm is at
$d\in\{50,100\}$ from one run and $d=200$ from a second, added after review asked for the third
width; the two carry separate run records rather than being merged.

Training the readout alone leaves the frontier where an untrained code puts it, at all three widths:
\EeightFrozenHundredKappa{} against \EeightRandomHundredKappa{} at $d=100$ and
\EeightFrozenTwoHundredKappa{} against \EeightRandomTwoHundredKappa{} at $d=200$, while training both
reaches \EeightTrainedHundredKappa{} and \EeightTrainedTwoHundredKappa{}. The gain is entirely in
$\Phi$. That is unsurprising for $\kappa$, which is a function of the code alone --- the frozen and
untrained rows must agree up to the column normalisation that distinguishes them, and their agreement
is a consistency check rather than a result --- but it is the point: whatever the training buys for
support recovery, it buys by moving the code.

\begin{figure}[t]
\centering
\includegraphics[width=\linewidth]{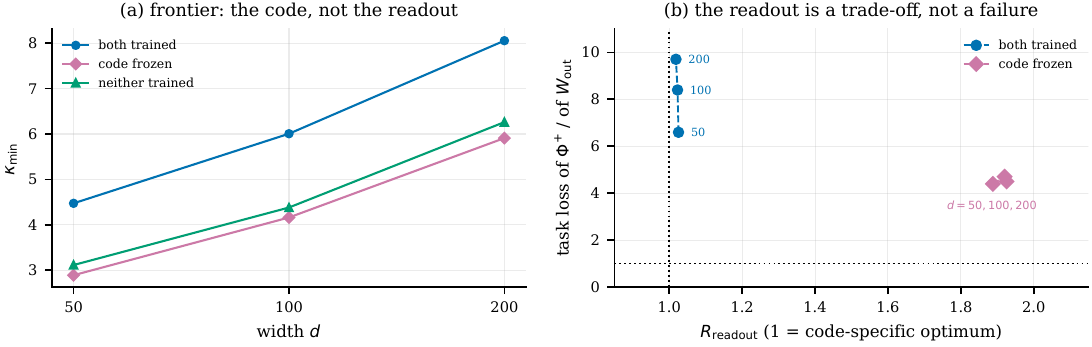}
\caption{\textbf{Which half of the training does the work.} Medians over
\EeightSeeds{} models per arm and width. \textbf{(a)} The frontier tracks the code: training only the
readout leaves $\kappa_{\min}$ where an untrained code puts it --- below it, in fact --- while
training both lifts it at every width. \textbf{(b)} The same models placed by cross-talk against task
loss. The horizontal axis is $R_{\mathrm{readout}}$, so $1$ is the best readout \emph{of the code that
arm was given}; the vertical axis is how much lower the trained readout's own training loss is than
that optimum's. The frozen arm sits at nearly twice the cross-talk, which invites reading it as a
failed optimisation, but it beats the cross-talk optimum on the loss it trained under, so the
cross-talk is bought rather than lost. Its three widths coincide: neither coordinate moves as the
problem grows. The jointly trained arm is better on both axes at once and pulls further ahead with
width, so co-adaptation does not trade the two objectives against each other --- it aligns them.}
\label{fig:e8}
\end{figure}

The readout column is the surprise, and it is stable across width. The frozen arm trained
$W_{\mathrm{out}}$ for the full budget and ended at
$R_{\mathrm{readout}}=\EeightFrozenHundredRreadout{}$ at $d=100$ and
\EeightFrozenTwoHundredRreadout{} at $d=200$ --- nearly twice the cross-talk of the best readout
\emph{of the code it was given} --- while the jointly trained arm reaches
\EeightTrainedHundredRreadout{} and \EeightTrainedTwoHundredRreadout{}.

It is tempting to read that as gradient descent failing at the conditional problem, and it is not.
The cross-talk optimum is attained by $\Phi^{+}$ for \emph{every} code, so a readout with
$R_{\mathrm{readout}}$ near $1$ always exists and its existence settles nothing; what matters is
whether the arm's own objective prefers its $W_{\mathrm{out}}$ to that optimum. It does, decisively.
Scored under the loss the arm trained on, in its own post-ReLU representation, on fresh draws from a
seed no training used (\texttt{scripts/frozen\_readout\_tradeoff.py}), the frozen
$W_{\mathrm{out}}$ beats $\Phi^{+}$ by a factor of \EeightFrozenHundredTaskGain{} at $d=100$ and
\EeightFrozenTwoHundredTaskGain{} at $d=200$, on \EeightFrozenTwoHundredTaskWins{} of
\EeightSeeds{} models at each width. The cross-talk it gives up is bought, not lost. As a check on
the scoring, the untrained arm --- whose readout \emph{is} $\Phi^{+}$ --- returns a ratio of
\EeightRandomTwoHundredTaskGain{}. Figure~\ref{fig:e8}(b) places both arms in the two coordinates at
once, which is the only view in which the trade-off is visible as a trade-off.

So the two objectives genuinely pull apart on a fixed random code, and the contrast is about what
co-adaptation does to them rather than about optimisation difficulty: the jointly trained arm attains
$R_{\mathrm{readout}}=\EeightTrainedTwoHundredRreadout{}$ \emph{and} a task advantage of
\EeightTrainedTwoHundredTaskGain{} at $d=200$, so it does not trade one against the other at all,
while the frozen arm can only buy the second with the first. The joint arm's task advantage also
grows with width, \EeightTrainedFiftyTaskGain{} to \EeightTrainedHundredTaskGain{} to
\EeightTrainedTwoHundredTaskGain{}, where the frozen arm's is flat
(\EeightFrozenFiftyTaskGain{}, \EeightFrozenHundredTaskGain{},
\EeightFrozenTwoHundredTaskGain{}). Training the code appears to align the task objective with the
cross-talk objective; freezing it leaves them in conflict that no amount of readout training
resolves.

\paragraph{The frontier is now exact.} Earlier we reported $\kappa_{\min}$ over a subset of features
and called it an upper bound. Evaluating every feature of all 180 models shows the bound was tight:
the subset overestimated by \FullLfourTwoHundredSubsetOver{} on average at $d=200$ and never by more
than a tenth of a level. Its \emph{justification} is weaker than its accuracy, though. The subset was
chosen by Theorem~\ref{thm:arrow}, and that theorem is sufficient rather than necessary: on trained
codes it locates the true argmin in only \FullLfourTwoHundredArgminFound{} of 20 models at $d=200$,
where the true argmin sits at leverage rank \FullLfourTwoHundredArgminRankMed{} by median, outside
the window. On untrained codes it succeeds every time. Training pulls the lowest-leverage feature and
the worst-frontier feature apart, and the numbers above are the exact ones.

\paragraph{A fourth width, committed to in advance.} Before running it we extrapolated the frontier
from $d\in\{50,100,200\}$ and registered $\kappa_{\min}=\ScRegistered$ at $d=400$ in
\texttt{configs/e7\_stageC.json}. Then we ran the width --- \ScSeeds{} models per arm at $F=2d=800$,
same protocol --- and measured \ScMeasured{}, an error of \ScRegisteredErrorPct\% against the
registered value.

Two caveats on that number, both of which make it weaker than it first reads. The registered
extrapolation used the subset frontier, the only statistic available when it was written; refitting
on the exact frontier gives $\kappa_{\min}\propto d^{\ScFitExponent}$ and \ScPredicted{}, an error of
\ScPredErrorPct\%. We report the registered comparison as the primary one and the refit beside it,
because the refit was computed after seeing the measurement and is the more flattering of the two;
the three defensible pairings of predicted and measured all fall between $-2.2\%$ and $-2.6\%$, so
nothing here turns on the choice. And a prediction from three points fitting two parameters is a
weak instrument: it tests that the growth is regular, not that any particular law holds. We read the
result as a local rate rather than a law, because the exponent drifts downward across the four widths
and Section~\ref{sec:res-e6} explains why four points cannot separate the candidate shapes.

Two other things replicate at this width. The loss still decides the interface: $L^2$ reaches
$R_{\mathrm{geom}}=\ScLtwoRgeom$ with its frontier at \ScLtwoKappa, against $L^4$'s
\ScLfourRgeom{} and \ScLfourKappa{}, and the untrained code again sits near the floor at
\ScRandRgeom{} with a frontier of \ScRandKappa{}. And the matched decoder comparison continues in the
direction the smaller widths established, more sharply: at $d=400$ the network leads the affine probe
by \ScPrimDiff{} of a sparsity level ([\ScPrimCiLow, \ScPrimCiHigh]) in \ScPrimNetWins{} of
\ScSeeds{} models, where the three smaller widths gave a tie or a third of a level. Under the
conservative reading, in which the probe keeps the maximum over families evaluated on the test set,
the difference is \ScPrimTestMaxDiff{} --- so the ordering at this width does not depend on which of
the two defensible selection rules is used. Section~\ref{sec:probes} explains why this is a statement
about supervised fitting and not about what an affine readout can represent.

\paragraph{The margin over an untrained code is closing, and that is the finding of this width.}
Both frontiers grow with width, but not at the same rate: fitted over all four widths the trained
code gives $d^{\ScExpTrained}$ and the untrained one $d^{\ScExpRandom}$
(Fig.~\ref{fig:frontierscaling}a). The untrained exponent is the larger, so the advantage of training
is a decreasing function of width in both readings of ``advantage''. As a ratio it falls throughout
and faster each time: \MarginFiftyRatio, \MarginHundredRatio, \MarginTwoHundredRatio,
\MarginFourHundredRatio. In absolute sparsity levels --- the operationally meaningful unit, since the
frontier licenses supports up to $\lceil\kappa_{\min}\rceil-1$ --- it rises to a peak at $d=200$ and
then falls: \MarginFiftyGap, \MarginHundredGap, \MarginTwoHundredGap, \MarginFourHundredGap. The fall
is larger than the sampling noise: a two-sample bootstrap over models --- the arms are separately
trained, so the difference is between independent groups and not paired --- gives a drop of
\MarginDrop{} with
interval [\MarginDropCiLow, \MarginDropCiHigh], and the two widths' intervals do not overlap. Taken
at face value the two fitted laws meet near $d\approx\ScCrossing$.

We state this rather than the extrapolation. Four widths spanning a factor of eight cannot establish
where two rates cross, and $d\approx\ScCrossing$ is well outside the range measured. What the data
does support is narrower and still consequential: \emph{training buys frontier at every width we
measured, and the amount it buys stops growing by $d=400$.} A reader entitled to conclude that
gradient descent finds structurally better codes would want that margin to be at least stable in
width, and here it is not. Whether the trained code's deceleration reflects a real ceiling or the
fixed optimisation budget --- \EsevenSteps{} steps at every width, so the largest models get the
least optimisation per parameter --- this campaign cannot say, and it is the first thing a larger
budget should test.

\paragraph{And the shortcut has stopped being defensible.} The same width shows the subset
heuristic's justification collapsing, which is why every number in the two paragraphs above is the
exact minimum over all $F$ features. The mean overestimate grows \SubsetFiftyOver, \SubsetHundredOver,
\SubsetTwoHundredOver, \SubsetFourHundredOver; the median leverage rank of the true argmin grows
\SubsetFiftyRank, \SubsetHundredRank, \SubsetTwoHundredRank, \SubsetFourHundredRank; and at $d=400$
the subset locates the true argmin in \SubsetFourHundredFound{} of \ScSeeds{} trained models. The
error is still small in absolute terms, but it is one-sided and it is asymmetric between the arms:
untrained codes keep the true argmin at leverage rank 1, so their subset value \emph{is} exact, while
trained codes drift. A margin computed from subset values is therefore biased in favour of training
by an amount that grows with exactly the variable under study --- which is why we computed the exact
frontier for all four widths before making the claim above, and why we would not trust the shortcut
past $d=200$ in future work.

\begin{figure}[t]
\centering
\includegraphics[width=\linewidth]{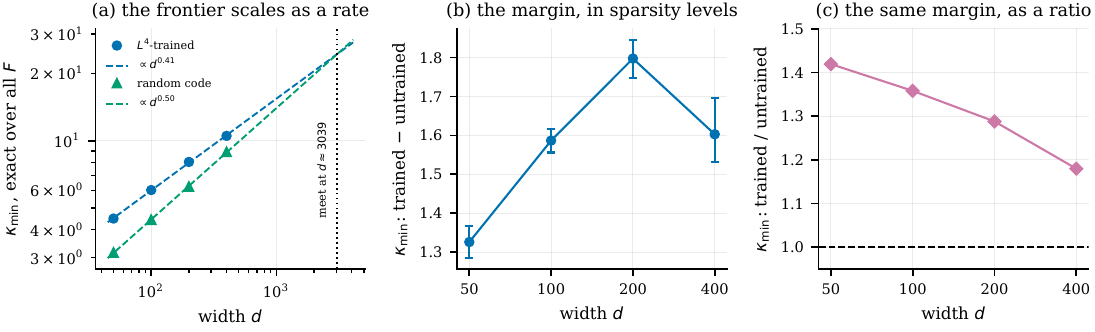}
\caption{\textbf{The frontier's rate, and the margin it puts in question.} All values are the exact
$\kappa_{\min}$ over every one of the $F$ features, never the subset. \textbf{(a)} Both codes'
frontiers are close to straight on these axes over the four widths, and the untrained slope is the
steeper. The power-law fits meet near $d\approx\ScCrossing$, drawn to show what the two slopes imply
and not offered as a prediction: it lies far outside the measured range, and four points do not
identify the shape --- other forms fit as well and move the meeting point by a factor of two. The
robust statement is the ordering of the slopes, not where they cross. \textbf{(b)} The margin in sparsity
levels, with two-sample bootstrap intervals over models; it peaks at $d=200$ and the fall to $d=400$
exceeds the sampling noise. \textbf{(c)} The same margin as a ratio, which falls at every step. The
two panels are shown separately because they are the two honest readings of the same quantity and
they do not agree until $d=400$.}
\label{fig:frontierscaling}
\end{figure}

\subsection{A second campaign, at double the feature load}
\label{sec:stageb}

Everything above comes from one campaign at one feature load, $F=2d$. Stage B is a separate
campaign, run with its own seeds and its own run record, and it varies that load
(Table~\ref{tab:stageb}). Its $p=0.01$ cells sit at $F=\SbLoadHigh d$ --- twice stage A's load, at
$d\in\{100,200\}$ --- and its $p=0.02$ cells stay at $F=\SbLoadMatched d$ and double the training
sparsity instead. Every other setting matches stage A: the same optimiser, step count, batch size,
probe budget and evaluation grid, verified against the two configuration files.

An earlier version of this section described the $p=0.01$ cells as exact repeats of stage A and read
their larger effects as a replication. That was wrong: they are a different design, and
Appendix~\ref{app:withdrawn} records the error and what it cost.

\paragraph{At double the load the ordering holds and the frontier falls.} The decoder comparison
keeps its sign and grows: the network leads the affine probe by \SbLfourHundredDiff{} of a sparsity
level at $d=100$ ([\SbLfourHundredCiLow, \SbLfourHundredCiHigh], \SbLfourHundredNetWins{} of
\SbSeeds{} models) and \SbLfourTwoHundredDiff{} at $d=200$
([\SbLfourTwoHundredCiLow, \SbLfourTwoHundredCiHigh], \SbLfourTwoHundredNetWins{} of \SbSeeds{}),
against \PrimLfourHundredPostDiff{} and \PrimLfourTwoHundredPostDiff{} at the same widths and half
the load. The untrained arm keeps its own opposite sign, \SbRandHundredDiff{} and
\SbRandTwoHundredDiff{} levels to the probe; and every $L^2$ cell has both decoders at $s_{95}=0$,
\SbLtwoJointFailures{} models across \SbLtwoCells{} cells, the same joint failure stage A shows.

The frontier moves the other way, as it must: doubling the number of features in the same $d$ costs
recovery, and $\kappa_{\min}$ falls from \FullLfourHundredKappaMin{} to \SbKappaLfourHundred{} at
$d=100$ and from \SbMatchPoneKappa{} to \SbKappaLfourTwoHundred{} at $d=200$. So the two campaigns
together say that the loss--geometry separation and the sign structure of the decoder comparison are
not artefacts of $F=2d$, while the absolute frontier is a strong function of the load. The
comparison is across campaigns with different seeds, which is why we read the ordering and not the
difference of the means.

\paragraph{Training sparsity, at matched load, buys nothing.} Stage B's $p=0.02$ cells are the only
place any training variable other than the loss is varied, and they must be compared against stage
A's cells at the \emph{same} load, $F=\SbLoadMatched d$ --- not against stage B's own $p=0.01$ cells,
which are at twice that. Made correctly, the comparison gives nothing to the hypothesis that denser
training states buy frontier. At $d=50$ the two are indistinguishable, \SbMatchFiftyPoneKappa{}
against \SbMatchFiftyPtwoKappa{}. At $d=200$ the higher sparsity is \emph{worse}, \SbMatchPtwoKappa{}
[\SbMatchPtwoKappaLo, \SbMatchPtwoKappaHi] against \SbMatchPoneKappa{}
[\SbMatchPoneKappaLo, \SbMatchPoneKappaHi], with the per-model ranges disjoint.

An earlier version reported the opposite as ``the cleanest small result in the campaign'' by
comparing across the load. We withdraw it. What remains is a null: over the two widths where a
matched comparison exists, doubling the training sparsity does not improve the geometry the code
presents --- indistinguishable at $d=50$ and worse at $d=200$ --- and the decoder ordering survives
the change ($L^4$ at \SbLoadLfourFiftyDiff{} and \SbLoadLfourTwoHundredDiff{}, $L^2$ still pinned at
$\kappa_{\min}=\SbKappaLoadLtwoTwoHundred$). Both matched widths rest on 10 models against stage A's
20, and the direction at $d=200$ comes from one width, so we read this as an absence of evidence for
the withdrawn claim rather than as evidence for its converse.

\begin{table}[t]
\centering
\caption{Stage B: a separate campaign with its own seeds and run record. The $F/d$ column is the one
to read first. The $p=0.01$ rows are at $F=\SbLoadHigh d$, twice stage A's load; the $p=0.02$ rows
are at $F=\SbLoadMatched d$ and vary the training sparsity instead. ``mean difference'' is network
minus best post-ReLU affine probe in $s_{95}$, paired within seed, and $\kappa_{\min}$ is the exact
minimum over all $F$ features. Any comparison against stage A must hold $F/d$ fixed; doing otherwise
produced a claim we withdraw in Appendix~\ref{app:withdrawn}. Generated from
\texttt{results/e7\_stageB/derived/}.}
\label{tab:stageb}
\small
\begin{tabular}{llcccccc}
\toprule
code & $d$ & $F/d$ & $p$ & mean difference & 95\% interval & net / tie / probe & $\kappa_{\min}$ \\
\midrule
$L^4$ & 100 & 4 & 0.01 & +0.70 & [+0.40, +1.00] & 7 / 3 / 0 & 4.3072 \\
$L^4$ & 200 & 4 & 0.01 & +1.00 & [+0.70, +1.30] & 9 / 1 / 0 & 5.6212 \\
$L^4$ & 50 & 2 & 0.02 & +0.00 & [+0.00, +0.00] & 0 / 10 / 0 & 4.5074 \\
$L^4$ & 200 & 2 & 0.02 & +0.40 & [+0.10, +0.70] & 4 / 6 / 0 & 7.6479 \\
$L^2$ & 100 & 4 & 0.01 & +0.00 & [+0.00, +0.00] & 0 / 10 / 0 & 1.0003 \\
$L^2$ & 200 & 4 & 0.01 & +0.00 & [+0.00, +0.00] & 0 / 10 / 0 & 1.0076 \\
$L^2$ & 50 & 2 & 0.02 & +0.00 & [+0.00, +0.00] & 0 / 10 / 0 & 1.0000 \\
$L^2$ & 200 & 2 & 0.02 & +0.00 & [+0.00, +0.00] & 0 / 10 / 0 & 1.0031 \\
untrained & 100 & 4 & 0.01 & -1.80 & [-2.00, -1.50] & 0 / 0 / 10 & 3.3999 \\
untrained & 200 & 4 & 0.01 & -2.00 & [-2.00, -2.00] & 0 / 0 / 10 & 4.7230 \\
\bottomrule
\end{tabular}

\end{table}

\subsection{What the decoder comparison can and cannot mean}
\label{sec:probes}

The comparison in the previous subsections is easy to over-read, and the direction of the
over-reading is the flattering one. This subsection states the ceiling on it and then measures
where the comparison actually sits against that ceiling, because a ceiling one has only argued for
is a caveat and a ceiling one has measured is a result.

\paragraph{The probe class contains the network's own decoder.} The affine probe is fitted on the
post-ReLU state $z=\mathrm{ReLU}(\Phi b)$ with a design matrix that appends an intercept and nothing
else (\code{\_design} in \code{src/lrtr/probes.py}), so its hypothesis class is every affine map of
$z$ --- which includes $W_{\mathrm{out}}$, the map the network's own output layer applies. Writing
that map in the probe's own parameterisation,
\[
  W_{\mathrm{net}} \;=\; \begin{bmatrix} W_{\mathrm{out}}^{\top} \\ 0 \end{bmatrix}
  \;\in\;\R^{(d+1)\times F},
\]
scoring $W_{\mathrm{net}}$ as a probe reproduces the network's own evaluation exactly. It follows
that a probe configuration attaining the network's score \emph{exists} for every model in every
cell. Whenever the network scores higher, what the gap measures is that our fitting and selection
procedure did not find that map: a fixed ridge grid, a fixed number of steps, and $8{,}192$ training
states. It is not a statement that an affine readout cannot represent what the network computes,
because it demonstrably can.

\paragraph{The selection misses it, and we can say so by how much.} That argument bounds the
comparison without telling us whether the bound is tight, so we tested it. The campaign selects a
probe configuration by maximising a validation objective --- the area under the per-sparsity
validation recovery curve, stored for every configuration in the saved records --- and $W_{\mathrm{net}}$
is a member of the class that objective ranges over. Evaluating the same objective at
$W_{\mathrm{net}}$ therefore asks whether the selection missed something it had the data to find. It
needs no refitting, so it runs over the released weights in seconds per model
(\code{scripts/probe\_ceiling.py}).

It did miss it, in every cell where the network leads and in every model of those cells. The
validation gap between the network's own map and the selected probe is
\CeilLfourHundredValGap{} ([\CeilLfourHundredValCiLow, \CeilLfourHundredValCiHigh]) at $d=100$,
\CeilLfourTwoHundredValGap{} ([\CeilLfourTwoHundredValCiLow, \CeilLfourTwoHundredValCiHigh]) at
$d=200$ and \CeilLfourFourHundredValGap{}
([\CeilLfourFourHundredValCiLow, \CeilLfourFourHundredValCiHigh]) at $d=400$, all positive, and
positive in \CeilLfourMissed{} of \CeilLfourModels{} models across the \CeilLfourCells{} $L^4$ cells
at $d\ge100$ in both campaigns, at both feature loads
(\CeilSbLfourHundredValGap{} and \CeilSbLfourTwoHundredValGap{} in stage B's $F=\SbLoadHigh d$ cells).

The gap tracks the reported difference in size as well as in sign. It is largest at $d=400$, the
width where the network's lead is \ScPrimDiff{} levels, and smallest at $d=50$
(\CeilLfourFiftyValGap, [\CeilLfourFiftyValCiLow, \CeilLfourFiftyValCiHigh]), the width where the
lead is zero. It is not zero there, so we do not offer $d=50$ as a null: a small failure to find the
better map need not cost a whole sparsity level, and $s_{95}$ is floored.

The $L^2$ cells are the internal control for the measurement itself. There neither decoder recovers
anything and the network has no lead, and there the validation gap sits at essentially zero. So a
positive gap is not something this analysis produces wherever it looks: it appears where the network
leads, vanishes where nothing does, and reverses where the probe leads.

So the residual $L^4$ difference is a search failure, measured rather than assumed. A better member
of the probe's own class was available, validation could see that it was better, and the procedure
returned something else.

\paragraph{The control rules out the obvious objection to that.} One could suspect this of being an
artefact --- that the network's own map simply always looks good on a validation objective computed
from the network's own states. The untrained arm says otherwise, and by a wide margin in the
opposite direction: there the selected probe beats $W_{\mathrm{net}}$ on validation by
\CeilRandFiftyValGap{} at $d=50$, \CeilRandTwoHundredValGap{} at $d=200$ and
\CeilRandFourHundredValGap{} at $d=400$, and the selection fails to beat it in
\CeilRandMissed{} of \CeilRandModels{} models. On a code that was never trained the fitting procedure
finds a far better affine map than the network's own readout, which is exactly what it fails to do on
a trained one.

As a validity check, this analysis rebuilds each cell's split bundle from the campaign's own seed and
recomputes both sides from the released weights. The differences it obtains reproduce the published
ones exactly --- a worst-cell discrepancy of \CeilReproMaxDelta{} of a sparsity level across all
\CeilReproCells{} cells of the three campaigns, Tables~\ref{tab:primary} and~\ref{tab:stageb}
together. The check is stated over every cell rather than over the ones it happens to pass: an
earlier version was scoped to Table~\ref{tab:primary} alone, which excluded the single cell where it
then failed.

\paragraph{What the comparison does and does not support.} It does not support ``the network's
nonlinearity buys expressive power'': the class contains its decoder, and the gap is now
\emph{known} to be a search failure rather than merely arguably one. We report the $L^4$ differences
of \PrimLfourHundredPostDiff{} to \ScPrimDiff{} and build nothing on them.

What the comparison does support is a property of the code, because the protocol's asymmetries are
held fixed while the code varies. Both sides read the same state, both get one validation-selected
global threshold, both are scored on the same held-out supports. Under that fixed protocol the sign
of the difference is set by how the code was produced: positive and growing with width for $L^4$,
exactly zero for $L^2$, and negative by one to two and a half levels for an untrained code. A
protocol that systematically favoured the network could not produce a two-and-a-half-level loss on
any arm.

Read that way the comparison says something we did not expect. Training does not only improve the
geometry the code presents (Section~\ref{sec:res-primary}); it also puts the network's own readout
somewhere a supervised fit of the same class, on the same data, with a standard grid, does not
reach --- while on an untrained code that same fit comfortably outperforms the network's readout.
The trained readout is not merely good; it is hard to find by fitting. That is a form of the
co-adaptation E8 isolates, seen from the decoder's side rather than the code's.

\paragraph{The question this leaves open.} Knowing that the selection misses $W_{\mathrm{net}}$ does
not say why. Two mechanisms are consistent with it and this experiment does not separate them: the
optimiser and grid may simply be too coarse to reach that region, or the probe's surrogate objective
--- ridge, logistic or squared hinge --- may have its optimum somewhere other than at the map that
maximises exact support recovery, in which case no amount of better optimisation of \emph{that}
objective would find it. Separating them needs a refit \emph{started} at $W_{\mathrm{net}}$: if the
fit stays, the grid was too coarse; if it walks away and scores worse, the objective is misaligned.
That is the expensive direction --- the full grid and the margin fits for every model, where the
measurement above needs neither --- so \code{fit\_probe} accepts a starting point and
\code{probe\_ceiling.py --with-refit} wires it up, and we report the question rather than a partial
answer to it.

\section{Background and related work}
\label{sec:background}

\subsection{Computation in superposition}

A network of hidden width $d$ \emph{represents $F\gg d$ features in superposition} if its
activation $a(x)\in\R^d$ is approximately $\Phi b$ for a feature-encoding matrix
$\Phi\in\R^{d\times F}$ and a sparse Boolean $b$. Following~\cite{hanni2024superposition},
features $f_1,\dots,f_F$ are \emph{$\varepsilon$-linearly represented} by $a$ if there is a
readout $R\in\R^{F\times d}$ with $|r_k\cdot a(x)-f_k(x)|<\varepsilon$ for all $k$ and $x$.

H\"{a}nni et al.\ construct an approximate-linear recursive template for sparse Boolean
circuits, alternating a targeted superpositional AND layer with an error-correction layer, and
certify emulation of circuits of width $\Otilde(d^{3/2})$. Adler and
Shavit~\cite{adler2024superposition} give parameter-description lower bounds and explicit
constructions for computing in superposition, obtaining a near-quadratic capacity by
thresholding after each stage. Adler et al.~\cite{adler2025combinatorial} reframe superposition
combinatorially through feature channel coding, and Kong et al.~\cite{kong2025expand} show
empirically that reducing inter-feature interference at fixed parameter count improves
performance. This paper reproves none of these constructions; it identifies the interface-level
reason the two regimes coexist.

\subsection{Trained toy models of computation in superposition}

A recent empirical line studies whether small trained networks actually compute in
superposition. Braun et al.~\cite{braun2025apd} introduced the compressed-computation task ---
a width-$50$ network asked to compute $100$ sparse ReLU features --- as a testbed. Bhagat et
al.~\cite{bhagat2026cc} argued that the standard $L^2$-trained model largely exploits a mixing
term rather than genuinely computing in superposition, and Ferreira da Silva and
Heimersheim~\cite{ferreiradasilva2026l4} showed that training under an $L^4$ loss elicits a
solution that does compute in superposition, assigning each feature a sparse codeword over
neurons and decoding it with a pseudoinverse of the encoder --- that is, a linear readout
interface in the sense studied here. In parallel,
\cite{ferreiradasilva2026errcorr} report perturbation-based evidence for feature-specific
error correction in large language models, the mechanism whose certified tolerance drives the
$d^{3/2}$ compatibility scale of the H\"{a}nni template. This debate is about which decoding
interface a trained network maintains; the diagnostic proposed here is a direct instrument for
it, and Section~\ref{sec:res-e5} applies it to exactly the models under dispute.

\subsection{Frame theory, Welch bounds and cross-Gramians}
\label{sec:frames}

Classical Welch bounds~\cite{welch1974lower} lower-bound the correlations of overcomplete
unit-norm systems; the maximum-correlation form we use,
$\max_{i\neq j}|\langle\phi_i,\phi_j\rangle|\ge\sqrt{(F-d)/(d(F-1))}$, appears for instance as
Theorem~2.3 of Strohmer and Heath~\cite{strohmer2003grassmannian} in the study of Grassmannian
frames. Waldron~\cite{waldron2003generalized} showed that generalised Welch-bound-equality
sequences are exactly tight frames, and Datta, Howard and Cochran~\cite{datta2012geometry}
derived the entire family of higher-order Welch bounds from a geometric rank/Grammian
argument, of which the rank--trace computation behind Theorem~\ref{thm:floor} is the
first-order instance; see also~\cite{klappenecker2005mub} for the design-theoretic view. For
pairs of sequences, Aceska and Kaczanowski~\cite{aceska2022crossframe} introduced the
cross-frame potential, and Christensen, Datta and Kim~\cite{christensen2020welch} proved
Welch-type bounds for the maximum coherence of dual frame pairs, characterising equality via
equiangular tight frames.

Theorem~\ref{thm:floor} sits in this lineage: it is a unit-diagonal cross-Gramian formulation
assuming only $\operatorname{rank}(M)\le d$ and $M_{ii}=1$ --- in particular, no duality
relation between the two sequences, which, as finite sequences, are of course frames for their
spans. We include its short rank--trace proof only to keep the paper self-contained; the proof
mechanism is standard in this literature, cf.~\cite{datta2012geometry}. The contribution here
is the application to readout interfaces, the gain-normalised form
(Corollary~\ref{cor:calibfree}), and the diagnostic built on top of them.

\subsection{Compressed sensing and sparse recovery}

The threshold-recovery results used below are standard coherence-based and random-dictionary
support-recovery arguments. Their role is to make explicit a distinction relevant to neural
computation: small-error linear readout requires uniformly small coordinate error, whereas
threshold recovery only requires the aggregate interference from the active support to remain
below a margin.

The collision frontier of Section~\ref{sec:frontier} sits in this neighbourhood and we place it
carefully, because two established literatures ask adjacent questions.

The first is group testing. Recovering a sparse Boolean $b$ from $\Phi b$ is exactly
\emph{quantitative} group testing when $\Phi$ is a real measurement matrix, and the combinatorial
theory of superimposed, disjunct and separable codes descends from Kautz and
Singleton~\cite{kautz1964superimposed}. That theory asks how many tests suffice to identify the
\emph{whole} support, by \emph{any} decoder, usually for a matrix to be designed. We ask something
weaker and more specific: for a matrix we are given, and one coordinate of it, is that coordinate's
presence decidable by an \emph{affine} function of $\Phi b$? The gap is real in both directions.
A matrix can be identifiable in the group-testing sense while $\kappa_i\le s$, because
$\kappa_i\le s$ only requires a collision between \emph{convex combinations} of admissible states,
not an exact coincidence of two integer ones. Conversely $\kappa_i>s$ certifies a linear probe
with a margin, which identifiability alone does not supply. The restriction to affine decoders is
not a weakness of the analysis but the object of interest, since linear probing is how features
are read in practice.

The second is the null space property. Writing $z'$ for $z$ extended by zero at coordinate $i$,
the constraint $\Phi_{-i}z=\phi_i$ says precisely that $z'-e_i\in\ker\Phi$, so $\kappa_i$ is a
statement about kernel vectors normalised at one coordinate --- the same geometry as the null
space property and as the neighbourliness of randomly projected
polytopes~\cite{donoho2005neighborliness}. Two differences matter. The null space property is
quantified over all supports and certifies recovery by $\ell_1$ minimisation, a nonlinear decoder;
$\kappa_i$ is per-coordinate and certifies an affine one. And the null space property is hard to
verify, which is why the literature tests it by semidefinite
relaxation~\cite{daspremont2011nullspace}, whereas $\kappa_i$ is the exact value of one linear
programme --- a consequence of asking the weaker, one-coordinate question rather than of a better
algorithm. The box $\norm{z}_\infty\le1/\alpha$, which comes from bounded activations and has no
counterpart in the null space property, turns out to bind rarely, so for most features the
frontier is a minimum-$\ell_1$-representation quantity and the tools of that literature apply
directly.

\subsection{Sparse autoencoders}

Sparse autoencoder work scales dictionary learning to millions of
latents~\cite{cunningham2023sparse,gao2024scaling,lieberum2024gemma} and studies
reconstruction--sparsity trade-offs~\cite{michaud2025manifolds}. Such results motivate the
importance of overcomplete internal representations, but dictionary size is a reconstructive
quantity and does not measure the recursive Boolean interface capacities studied here. This
paper answers one of the open questions listed by Sharkey et al.~\cite{sharkey2025open}
concerning what theoretical insight can be gleaned from computation performed natively in
superposition.

\paragraph{Delineation from a prior sparse-autoencoder study.} The closest prior work measures
sparse autoencoders with an overlapping toolset and a different question.
\citet{borobia2026pruning}, in \emph{Neurocomputing}, train TopK autoencoders on weight-pruned
language models and ask which features survive compression, using reconstruction fidelity and
feature-level survival statistics as the instrument. That paper is the reference point for the
dictionary measurements here, and the division of labour between the two is worth stating plainly
because the objects overlap and the quantities do not. This paper trains no autoencoder, prunes
nothing, and measures no reconstruction. It
takes decoder matrices as given --- including ones that paper never examined, such as the
randomly-initialised control of Section~\ref{sec:dictionaries} --- and computes a different quantity:
the cross-talk a linear reader of that decoder must accept, its exact decomposition into leverage
dispersion, and the sparsity at which affine recovery must fail. No model, dictionary, measurement or
result is shared between the two, and neither paper's conclusions depend on the other's.


\section{Limitations and what we withdrew}
\label{sec:limitations}

\begin{enumerate}[leftmargin=*,itemsep=1pt]
\item \textbf{One trained setting.} E5 covers a single toy task at $d=\EfiveD$, $F=\EfiveF$,
with \EfiveSeeds{} seeds per condition. Nothing here is a claim about large trained language
models.
\item \textbf{Out-of-distribution evaluation.} The diagnostic is defined on Boolean sparse
states, while the toy models are trained on continuous sparse inputs. It therefore measures
which interface the learned code \emph{supports}, not what the network does on its native
distribution.
\item \textbf{Constant sparsity in the application.} The H\"{a}nni-template comparison assumes
$s=O(1)$, whereas the distributional results allow $s$ up to $O(d/\log d)$; real features are
continuous with heterogeneous sparsity.
\item \textbf{The floor bounds linear readouts only.} Nonlinear decoders and
distribution-specific average-case mechanisms are outside its scope --- which is the point of
the separation, but also a limit on what the diagnostic can certify.
\item \textbf{Fit with few points.} The $c\,d/\ln d$ fit in E6 rests on \EsixFitPoints{}
widths, on which a plain linear law fits nearly as well ($R^2=\EsixInterpRsqLin$ against
$\EsixInterpRsqLog$). The data are consistent with the predicted shape and do not discriminate
it from the alternatives; the logarithm's base is not identifiable at all.

\item \textbf{$s_{95}$ is a point estimate, not a certified threshold.} The trial budget shrinks
with width, and at $d=\EsixDmax$ even a perfect run gives a Wilson lower bound of
\EsixTopBestWilsonLow --- below the $95\%$ level being estimated. The reported thresholds are
best estimates; certifying them at confidence would need more trials at the large widths.

\item \textbf{The linear side's non-exactness is the theorem, not a measurement.} Wherever we
report that a linear readout ``cannot be exact'', that follows from
Theorem~\ref{thm:floor} by construction for any rank-$d$ unit-diagonal interface. What the
experiments measure is the \emph{magnitude} of that error for codes that were actually learned,
and how far a threshold decoder gets before it bites.
\item \textbf{Correlated curve points.} The nested-support trick makes each point of
$p_{\mathrm{rec}}(s)$ marginally unbiased but correlates points within a trial; the Wilson
intervals are per-point and should not be read as a joint confidence band.
\item \textbf{Open log gap.} The Adler--Shavit capacity bracket retains a single logarithmic
factor of slack, which is not resolved here.

\item \textbf{The leverage prediction is a sufficient condition, and its selection heuristic decays
with width.} Corollary~\ref{cor:failthresh} predicts the $L^2$ collapse from leverage alone and
holds at all three widths of Section~\ref{sec:res-primary}. But it is applied at $h_{\min}$, and it
is sufficient rather than necessary, which shows up in where the true worst feature sits: the
lowest-leverage window contains the true $\arg\min\kappa$ in \FullLfourTwoHundredArgminFound{} of 20
$L^4$ models at $d=200$ and \SubsetFourHundredFound{} of \ScSeeds{} at $d=400$, where the true
$\arg\min$ sits at leverage rank \SubsetFourHundredRank{} by median --- against every model of every
untrained code, at every width. Training pulls the two apart, and the gap widens with width, so the
theorem localises the failure less and less well as the code improves. Nothing reported depends on
the heuristic --- the frontiers quoted are exact --- but a reader should not take low leverage as a
reliable pointer to the binding feature, and past $d=200$ should not use it as a sampling window at
all.

\item \textbf{The exact frontier does not extrapolate much further.} Every $\kappa_i$ reported here is
exact --- all $F$ features of every model at all four widths, including the \ScSeeds{} models per arm
at $d=400$, $F=800$. An earlier version of this limitation said the $d=400$ sweep would be
unaffordable and that a fourth width would have to return to a subset estimate; it was run instead,
in a few hours on ten CPU workers, and this entry is corrected rather than removed because the
estimate was wrong in the direction that mattered. What remains true is the shape of the cost: one
linear programme per feature --- \FrontierLpSecondsTwoHundred\,s at $d=200$, from the recorded sweep
times --- growing steeply in $d$ while $F$ grows with it, with the wall-clock of every campaign in
Appendix~\ref{app:cost}. A fifth width would be a different
proposition, and the honest statement is that the statistic is exact everywhere we report it and that
we do not know where it stops being affordable.

\item \textbf{The nonlinear witness is built but not yet used.} The extractor of
Section~\ref{sec:witness} is implemented and validated --- every witness it returns is confirmed
inseparable by a linear programme, and on small codes against brute-force enumeration --- but two
things are missing. The Radon compression to $d+2$ states is not implemented, so the witnesses are
valid and not minimal; and no result reported in this paper is derived from one, so the sharper
claim it licenses is not made here.

\item \textbf{How much the ReLU costs is measurable in only three cells.} The decoder comparison
says the ReLU costs an $L^4$ code almost nothing and an untrained code a great deal
(Section~\ref{sec:res-primary}). Quantifying that on the frontier rather than on a fitted probe is
harder: $\kappa$ is defined on $\Phi b$, and $\mathrm{ReLU}(\Phi b)$ is not a linear image of the
state polytope, so the post-ReLU side must be sampled and tested rather than solved. Sampling can
only refute separability, never confirm it, so it overestimates the frontier --- here by
\ReluLooseMin{} to \ReluLooseMax{} levels against the exact $\kappa$ --- and the overestimate cancels
only in a difference taken with the same sampler and the same draws on both sides
(\ReluGapDraws{} draws per side, \ReluGapFeatures{} features per cell, sampler ceiling
$s\le\ReluGapSmax$).

Three of the nine cells give a measurable difference and they agree with the decoder comparison: at
$d=50$ the ReLU costs the $L^4$ code \ReluLfourFiftyGap{} levels
$[\ReluLfourFiftyCiLow,\ReluLfourFiftyCiHigh]$ on \ReluLfourFiftyPositive{} of \ReluGapFeatures{}
features and the untrained code \ReluRandFiftyGap{} $[\ReluRandFiftyCiLow,\ReluRandFiftyCiHigh]$ on
\ReluRandFiftyPositive{}; at $d=100$ the untrained code loses \ReluRandHundredGap{}
$[\ReluRandHundredCiLow,\ReluRandHundredCiHigh]$ on all \ReluRandHundredPositive{}. In
\ReluCensoredCells{} cells both sides reach the sampler ceiling, so the difference there is censoring
and not measurement, and we report no value --- including for every $L^4$ cell above $d=50$, which is
precisely where we would most want one. Every $L^2$ cell is a flat zero because its frontier is at
$1$ on both sides. A larger sampling budget would lift the ceiling and we have not spent one, so the
claim that the cost is small on trained codes rests on $d=50$ plus the decoder comparison, and not on
the frontier at the two larger widths.

\item \textbf{A dictionary-only score is not quite a property of the autoencoder.} Every decoder is
analysed under a declared unit-column gauge (Section~\ref{sec:dictionaries}), which removes the
scaling freedom, but not the one that matters most in practice. A dead latent --- a column whose
encoder never fires on the data --- still occupies a direction, still enters the row space, and still
changes the leverage of every other feature, while leaving the function the autoencoder computes
untouched. Two functionally identical dictionaries can therefore receive different scores here, and
the dictionaries most likely to carry such columns are the widest ones, which is the axis
Section~\ref{sec:dictionaries} varies. Resolving it needs matched activation data to define the
active support, which is exactly the input this section is built to avoid using, so it is a real
limit on the method and not only on this application of it. What the section measures is the geometry
the released matrix presents; whether that matrix has columns the model never uses is a separate
question we do not answer.

\item \textbf{One control the analysis wants and does not have.} A random \emph{tight frame}, whose
$R_{\mathrm{geom}}$ is exactly $1$ rather than the \SaeCtlRgeomMax{} an i.i.d.\ code reaches at
worst, would sharpen
the claim that near-floor geometry is generic: as it stands the comparison is against a code that
is close to the floor rather than on it, so a small residual advantage for the trained code could be
conditioning rather than structure.

\item \textbf{The margin over an untrained code stops growing inside our own range.} This is the
sharpest limitation on the headline claim, and it comes from our own data rather than from something
we did not run. Section~\ref{sec:res-primary} reports the trained frontier at $d^{\ScExpTrained}$ and
the untrained one at $d^{\ScExpRandom}$: the untrained exponent is larger, the ratio falls at every
step, and the margin in sparsity levels peaks at $d=200$ and falls by \MarginDrop{}
$[\MarginDropCiLow,\MarginDropCiHigh]$ at $d=400$. ``Training buys the code'' therefore holds at every
width we measured and is not established as a property that survives scale. Two candidate
explanations are not separated here: a real ceiling on what the frontier can buy, or the fixed
optimisation budget of \EsevenSteps{} steps at every width, which gives the largest models the least
optimisation per parameter. Distinguishing them needs a budget sweep at fixed width, which is cheap,
and a fifth width, which is not.
\end{enumerate}

\section{Discussion}
\label{sec:discussion}

Four readings survive this work, and two that motivated it do not.

\textbf{The floor is a target, not an obstacle --- and reaching it means less than it looks.}
Gradient descent under the unit-diagonal constraint converges to a unit-norm tight frame
(Appendix~\ref{app:earlier}), and a network trained on a task with no interference objective lands
within \EfiveLfourRatioPinvMean{} of the same bound. But an untrained i.i.d.\ code reaches
$\Rgeom=\SARandFiftyRgeom{}$ with no training at all, and \SaeControls{} shape-matched controls land
between \SaeCtlRgeomMin{} and \SaeCtlRgeomMax{} across a twentyfold range of load. Proximity to the
rank--trace floor is what the ensemble gives for free, not a signature of learned structure, and
Proposition~\ref{prop:leverage} says why rather than leaving it as an empirical coincidence: the ratio
is exactly one when leverage is equalised, and an unstructured code equalises it. Any attainment
ratio reported without a shape-matched baseline is uninterpretable, and we could not find that
baseline in the work we build on.

\textbf{No released dictionary we measured is at the floor, and the excess tracks what the model
learned.} All \SaeDictsAboveControl{} of \SaeDicts{} decoders sit above their own control, at $\Rgeom$
from \SaeDictRgeomMin{} to \SaeDictRgeomMax{}, with no exception across five model families, three
training recipes and two decoder sites. The direction is the opposite of the natural guess: a random
code of the same shape carries strictly \emph{less} cross-talk under the best linear readout each
admits, so whatever sparse reconstruction optimises, it is not the evenness of its own leverage. The matched pair EleutherAI released settles where the excess comes from: an
autoencoder trained by an identical recipe on a \emph{randomly initialised} copy of the same model
returns to the floor (\SaeRandInitRgeomMin--\SaeRandInitRgeomMax) in all \SaePairedLayers{} paired
layers. Fitting a sparse autoencoder does not by itself produce dispersed leverage; representing a
trained model's features does. For practice this reorders two intuitions: the layer a dictionary is
trained on changes its linear readability by far more than its width does, though not in one ordering
that holds across families; and the coefficient of variation --- the obvious summary --- is the wrong
one, because it moves when the ratio does not.

\textbf{The loss decides whether a linear interface exists at all.} This is the largest effect we
report and the one that replicates most widely. $L^2$ drives $\Rgeom$ to \SALtwoTwoHundredRgeom{} at
$d=200$ and \ScLtwoRgeom{} at $d=400$ while collapsing its frontier to $1$ --- its worst feature is
separable at singletons and at no larger sparsity --- and $L^4$ stays at \SALfourTwoHundredRgeom{}
with a frontier that grows to \ScMeasured{}. What
distinguishes models is therefore not whether they approach the floor but whether the floor
\emph{binds} them, and Stage B reproduces the separation in an independent campaign at twice the
feature load.

\textbf{What training buys is the code, not the readout.} A random code with only its readout trained
keeps an untrained code's frontier, \EeightFrozenHundredKappa{} against \EeightRandomHundredKappa{},
while training both reaches \EeightTrainedHundredKappa{}. The frozen arm's own decoder ends at
$\Rreadout=\EeightFrozenHundredRreadout{}$, near twice the closed-form cross-talk optimum for the
code it was handed --- and that is not a failed optimisation. It beats that optimum by
\EeightFrozenHundredTaskGain{} on the loss it was trained under, so on a fixed random code the task
objective and the cross-talk objective genuinely diverge and the readout is right to follow the
former. What the joint arm gains is that it need not choose: it holds
$\Rreadout=\EeightTrainedHundredRreadout{}$ while beating the same baseline by
\EeightTrainedHundredTaskGain{}. Training the code aligns the two objectives; freezing it leaves them
in a conflict no amount of readout training resolves.

\textbf{What does not survive, first: the decoder framing, in either direction.} This work began with
the claim that an $L^4$ network's thresholded output keeps working past where linear readouts of its
own code collapse. Against the fixed readouts of Appendix~\ref{app:earlier} that appeared true; a
matched comparison against a fitted probe was registered as the test, and the registered prediction
was a tie. What the corrected comparison gives is neither: across \PrimModels{} models the two tie in
\PrimTies{} --- but \PrimLtwoJointFailures{} of those ties are $L^2$ cells where both decoders sit at
$s_{95}=0$, which is joint failure and not agreement --- with \PrimNetWins{} to the network,
\PrimProbeWins{} to the probe, and a difference that grows with width to \ScPrimDiff{} levels in
\ScPrimNetWins{} of \ScSeeds{} models at $d=400$.

We decline to claim that as a nonlinear advantage, and not merely on principle. The probe's
hypothesis class contains the network's own output layer, so a probe attaining the network's score
exists in every cell; writing that map in the probe's parameterisation and scoring it under the very
objective the selection maximises shows that the selection missed it in \CeilLfourMissed{} of
\CeilLfourModels{} models across every $L^4$ cell above the tie width, by
\CeilLfourHundredValGap{} to \CeilLfourFourHundredValGap{} of validation objective. The gap is a
search failure, measured. The untrained arm is the control that makes this a finding rather than a
tautology: there the same selection beats the network's own map by \CeilRandFourHundredValGap{} and
fails to do so in \CeilRandMissed{} of \CeilRandModels{} models (Section~\ref{sec:probes}).

What the comparison does support is a property of the code, because the
protocol is held fixed while the code varies: the sign is positive under $L^4$, zero under $L^2$, and
negative by up to \PrimRandFourHundredDiff{} levels on a code that was never trained. A single cell
of that table says little; the ordering across the three is the result. And it comes with a
mechanism: on a trained code the network's own readout sits where a supervised fit of the same class
on the same data does not reach, while on an untrained code that fit comfortably beats it. The
trained readout is not merely better --- it is hard to find by fitting, which is co-adaptation seen
from the decoder's side.

\textbf{And second: the representation gap.} An earlier version located the effect in what the ReLU
discards, reporting that a probe reading the preactivation beats the network by a full sparsity level.
That was an artefact of the threshold-selection defect. On trained codes the pre-ReLU advantage is now
at most a quarter of a level with most seeds tied, and at $d=400$ it reverses. The surviving form is
narrower and confined to the untrained arm, where the preactivation is more affinely decodable by
\PrimRandFiftyPreDiff{} to \PrimRandFourHundredPreDiff{} levels: the ReLU-plus-threshold pipeline
throws away a large and growing amount on a code that was never trained, and next to none on one
trained under $L^4$. Appendix~\ref{app:withdrawn} records both withdrawals with their causes, along
with four others.

\textbf{An instrument, and what it does not settle.} This gives the debate
in~\cite{bhagat2026cc,ferreiradasilva2026l4} something to measure. ``Does this network compute in
superposition?'' becomes three questions with numerical answers: how far the code is from its own
floor rather than from the ensemble's, how that distance decomposes into leverage dispersion, and, per
feature, the largest sparsity at which any affine rule could succeed. The first two are statements
about the row space under a declared unit-column gauge, and are invariant to any invertible
transformation of the ambient coordinates --- which is what makes them properties of the code rather
than of a basis, and equally what stops them from being claims about numerical conditioning or
noise robustness. None requires
reverse-engineering a mechanism, and the third is exact rather than sampled --- which matters, because
this project's history is a series of empirical readings that reversed when an estimand or a baseline
was tightened, while $\kappa_i$ never moved.

The consequence question deserves stating precisely, because the honest answer is neither ``none'' nor
``established''. In operational units the diagnostic does say something: all \SaeFailEarlier{} of
\SaeDicts{} dictionaries reach guaranteed affine failure \SaeFailGapMin{} to \SaeFailGapMax{} sparsity
levels earlier than an i.i.d.\ code of their own shape and operating sparsity. At a sparsity where a
random code is still safe, a released dictionary already has a feature no affine rule can recover.
That is a statement a practitioner can act on, and it is a worst-case guarantee rather than an
average.

What it is not is an external endpoint. Both the ratio and the failure threshold are functionals of
the same leverage distribution, so their agreement is a consistency check and not independent
corroboration; and neither is measured against reconstruction quality, attribution accuracy, steering
behaviour, or any other quantity an interpretability pipeline actually optimises. Whether a
\SaeFailGapMax{}-level earlier failure boundary changes what such a pipeline recovers in practice is
the question we would most want a reader to take next, and the one this paper does not answer.

For capacity theory, the interface distinction explains why the H\"{a}nni and Adler--Shavit results
coexist rather than compete (Appendix~\ref{sec:application}). It also cautions against reading
sparse-autoencoder dictionary sizes as computation capacities: those are reconstructive quantities,
and both halves of this paper show that reconstruction and linear decodability are governed by
different criteria --- most sharply in the dictionaries, which optimise the first and come out behind
an unstructured code of their own shape on the second.


\section*{Acknowledgements}
The authors thank the VRAIN institute and the Universitat Polit\`ecnica de Val\`encia for
institutional support, and two frame-theory researchers who, after reading the first version of this
preprint, pointed out prior work containing the proof mechanism of Theorem~\ref{thm:floor} and raised
the calibration question that became Corollary~\ref{cor:calibfree} and Remark~\ref{rem:calib}.

We also thank the external reviewers who audited this manuscript adversarially before submission.
Their findings are the reason for two of the eight entries in Appendix~\ref{app:withdrawn} and for
three of the nine entries in the released defect register.

\section*{Funding}
H.B.\ acknowledges support from the UPV doctoral programme. No additional external funding was
received for this work.

\section*{Code and data availability}
All code, configurations, run records, raw results and the sources of this manuscript are publicly
available at \url{https://github.com/hecboar/linear-readout-threshold-recovery}, together with a
single entry point that reproduces every experiment, table and figure (Appendix~\ref{app:repro}).
The trained weights of all 400 networks are included, so every analysis in
Section~\ref{sec:campaigns} can be recomputed without retraining. The repository also carries the
defect register and the decision log the manuscript refers to.

\section*{Declaration of generative AI and AI-assisted technologies in the writing process}
During the preparation of this work the authors used AI-assisted tools for exploratory mathematical
checking, drafting support, language editing and literature-search assistance. After using these
tools the authors reviewed and edited the content as needed and take full responsibility for the
content of the publication. AI tools are not listed as authors and did not determine the research
questions, proof strategy or final claims. All mathematical statements, proofs, citations and
comparisons with source papers were verified by the authors.

\section*{Changes from version 1, including three corrected statements}
\addcontentsline{toc}{section}{Changes from version 1}

Version~1 of this preprint appeared under the title \emph{Linear-Readout Floors and Threshold
Recovery for Computation in Superposition}. This version supersedes it. Most of what follows is new
work rather than correction, but three formal statements in v1 are wrong as written and two of its
empirical readings did not survive later measurement, so we list those first and separately. A reader
who cited v1 should read this section.

\paragraph{Three statements that were too strong, and are corrected here.} An external adversarial
audit found, and we verified independently against the implementation, that:

\begin{itemize}[leftmargin=*,itemsep=3pt]

\item Theorem~\ref{thm:arrow} and Corollary~\ref{cor:failthresh} were stated for all $h_i<1$, while
their shared proof establishes them only for $h_i\le 1/2$. Explicit unit-norm counterexamples at
$d=2$, $F=3$ give $\kappa=\infty$ against a finite bound. Both now carry the $h_i\le 1/2$
hypothesis, $\kappa_i$ is defined as $+\infty$ when the linear programme is infeasible, and the
counterexamples are regression tests in the released code
(\code{tests/test\_arrow\_counterexamples.py}).

\item A proposition concluded that a feature with $\kappa_i=1$ remains separable at $s=1$.
Theorem~\ref{thm:frontier} requires $\kappa_i>s$, so the correct conclusion is that such a feature is
separable at no positive sparsity. Corrected.

\item Theorem~\ref{thm:codefloor} claimed that the row-wise optimum coincides with the algebraic
pseudoinverse. It is the \emph{gain-calibrated} pseudoinverse $D_h^{-1}\Phi^{+}$. The clause is
removed.

\end{itemize}

No measurement in v1 or here depends on the invalid region. The corollary is only ever applied at
$\min_i h_i$, and the largest $\min_i h_i$ over all \KdThreeModels{} models of the main grid is
\KdThreeMaxLevMin, inside $h_i\le 1/2$. The tests that reported no violations sampled only that
region, which is a gap in the tests rather than a false test, and is why the error survived to
publication.

\paragraph{Two readings withdrawn.} A defect in how a shared decision threshold was selected --- the
search never scored the default it was meant to improve on --- invalidated two empirical claims made
in v1: that the two pre-registered estimands disagree, and that an affine probe of the preactivation
beats the network by a full sparsity level. Refitted correctly, the estimands agree at every width
and the pre-ReLU advantage is at most a quarter of a level on trained codes, reversing at the largest
width. Appendix~\ref{app:withdrawn} records these and six further withdrawals with the cause of each.

\paragraph{What is new.} The theory is unchanged apart from the corrections above. The empirical
content is substantially new: the leverage identity of Proposition~\ref{prop:leverage} is put to
predictive use rather than stated in passing; \SaeDicts{} released decoder matrices are measured
against shape-matched controls, with a trained-versus-randomly-initialised pair locating the cause;
the controlled campaigns are extended to four widths, two feature loads and two training sparsities;
a frozen-code arm separates what training the code contributes from what training the readout
contributes; and the ceiling on the decoder comparison is measured rather than argued. Every measured
number is generated from the released results by \code{scripts/make\_numbers.py} and checked against
this manuscript by \code{scripts/check\_manuscript\_numbers.py}.

\bibliographystyle{tmlr}
\bibliography{refs}

\appendix
\section{The Interface Diagnostic}
\label{sec:method}

\subsection{Problem statement}

\begin{problem}[Interface diagnosis]
\label{prob:main}
\textbf{Input:} a code $\Phi\in\R^{d\times F}$ with $F>d$ (or a trained network from which an
effective code is read off); a readout family $\mathcal G$; a sparse-state model, here
uniformly random supports of size $s$ drawn from a grid $\mathcal S$; a threshold $\theta$; a
trial budget $T$; a confidence level $1-\alpha$.

\textbf{Output:} (i) the attainment ratio $r=\widehat W/W(F,d)\ge1$ and the maximum-form ratio
$r_{\max}$, together with its factorisation $R_{\mathrm{geom}}\times R_{\mathrm{readout}}$
against the code's own floor (Section~\ref{sec:codefloor}); (ii) for each $s\in\mathcal S$, the
empirical $p_{\mathrm{rec}}(s)$ with a Wilson interval and the exact per-coordinate RMS error
$\rho_{\mathrm{lin}}(s)$ of the calibrated linear interface; (iii) the criterion-separation
witness $(s_{95},\rho_{\mathrm{lin}}(s_{95}))$; (iv) for each feature in a chosen subset, the
collision frontier $\kappa_i$ and hence the largest sparsity at which any affine probe can
detect that feature (Algorithm~\ref{alg:frontier}).
\end{problem}

The output is a diagnosis, not a score to be optimised, and the four parts answer different
questions. A code with $r$ close to $1$ is Welch-optimal in the mean-square sense, but only
$R_{\mathrm{readout}}$ says anything about the decoder; a large $s_{95}$ with
$\rho_{\mathrm{lin}}(s_{95})$ well above $d^{-1/2}$ says the code supports support recovery in a
regime where analog reconstruction from the same map is useless; and $\kappa_i$ is the only part
that is per-feature and exact rather than aggregate and empirical, which is what makes it usable
as a statement about an individual feature rather than about the code on average.

\subsection{Algorithm 1: floor attainment}

\begin{algorithm}[t]
\caption{\textsc{FloorAttainment} --- how close the readout interface sits to the floor}
\label{alg:one}
\begin{algorithmic}[1]
\Require code $\Phi\in\R^{d\times F}$, $F>d$; readout $G\in\R^{F\times d}$ or a rule in
  $\{\textsc{tied},\textsc{pinv},\textsc{lsq}\}$
\Ensure attainment ratio $r$ (and $r_{\max}$ if requested), or \textsc{Degenerate}
\State $G\gets\Phi^{\top}$, $\Phi^{+}$, or the least-squares fit, per the rule
\State $D_{ii}\gets\langle g_i,\phi_i\rangle$ for $i=1,\dots,F$
  \Comment{the diagonal of $M=G\Phi$, in $O(Fd)$; $M$ itself is not formed}
\If{$\min_i|D_{ii}|<\epsilon_{\mathrm{gain}}$} \Return \textsc{Degenerate}
  \Comment{some feature is read out with vanishing gain}
\EndIf
\State $\bar G\gets D^{-1}G$ \Comment{gain normalisation; see Corollary~\ref{cor:calibfree}}
\State $H\gets\bar G^{\top}\bar G$;\quad $\Sigma\gets\Phi\Phi^{\top}$
  \Comment{both $d\times d$; $O(Fd^2)$}
\State $\widehat W\gets\dfrac{\operatorname{tr}(H\Sigma)-F}{F(F-1)}$
  \Comment{exact, by Lemma~\ref{lem:suffstat}}
\State $W\gets(F-d)/(d(F-1))$ \Comment{Theorem~\ref{thm:floor}}
\If{the maximum form is requested}
  \State form $\widetilde M\gets\bar G\Phi$ and set
    $\widehat\mu\gets\max_{i\neq j}|\widetilde M_{ij}|$
    \Comment{$O(F^2d)$ time, $O(F^2)$ memory}
\EndIf
\State \Return $r\gets\widehat W/W$ \ (and $r_{\max}\gets\widehat\mu/\sqrt{W}$ if computed)
\end{algorithmic}
\end{algorithm}

Step~3 is not a numerical guard but a semantic one: without it, an interface can be made to
look arbitrarily good by shrinking $G$, which is exactly the confusion that
Remark~\ref{rem:calib} dispels. By Theorem~\ref{thm:floor}, a correct implementation always
returns $r\ge1$; we use that as a runtime assertion, and no violation was observed in
\EoneMeasurements{} gain-normalised interfaces (Section~\ref{sec:res-e1}).

The mean-square statistic that the theorem actually constrains never requires the interface to
be built. The maximum statistic does, and we say so rather than folding it into the same cost.

\begin{lemma}[Mean-square cross-talk from $d\times d$ statistics]
\label{lem:suffstat}
Let $\bar G=D^{-1}G$ with $D_{ii}=\langle g_i,\phi_i\rangle\neq0$, let
$\widetilde M=\bar G\Phi$, and set $H=\bar G^{\top}\bar G$ and $\Sigma=\Phi\Phi^{\top}$, both
in $\R^{d\times d}$. Then
\[
  \sum_{i\neq j}\widetilde M_{ij}^{2}
  \;=\;\norm{\bar G\Phi}_F^{2}-F
  \;=\;\operatorname{tr}(H\Sigma)-F .
\]
\end{lemma}

\begin{proof}
$\widetilde M$ has unit diagonal by construction, so the $F$ diagonal entries contribute
exactly $F$ to $\norm{\widetilde M}_F^2$ and the remainder is the off-diagonal sum. Cyclicity
of the trace gives $\norm{\bar G\Phi}_F^2=\operatorname{tr}(\Phi^{\top}\bar G^{\top}\bar
G\Phi)=\operatorname{tr}(\bar G^{\top}\bar G\,\Phi\Phi^{\top})=\operatorname{tr}(H\Sigma)$.
\end{proof}

Forming $H$ and $\Sigma$ costs $O(Fd^2)$ time and $O(d^2)$ memory, against $O(F^2d)$ and
$O(F^2)$ for the direct route --- a reduction of a factor $F/d$ in time and $F^2/d^2$ in
memory, with no approximation whatever. The identity is verified against the dense computation
to relative accuracy $10^{-10}$ in the accompanying tests.

\subsection{Algorithm 2: separation profile of a given code}

Algorithm~\ref{alg:two} is the diagnostic proper: it takes a code and a readout that already
exist --- a designed code, an optimised one, or the effective code of a trained network --- and
holds them fixed while the support varies. This is the only version that says anything about a
particular object, and it is the one applied to trained networks in
Section~\ref{sec:res-e5}.

\begin{algorithm}[t]
\caption{\textsc{SeparationProfile} --- support recovery versus irreducible analog error, on a
fixed code}
\label{alg:two}
\begin{algorithmic}[1]
\Require code $\Phi\in\R^{d\times F}$ and readout $G\in\R^{F\times d}$; sparsity grid
  $\mathcal S$; trials $T$; threshold $\theta$; seed $\sigma$
\Ensure $\{(s,p_{\mathrm{rec}}(s),\text{CI},\rho_{\mathrm{lin}}(s))\}_{s\in\mathcal S}$ and the
  separation witness $(s_{95},\rho_{\mathrm{lin}}(s_{95}))$
\State $s_{\max}\gets\max\mathcal S$;\quad $\bar G\gets D^{-1}G$ as in
  Algorithm~\ref{alg:one}
\State $\norm{A}_F^2\gets\operatorname{tr}(H\Sigma)-F$;\quad
       $\norm{A\mathbf 1}^2\gets\norm{\bar G\Phi\mathbf 1}^2
       -2\,\mathbf 1^{\top}\bar G\Phi\mathbf 1+F$
  \Comment{Lemma~\ref{lem:moments}; $O(Fd^2)$}
\State $e_s\gets\frac1F\!\left[(p-q)\norm{A}_F^2+q\norm{A\mathbf 1}^2\right]$,
  $p=\tfrac sF$, $q=\tfrac{s(s-1)}{F(F-1)}$, for each $s\in\mathcal S$
  \Comment{exact; Prop.~\ref{prop:uniform-energy}}
\For{$t=1$ \textbf{to} $T$}
  \State draw nested supports $S_1\subset\cdots\subset S_{s_{\max}}$ from seed $(\sigma,t)$
  \State $Z\gets\bar G\,[\,x_1,\dots,x_{s_{\max}}\,]$, $x_s=\sum_{j\in S_s}\phi_j$
    \Comment{$O(Fd\,s_{\max})$}
  \State $k_s\mathrel{+}=1$ for each $s$ with $\mathbf 1\{Z_{\cdot s}\ge\theta\}=\mathbf 1_{S_s}$
\EndFor
\State $p_{\mathrm{rec}}(s)\gets k_s/T$ with Wilson interval;\quad
       $\rho_{\mathrm{lin}}(s)\gets\sqrt{e_s}$
\State $s_{95}\gets\max\{s:\ p_{\mathrm{rec}}(s')\ge0.95\ \forall s'\le s\}$
\State \Return the profile and $(s_{95},\rho_{\mathrm{lin}}(s_{95}))$
\end{algorithmic}
\end{algorithm}

The support-recovery side of Algorithm~\ref{alg:two} is Monte Carlo over supports; the analog
side is \emph{exact} for the given code, computed once before the loop rather than estimated
inside it. Total cost is $O(Fd^2+TFd\,s_{\max})$ time and $O(Fd+d^2)$ memory: the $F\times F$
interface is never formed.

\subsection{A scaling experiment on the random-code ensemble}
\label{sec:alg-ensemble}

Algorithm~\ref{alg:ensemble} is a different object and we keep it named apart. It redraws a
fresh code at every trial, so what it characterises is the $(d,F)$ random-code \emph{ensemble},
not any individual code; it is the vehicle for the width-scaling study of
Sections~\ref{sec:res-e1} and~\ref{sec:res-e6}, and it is not a diagnosis of anything. Because
the code is regenerated block by block from a counted seed and never stored, it reaches widths
at which the code would not fit in memory.

\begin{algorithm}[t]
\caption{\textsc{EnsembleScaling} --- the same two criteria, averaged over random codes}
\label{alg:ensemble}
\begin{algorithmic}[1]
\Require $d$, $F$, sparsity grid $\mathcal S$, trials $T$, threshold $\theta$, seed $\sigma$,
  block size $B$
\Ensure the ensemble-averaged profile and $(s_{95},\rho_{\mathrm{lin}}(s_{95}))$
\State $s_{\max}\gets\max\mathcal S$;\quad $k_s\gets0$ for all $s$;\quad $e_s\gets0$ for all $s$
\For{$t=1$ \textbf{to} $T$}
  \State draw nested supports $S_1\subset\cdots\subset S_{s_{\max}}$ from seed $(\sigma,t)$
  \State $X\gets$ partial sums $[\,x_1,\dots,x_{s_{\max}}\,]$, $x_s=\sum_{j\in S_s}\phi_j$
    \Comment{$d\times s_{\max}$}
  \State $\Sigma\gets 0_{d\times d}$,\ $v\gets 0_d$
  \For{each block $\beta$ of $B$ columns of $\Phi$}
    \State regenerate $\Phi_\beta$ from seed $(\sigma,t,\beta)$ \Comment{$\Phi$ is never stored}
    \State $Z\gets\Phi_\beta^{\top}X$; mark $s$ as failed if
      $\mathbf 1\{Z_{\cdot s}\ge\theta\}$ disagrees with $\mathbf 1_{S_s}$ on this block
    \State $\Sigma\mathrel{+}=\Phi_\beta\Phi_\beta^{\top}$;\quad
           $v\mathrel{+}=\Phi_\beta\mathbf 1$
  \EndFor
  \State $k_s\mathrel{+}=1$ for each surviving $s$
  \State $\norm{A}_F^2\gets\norm{\Sigma}_F^2-F$;\quad
         $\norm{A\mathbf 1}^2\gets v^{\top}\Sigma v-2\norm{v}^2+F$
    \Comment{Lemma~\ref{lem:moments}}
  \State $e_s\mathrel{+}=\frac1F\!\left[(p-q)\norm{A}_F^2+q\norm{A\mathbf 1}^2\right]$,
    $p=\tfrac sF$, $q=\tfrac{s(s-1)}{F(F-1)}$ \Comment{Prop.~\ref{prop:uniform-energy}}
\EndFor
\State $p_{\mathrm{rec}}(s)\gets k_s/T$ with Wilson interval;\quad
       $\rho_{\mathrm{lin}}(s)\gets\sqrt{e_s/T}$
\State $s_{95}\gets\max\{s:\ p_{\mathrm{rec}}(s')\ge0.95\ \forall s'\le s\}$
\State \Return the profile and $(s_{95},\rho_{\mathrm{lin}}(s_{95}))$
\end{algorithmic}
\end{algorithm}

Both algorithms share the same asymmetry: the support-recovery side is estimated, the analog
side is computed. That is what makes the comparison a witness rather than two noisy curves ---
the analog error carries no Monte Carlo uncertainty, so any gap at $s_{95}$ is attributable to
the recovery side alone.

\begin{lemma}[Sufficient statistics of the tied interface]
\label{lem:moments}
Let $\Phi$ have unit-norm columns, $M=\Phi^{\top}\Phi$, $A=M-I_F$,
$\Sigma=\Phi\Phi^{\top}\in\R^{d\times d}$ and $v=\Phi\mathbf 1\in\R^d$. Then
\[
  \norm{A}_F^2=\norm{\Sigma}_F^2-F,
  \qquad
  \norm{A\mathbf 1}_2^2=v^{\top}\Sigma v-2\norm{v}_2^2+F .
\]
\end{lemma}

\begin{proof}
By cyclicity of the trace, $\norm{\Phi^{\top}\Phi}_F^2=\operatorname{tr}((\Phi^{\top}\Phi)^2)
=\operatorname{tr}((\Phi\Phi^{\top})^2)=\norm{\Sigma}_F^2$; the diagonal of $M$ is all ones and
contributes $F$, so subtracting $F$ leaves the off-diagonal energy, which is $\norm{A}_F^2$.
For the second identity, $A\mathbf 1=\Phi^{\top}\Phi\mathbf 1-\mathbf 1=\Phi^{\top}v-\mathbf 1$,
hence $\norm{A\mathbf 1}^2=v^{\top}\Phi\Phi^{\top}v-2\mathbf 1^{\top}\Phi^{\top}v+F
=v^{\top}\Sigma v-2\norm{v}^2+F$, using $\mathbf 1^{\top}\Phi^{\top}v=v^{\top}v$.
\end{proof}

Lemma~\ref{lem:moments} is the reason the diagnostic scales: it replaces an $F\times F$ object
by a $d\times d$ one, and combined with Proposition~\ref{prop:uniform-energy} it yields the
average linear error at every sparsity without a single Monte Carlo draw.

\subsection{Hyperparameters, complexity and cost}

The diagnostic has five hyperparameters: the readout rule, the threshold $\theta$ (default
$1/2$, the natural midpoint for unit-norm codes and Boolean states), the sparsity grid
$\mathcal S$, the trial budget $T$, and the block size $B$, which trades memory for the number
of regeneration passes. Only $T$ affects the statistical resolution;
Section~\ref{sec:res-e1} reports the threshold sensitivity of the conclusions.

Algorithm~\ref{alg:one} costs $O(Fd^2)$ time and $O(d^2)$ memory for the ratio $r$, by
Lemma~\ref{lem:suffstat}. The maximum-form ratio $r_{\max}$ is the one exception in the whole
method: no $d\times d$ reduction of a maximum is available to us, so obtaining it exactly costs
$O(F^2d)$ time and $O(F^2)$ memory, and it is computed only when asked for. Since the floor of
Theorem~\ref{thm:floor} is a statement about the mean square, nothing about the floor
comparison is lost by omitting it; we report $r_{\max}$ at the moderate widths where it is
affordable and drop it beyond. Algorithm~\ref{alg:two} costs $O(Fd^2+T\,F\,d\,s_{\max})$ time
and $O(Fd+d^2)$ memory. Algorithm~\ref{alg:ensemble} costs $O(T\,F\,d\,
s_{\max})$ time --- one pass of code regeneration and one $B\times d\times s_{\max}$ product
per block --- and $O(Bd+d^2)$ memory, independent of $F$. At $d=\EsixDmax$ and
$F=\EsixFmax$ the code would occupy $4$~GB in single precision; the streaming implementation
peaks near $B\,d\cdot4$ bytes. Measured wall-clock times are in Table~\ref{tab:cost}.

\subsection{Applicability, edge cases and failure modes}

\paragraph{Conditions of applicability.} The diagnostic assumes (i) $F>d$, otherwise the floor
is vacuous and Algorithm~\ref{alg:one} refuses to run; (ii) unit-norm code columns, or a
declared normalisation, since the threshold $\theta=1/2$ is calibrated to unit gain;
(iii) a sparse-state model that is exchangeable across coordinates, which is what makes the
closed form of step~3 valid.

\paragraph{Edge cases.} At $s=1$ the linear error is still non-zero while threshold recovery
succeeds whenever $\mu<\theta$; this is the smallest instance of the separation. When
$F/d\to1$ the floor tends to zero and the diagnosis becomes uninformative --- correctly so,
since a nearly complete code has no interference to speak of.

\paragraph{Failure modes.} Three are worth naming. (a) \emph{Degenerate gain}: if some
$|M_{ii}|$ is numerically zero the calibration is undefined and Algorithm~\ref{alg:one}
returns \textsc{Degenerate} rather than a number; this is not a corner case but the expected
outcome when a readout is fitted without the diagonal constraint
(Section~\ref{sec:res-e4}). (b) \emph{Uncovered feature}: a least-squares readout fitted on
states in which some feature never appears has an identically zero row, which triggers (a);
the implementation therefore enforces coverage of the fitting set. (c) \emph{Vacuous profile}:
if the coherence of the code exceeds $\theta$, $p_{\mathrm{rec}}(1)$ is already below $0.95$
and $s_{95}=0$; the certificate is then empty, and the diagnostic reports that rather than
extrapolating.

\subsection{Implementation and traceability}

\begin{table}[t]
\centering
\caption{Each step of the method and the function that implements it. Paths are relative to
the repository root.}
\label{tab:codemap}
\small
\begin{tabular}{lll}
\toprule
quantity & where defined & implementation \\
\midrule
$W(F,d)$ & Eq.~\eqref{eq:floor} & \code{lrtr/codes.py:welch\_floor} \\
$D^{-1}M$ & Cor.~\ref{cor:calibfree} & \code{lrtr/interface.py:unit\_diagonal} \\
$\widehat W,\widehat\mu,r,r_{\max}$ & Alg.~\ref{alg:one} & \code{lrtr/interface.py:crosstalk\_stats} \\
Algorithm~\ref{alg:one} & Sec.~\ref{sec:method} & \code{lrtr/diagnostic.py:interface\_floor\_diagnostic} \\
$\Sigma,v\mapsto\norm{A}_F^2,\norm{A\mathbf 1}^2$ & Lem.~\ref{lem:moments} & \code{lrtr/diagnostic.py:tied\_energy\_moments} \\
$\frac1F\E_S\norm{A\mathbf 1_S}^2$ & Prop.~\ref{prop:uniform-energy} & \code{lrtr/interface.py:linear\_energy\_uniform} \\
energy floor & Eq.~\eqref{eq:energyfloor} & \code{lrtr/interface.py:energy\_floor\_uniform} \\
threshold decoder $Q$ & Thm.~\ref{thm:coherence} & \code{lrtr/threshold.py:threshold\_decode} \\
fixed-code trial & Alg.~
ef{alg:two}, l.~5--7 & \code{lrtr/threshold.py:recovery\_trial\_fixed} \\
streaming trial & Alg.~
ef{alg:ensemble}, l.~6--10 & \code{lrtr/threshold.py:recovery\_trial\_streaming} \\
Algorithm~
ef{alg:two} & Sec.~
ef{sec:method} & \code{lrtr/diagnostic.py:fixed\_code\_separation\_profile} \\
Algorithm~
ef{alg:ensemble} & Sec.~
ef{sec:alg-ensemble} & \code{lrtr/diagnostic.py:random\_code\_scaling\_experiment} \\
$\operatorname{tr}(H\Sigma)-F$ & Lem.~
ef{lem:suffstat} & \code{lrtr/interface.py:crosstalk\_mean\_sq} \\
leverage form of $r$ & Prop.~
ef{prop:leverage} & \code{lrtr/interface.py:crosstalk\_mean\_sq\_pinv} \\
$s_{95}$ & Alg.~
ef{alg:two}, l.~9 & \code{lrtr/threshold.py:s95\_from\_curve} \\
\bottomrule
\end{tabular}
\end{table}

Table~\ref{tab:codemap} maps every quantity in Problem~\ref{prob:main} to the function that
computes it. The correctness of the implementation is checked by a test suite that verifies the
theorems numerically --- equality for tight frames, no violations for random interfaces, the
closed forms against Monte Carlo, and the streaming trial against an explicit dense evaluation
of the same code.


\section{Threshold recovery and the separation}
\label{sec:threshold}

The floor rules out small-error $\varepsilon$-linear readout with $\varepsilon=o(d^{-1/2})$
when $F\gg d$. It does not rule out thresholded Boolean recovery, because thresholding does
not require every inactive score to be small --- only that the aggregate interference stay
below the margin.

\begin{theorem}[Threshold recovery under coherent superposition]
\label{thm:coherence}
Let $\Phi$ have unit-norm columns with coherence $\mu$, let $b$ be $s$-sparse, let
$x=\Phi b+\eta$ with $\norm{\Phi^{\top}\eta}_\infty\le\nu$, and let $Q(z)_i=\mathbf
1\{z_i\ge1/2\}$. If $s\mu+\nu<1/2$ then $Q(\Phi^{\top}x)=b$, with margin
$\tau=\tfrac12-(s\mu+\nu)>0$.
\end{theorem}

\begin{proof}
For $i\in S=\operatorname{supp}(b)$ we have
$z_i=1+\sum_{j\in S,\,j\neq i}\langle\phi_i,\phi_j\rangle+\langle\phi_i,\eta\rangle
\ge1-(s-1)\mu-\nu\ge\tfrac12+\tau$. For $i\notin S$,
$z_i\le s\mu+\nu=\tfrac12-\tau$.
\end{proof}

\begin{corollary}[Random codes at quadratic load]
\label{cor:randomcode}
Let $F=\lfloor d^2\rfloor$ and let $\phi_1,\dots,\phi_F$ be independent uniform unit vectors in
$\R^d$. Then $\mu\le6\sqrt{\log d/d}$ with probability at least $1-d^{-5}$ for all
sufficiently large $d$, and on that event the threshold decoder recovers every $s$-sparse
Boolean state whenever $\norm{\Phi^{\top}\eta}_\infty<\tfrac12-6s\sqrt{\log d/d}$.
\end{corollary}

\begin{proof}
For independent uniform unit vectors $u,v$, the inner product $\langle u,v\rangle$ has the law
of the first coordinate of a uniform point on $S^{d-1}$, whose tail satisfies
$\Pr\{|\langle u,v\rangle|>t\}\le2\exp(-(d-1)t^2/2)$ for $0<t<1$
(see e.g.~\cite[Thm.~3.4.6]{vershynin2018hdp}). Taking a union bound over the
$\binom{F}{2}$ unordered pairs and using $2\binom F2\le F^2\le d^4$,
\[
  \Pr\{\mu>t\}\;\le\;d^4\exp\!\left(-\tfrac{(d-1)t^2}{2}\right).
\]
With $t=6\sqrt{\log d/d}$ we have $\tfrac{(d-1)t^2}{2}=18(1-1/d)\log d\ge9\log d$ for $d\ge2$,
so $\Pr\{\mu>t\}\le d^4 e^{-9\log d}=d^{-5}$. The recovery claim is then
Theorem~\ref{thm:coherence} with $\mu\le t$.
\end{proof}

The worst-case coherence bound is pessimistic for a \emph{random} support, which is the case
the diagnostic measures. The following average-case statement is the one used in
Algorithms~\ref{alg:two} and~\ref{alg:ensemble}.

\begin{theorem}[Random-support threshold recovery]
\label{thm:randomsupport}
Let $\Phi$ have independent uniform unit-vector columns and let $S$ be a uniformly random
support of size $s$, independent of $\Phi$. For $\delta\in(0,1)$ there is a universal $C>0$
such that, with probability at least $1-\delta$,
\[
  \max_{i\in[F]}\Big|\sum_{j\in S,\,j\neq i}\langle\phi_i,\phi_j\rangle\Big|
  \;\le\;C\sqrt{\frac{s\log(F/\delta)}{d}},
\]
and on that event thresholding $\Phi^{\top}(\Phi\mathbf 1_S+\eta)$ at $1/2$ recovers
$\mathbf 1_S$ whenever $\norm{\Phi^{\top}\eta}_\infty+C\sqrt{s\log(F/\delta)/d}<1/2$.
\end{theorem}

\begin{proof}
Conditionally on $S$ and on $\phi_i$, the summands are independent and each is distributed as
the first coordinate of a uniform point on $S^{d-1}$, hence sub-Gaussian with variance proxy
$O(1/d)$~\cite[Thm.~3.4.6]{vershynin2018hdp}. The sum of at most $s$ of them is sub-Gaussian
with proxy $O(s/d)$, so $\Pr\{|I_i|>t\mid S\}\le2\exp(-c_0dt^2/s)$. A union bound over
$i\in[F]$ and the stated choice of $t$ give the claim; the recovery statement follows as in
Theorem~\ref{thm:coherence}. Full details are in Appendix~\ref{app:proofs}.
\end{proof}

\begin{remark}[Which randomness the probability is over]
\label{rem:annealed}
The guarantee of Theorem~\ref{thm:randomsupport} is \emph{annealed}: the probability is taken
over the joint draw of the code $\Phi$ and the support $S$. It is therefore the right statement
for Algorithm~\ref{alg:ensemble}, which redraws the code at every trial, and it is not directly
the statement for Algorithm~\ref{alg:two}, which holds one code fixed. A quenched version ---
$\Phi$ fixed, probability over $S$ alone --- follows from concentration for sampling without
replacement, but it is not free: it needs hypotheses on the particular $\Phi$ (bounded
coherence and bounded row sums of $|\Phi^{\top}\Phi|-I$), which a given trained code may or may
not satisfy. We do not assume them. What the fixed-code diagnostic reports is therefore an
empirical recovery rate with a Wilson interval, and the theorem is quoted as the ensemble-level
expectation it is, not as a guarantee about the code in hand. The analog side has no such
caveat: Proposition~\ref{prop:uniform-energy} is exact for every fixed $\Phi$.
\end{remark}

We now give the average linear-readout error under the \emph{same} distribution, which is what
allows the two criteria to be compared without a distributional mismatch.

\begin{proposition}[Average linear energy floor, uniform random supports]
\label{prop:uniform-energy}
Let $F>d$, let $M=G\Phi\in\R^{F\times F}$ satisfy $\operatorname{rank}(M)\le d$ and $M_{ii}=1$
for all $i$, and set $A=M-I_F$. Let $S\subseteq[F]$ be a uniformly random support of size $s$
with $1\le s\le F/2$ and let $b=\mathbf 1_S$. Then
\begin{equation}
  \E_S\norm{Ab}_2^2
  \;\ge\;\frac{s(F-s)}{F(F-1)}\norm{A}_F^2
  \;\ge\;\frac{s(F-s)(F-d)}{d(F-1)} .
  \label{eq:energyfloor}
\end{equation}
In particular $\frac1F\E_S\norm{Ab}_2^2\ge\frac{s(F-d)}{2d(F-1)}$, which is $\Omega(s/d)$
whenever $F/d\to\infty$.
\end{proposition}

\begin{proof}
For a uniformly random size-$s$ support, $\E[b_i]=s/F$ and $\E[b_ib_j]=\frac{s(s-1)}{F(F-1)}$
for $i\neq j$. Hence $\E[bb^{\top}]=(p-q)I_F+q\,\mathbf 1\mathbf 1^{\top}$ with $p=s/F$ and
$q=\frac{s(s-1)}{F(F-1)}$, and
\[
  p-q=\frac{s(F-1)-s(s-1)}{F(F-1)}=\frac{s(F-s)}{F(F-1)} .
\]
Therefore
$\E_S\norm{Ab}_2^2=\operatorname{tr}(A^{\top}A\,\E[bb^{\top}])
=(p-q)\norm{A}_F^2+q\norm{A\mathbf 1}_2^2\ge(p-q)\norm{A}_F^2$.
Since $A$ has zero diagonal and $\operatorname{rank}(M)\le d$ with unit diagonal,
Theorem~\ref{thm:floor} gives $\norm{A}_F^2\ge F(F-d)/d$, which yields the second inequality.
For $s\le F/2$ we have $F-s\ge F/2$, and dividing by $F$ gives the per-coordinate form.
\end{proof}

\begin{corollary}[Separation at quadratic load, one sparse-state model]
\label{cor:separation}
Let $F=d^2$ and let $\Phi$ have independent uniform unit-vector columns. There is a universal
$c>0$ such that, for $1\le s\le c\,d/\log d$ and a uniformly random support $S$ of size $s$,
thresholding $\Phi^{\top}\Phi\mathbf 1_S$ at $1/2$ recovers $\mathbf 1_S$ exactly with
probability at least $1-d^{-10}$ for all sufficiently large $d$. For the \emph{same} sparse-state
model, every unit-diagonal rank-$d$ linear readout has average per-coordinate squared error at
least $\frac{s(F-d)}{2d(F-1)}=\Omega(s/d)$, and hence RMS error $\Omega(\sqrt{s/d})$.
\end{corollary}

\begin{proof}
Apply Theorem~\ref{thm:randomsupport} with $\delta=d^{-10}$, so $\log(F/\delta)=12\log d$; the
interference bound is $C\sqrt{12s\log d/d}$, which is below $1/2$ as soon as
$s\le c\,d/\log d$ with $c<1/(48C^2)$. The second statement is
Proposition~\ref{prop:uniform-energy} with $F=d^2$.
\end{proof}

\begin{remark}[What the separation compares]
\label{rem:criteria}
The two sides compare different \emph{interface criteria} on the same input distribution:
expected squared linear readout error versus exact-recovery probability of a threshold
decoder. Neither criterion dominates the other in general. The content of
Corollary~\ref{cor:separation} is that at quadratic feature load one of them is attainable
exactly and the other is provably not, in the same regime --- which is precisely the quantity
Algorithm~\ref{alg:two} reports. Corollary~\ref{cor:separation} is stated for uniform supports;
the Bernoulli variant, with the same conclusion, is in Appendix~\ref{app:proofs}.
\end{remark}


\section{The hierarchy is an implication, and only one way}
\label{sec:hierarchy}

We now have three criteria on the same code: analog reconstruction
(Section~\ref{sec:codefloor}), affine support recovery (Section~\ref{sec:frontier}), and whatever
a trained nonlinear decoder achieves. Calling them a hierarchy is only justified if one implies
another. It does, in one direction, and the implication runs through leverage --- the same
statistic that Theorem~\ref{thm:codefloor} showed governs the analog level.

\begin{lemma}[Leverage is redundancy, exactly]
\label{lem:sm}
If $\Phi$ has full row rank and $h_i<1$ then
$\min\{\norm{z}_2^2:\Phi_{-i}z=\phi_i\}=h_i/(1-h_i)$.
\end{lemma}

\begin{proof}
The minimum-norm solution gives
$\min\norm{z}_2^2=\phi_i^{\top}(\Phi_{-i}\Phi_{-i}^{\top})^{-1}\phi_i$, and
$\Phi_{-i}\Phi_{-i}^{\top}=\Sigma-\phi_i\phi_i^{\top}$. Sherman--Morrison gives
$(\Sigma-\phi_i\phi_i^{\top})^{-1}=\Sigma^{-1}+\Sigma^{-1}\phi_i\phi_i^{\top}\Sigma^{-1}/(1-h_i)$,
so the value is $h_i+h_i^2/(1-h_i)=h_i/(1-h_i)$.
\end{proof}

So a feature with low leverage is reproducible by the others with small coefficients, with no
detour through coherence. That is the missing step: it converts a statement about the analog
geometry into a bound on the affine frontier.

\begin{theorem}[Low leverage forces the frontier down]
\label{thm:arrow}
For every $i$ with $h_i\le\tfrac12$,
\begin{equation}
  \kappa_i\;\le\;1+\sqrt{\frac{(F-1)\,h_i}{1-h_i}} .
  \label{eq:arrow}
\end{equation}
\end{theorem}

\begin{proof}
Take the $z$ of Lemma~\ref{lem:sm}, so $\Phi_{-i}z=\phi_i$ and
$\norm{z}_2=\sqrt{h_i/(1-h_i)}$. Then $\norm{z}_1\le\sqrt{F-1}\,\norm{z}_2$ and
$\max(\sum_jz^{+}_j,1+\sum_jz^{-}_j)\le1+\norm{z}_1$. For $z$ to be \emph{admissible} it must also
satisfy the box, and $h_i\le1/2$ gives $\norm{z}_2\le1$ and hence $\norm{z}_\infty\le1$. So this
$z$ witnesses the right-hand side.
\end{proof}

\begin{remark}[The hypothesis is necessary, not cosmetic]
\label{rem:arrow-hyp}
Above $h_i=1/2$ the minimum-$\ell_2$ representation need not satisfy the box, and the inequality
is then not merely unproved but false. Take unit-norm columns in $\R^2$,
\[
  \phi_1=(1,0),\qquad
  \phi_2=\Big(\tfrac14,\tfrac{\sqrt{15}}4\Big),\qquad
  \phi_3=\Big(\tfrac14,-\tfrac{\sqrt{15}}4\Big).
\]
Here $h_1=8/9$, so the right-hand side of~\eqref{eq:arrow} equals $5$. But the second coordinate
forces $z_2=z_3$ and the first then forces $z_2=z_3=2$, so the only representation of $\phi_1$ by
the other two columns violates $\norm{z}_\infty\le1$: the programme~\eqref{eq:kappa} is infeasible
and $\kappa_1=+\infty$. An earlier version of this work stated~\eqref{eq:arrow} for all $h_i<1$;
that was wrong, and the counterexample is in the test suite.
\end{remark}

\begin{corollary}[A failure threshold with no decoder in it]
\label{cor:failthresh}
If
\[
  h_i\;\le\;\min\left\{\frac12,\;\frac{(s-1)^2}{(F-1)+(s-1)^2}\right\}
\]
then feature $i$ is \emph{not} affinely separable at sparsity $s$.
\end{corollary}

\begin{proof}
Solve $1+\sqrt{(F-1)h/(1-h)}\le s$ for $h$, intersect with the hypothesis of
Theorem~\ref{thm:arrow}, and apply Theorem~\ref{thm:frontier}.
\end{proof}

The cap only binds for $s-1>\sqrt{F-1}$; below that the second expression is already at most a
half and the statement is the unrestricted one. Dropping it makes the corollary false, and by the
same mechanism as Remark~\ref{rem:arrow-hyp}: with $a=(1,0)$,
$b=(-\tfrac{25}{44},\tfrac{\sqrt{1311}}{44})$ and
$\phi_i=(-\tfrac{17}{40},\tfrac{\sqrt{1311}}{40})$, all unit-norm in $\R^2$, the unique
representation is $z=(\tfrac15,\tfrac{11}{10})$, so $\kappa_i=+\infty$ and the feature is separable
at every sparsity --- yet $h_i=5/9$ lies below the uncapped $s=3$ threshold of $2/3$.

Corollary~\ref{cor:failthresh} is a statement about the code alone: no training, no decoder, no
probe. It predicts affine collapse from a single leverage score. Its cost is looseness --- the
$\sqrt{F-1}$ step is lossy, and Section~\ref{sec:res-e5} reports a case where it predicts
failure at $s=11$ against a measured frontier of $4$. It is a sufficient condition for collapse,
not an estimate of the frontier.

\begin{proposition}[The converse is false]
\label{prop:oneway}
Let $\Phi=[\,I_d\;\;I_d\,]\in\R^{d\times2d}$. Then $h_i=1/2$ for every $i$, so leverage is
perfectly uniform and $W_{\star}(\Phi)=W(F,d)$: the analog geometry is optimal. Yet
$\kappa_i=1$ for every $i$, the smallest value the frontier can take, so no feature is separable at
any positive sparsity.
\end{proposition}

\begin{proof}
$\Sigma=2I_d$, so $h_i=\phi_i^{\top}\phi_i/2=1/2$ and Theorem~\ref{thm:codefloor} gives equality.
Each column $i$ has a duplicate $j$ with $\phi_j=\phi_i$; taking $z=e_j$ gives
$\Phi_{-i}z=\phi_i$ with $\sum_jz^{+}_j=1$, $\sum_jz^{-}_j=0$ and $\norm{z}_\infty=1$, so
$\kappa_i\le\max(1,1)=1$. The reverse holds for every code and every feature, since
$z^{-}\ge0$ forces the objective to be at least $1+\sum_jz^{-}_j\ge1$. Hence $\kappa_i=1$.
Theorem~\ref{thm:frontier} requires $\kappa_i>s$, and $1>1$ is false, so feature $i$ is separable
at \emph{no} positive sparsity: the largest separable sparsity is $\lceil\kappa_i\rceil-1=0$.
Directly: the singleton state $e_i$ and the singleton state $e_j$ have identical representations,
so no rule of any kind can tell them apart.
\end{proof}

Together, Theorem~\ref{thm:arrow} and Proposition~\ref{prop:oneway} say:

\begin{quote}
Bad analog geometry \emph{implies} a bad affine frontier. Good analog geometry does \emph{not}
imply a good one.
\end{quote}

That is a strict one-way hierarchy, with the forward arrow proved and the converse refuted by an
explicit code, and it is what makes the word ``hierarchy'' literally true here rather than a
label attached to three separate measurements.

\subsection{The nonlinear level: a witness, not a frontier}
\label{sec:witness}

The third level has no frontier to compute. ``All nonlinear decoders'' has no capacity without a
restricted class, and restricting it leads to sign-rank, which we do not open. What it does have
is a finite certificate, and it comes free from Proposition~\ref{prop:farkas}: the states
appearing in~\eqref{eq:farkas} are explicit vectors that \emph{no} affine rule can label
correctly, and Radon's theorem says the obstruction compresses to at most $d+2$ of them. A
trained network either resolves them or does not, with margin
\[
  \gamma_{\mathrm{net}}
  =\min\Big(\min_k N_i(x^{+}_k)-\theta,\;\min_l\theta-N_i(x^{-}_l)\Big),
\]
which is a sharper report than a recovery curve: not ``a solver says the hulls intersect
somewhere among the $\binom{F}{s}$ admissible states'' but a named, listed set that any third
party can rebuild from $\Phi$ and check.

The extraction is implemented (\code{lrtr.affine\_frontier.radon\_witness}) and validated by
asking a linear programme for an affine separator of exactly the extracted states and requiring
it to be infeasible; on small codes it is also checked against brute-force enumeration of every
Boolean state. Two honest qualifications. The compression to $d+2$ is \emph{not} implemented, so
the witness we return is valid but not minimal --- on the codes we have examined it comes out at
about $d+2$ states already, which is small against $\binom{F}{s}$ but is not a handful, and the
per-state margin $\gamma_{\mathrm{net}}$ is therefore a minimum over that many states. And no
result in this paper is yet derived from it; the chain
$G_1\to G_2\to G_3$ is complete as a construction, and reporting witnesses for trained models is
future work rather than a claim made here.


\section{Application: two frameworks, two invariants}
\label{sec:application}

We now apply the interface distinction to the two frameworks whose apparent tension motivated
this work. All statements imported from the source papers were checked against
arXiv:2408.05451v1 and arXiv:2409.15318v3; the audit is recorded in
Appendix~\ref{app:imported}.

\begin{lemma}[Computation-layer interface; restated from {\cite[Thm.~11]{hanni2024superposition}}]
\label{lem:comp}
Fix $s=O(1)$ and suppose the targeted superpositional AND primitive of H\"{a}nni et al.\ is
applied in a regime satisfying its stated graph-balancing, sparsity, norm and
incoming-interference hypotheses. Then the computation layer produces outgoing readout error
$\ecomp(d)=\Otilde(\sqrt{s^2/d})=\Otilde(d^{-1/2})$.
\end{lemma}

\begin{lemma}[Correction-layer tolerance; restated from {\cite[Thm.~21]{hanni2024superposition}}]
\label{lem:corr}
The correction layer admits incoming error
$\varepsilon_{\mathrm{in}}<K(d)\,d^{1/4}/(F^{1/2}s^{1/4})$ with $K(d)=\mathrm{polylog}(d)$, and
produces $\varepsilon_{\mathrm{out}}=O((\log d)\sqrt{s/d})$ and
$\mu_{\mathrm{out}}=\Otilde(\sqrt{s/d})$.
\end{lemma}

\begin{proposition}[Template compatibility threshold; derived from Lemmas~\ref{lem:comp}--\ref{lem:corr}]
\label{prop:hanni}
Fix $s=O(1)$. The published H\"{a}nni computation--correction template is certified in the
regime $F\lesssim\Otilde(d^{3/2})$, and its recursive construction emulates circuits of width
$m$ using width $d=\Otilde(m^{2/3}s^2)$, giving a matching certified scale
$\Thetatilde(d^{3/2})$. This is a template-specific certificate, not an impossibility result
for other implementations.
\end{proposition}

\begin{proof}
The correction theorem applies after the computation layer when $\ecomp(d)\lesssim
\ecorr(F,d,s)$, i.e.\ $\Otilde(d^{-1/2})\lesssim K(d)\,d^{1/4}/F^{1/2}$ for $s=O(1)$.
Suppressing polylogarithmic factors, $F^{1/2}\lesssim d^{3/4}$, hence
$F\lesssim\Otilde(d^{3/2})$. Conversely, inverting $d=\Otilde(m^{2/3}s^2)$ at $s=O(1)$ gives
$m=\Omegatilde(d^{3/2})$.
\end{proof}

In the Adler--Shavit framework, exact Boolean state is restored by thresholding after every
stage. Their parameter-description lower bound $\Omega(m'\log m')$, converted to neurons under
their architectural assumption of $\Theta(n^2)$ parameters with $O(1)$ average description
length each, gives $\FrecAS(n)\le O(n^2/\log n)$; their explicit construction uses
$n=O(\sqrt{m'}\log m')$ neurons, giving $\FrecAS(n)\ge\Omega(n^2/\log^2 n)$. The remaining
logarithmic gap is open.

The two regimes do not contradict each other, and the reason is now precise. The H\"{a}nni
template maintains an approximate $\varepsilon$-linear interface, so it must feed its residual
error into a correction theorem whose tolerance is $\Otilde(d^{1/4}/F^{1/2})$; matching that
tolerance to the unavoidable $d^{-1/2}$ cross-talk scale of Theorem~\ref{thm:floor} gives the
$d^{3/2}$ threshold. Adler--Shavit instead threshold, leaving the small-error class entirely,
so Corollary~\ref{cor:impossible} does not apply to them. In the vocabulary of
Section~\ref{sec:method}: the first framework operates where the diagnostic's attainment ratio
matters, the second where its criterion-separation witness does.


\section{Experimental setup}
\label{sec:setup}

Six campaigns validate the diagnostic, from controlled random codes to a trained network.
E1--E3 establish that the implementation measures what the theory describes; E4 tests whether
\emph{learned} codes attain the floor; E5 applies the diagnostic to trained models; E6 tests
how the recovery threshold scales.

\paragraph{Environment.} All experiments run on CPU only, with no CUDA device.
Python~\EnvPython, NumPy~\EnvNumpy, SciPy~\EnvScipy, PyTorch~\EnvTorch{} on a
\EnvCpuCount-core machine; thread counts are pinned and two cores are reserved for the system.
Every campaign writes a run record with the exact command, seeds, configuration, environment,
wall-clock duration and exit status. Total measured compute is \DurTotalMin{} minutes
(Table~\ref{tab:cost}).

\paragraph{E1 --- floor on random codes.} $d\in\{16,32,64,128\}$, $F/d\in\{2,4,8,16\}$,
\EoneTrials{} independent codes per configuration, three readouts (tied, pseudoinverse and an
i.i.d.\ Gaussian readout as an adversarial control), and a harmonic tight frame at each
$(d,F)$ as the equality witness.

\paragraph{E2 --- recovery and separation at $F=d^2$.} $d\in\{64,128,256\}$,
\EtwoTrials{} trials per configuration, thresholds
$\theta\in\{0.4,0.5,0.6\}$ (\EtwoThetaCount{} values) as a sensitivity ablation.

\paragraph{E3 --- average linear energy.} $d\in\{32,64,128\}$ at $F=8d$, sparsities
$s\in\{1,2,4,8,16,32,64,128\}$, energies computed in closed form under both the Bernoulli and
the uniform-support model, with a \EthreeMcTrials-draw Monte Carlo cross-check of the closed
form.

\paragraph{E4 --- optimised codes.} $d\in\{16,32,64\}$, $F/d\in\{2,4,8,16\}$
(\EfourConfigs{} configurations), Adam at learning rate $10^{-2}$ for \EfourSteps{} steps,
\EfourRuns{} runs in total across four variants: \emph{free} ($G$ and $\Phi$ independent),
\emph{tied} ($G=\Phi^{\top}$), \emph{smooth-max}, and the \emph{uncalibrated} ablation in
which the diagonal constraint is removed.

\paragraph{E5 --- trained toy model.} The compressed-computation task
of~\cite{braun2025apd}: $d=\EfiveD$, $F=\EfiveF$, inputs with each coordinate active with
probability $p\in\{0.01,0.02\}$ and value $\sim U[-1,1]$, target $\mathrm{ReLU}(x)$,
$\hat y=W_{\mathrm{out}}\mathrm{ReLU}(W_{\mathrm{in}}x)$ trained by Adam for \EfiveSteps{} steps
under an $L^2$ and an $L^4$ loss, \EfiveSeeds{} seeds each, plus an i.i.d.\ random-code control
at the same $(d,F)$. Diagnosis uses \EfiveEvalTrials{} evaluation trials per sparsity and three
linear readouts, including a least-squares readout fitted on the evaluation distribution ---
deliberately the strongest linear baseline we can give the competing criterion.

\paragraph{E6 --- scaling.} $d\in\{128,256,512,1024\}$ at $F=d^2$, up to
$F=\EsixFmax$ features. The code is never materialised: each trial regenerates it in blocks
from a counted seed sequence and evaluates every sparsity in one pass using nested supports.

\paragraph{Statistics.} Recovery probabilities carry Wilson $95\%$ intervals. Group
comparisons in E5 use a two-sided exact Mann--Whitney $U$ test with Cliff's $\delta$ as effect
size. The scaling law in E6 is a one-parameter least-squares fit of $s_{95}=c\,d/\ln d$, the
shape being predicted by the theory and only the constant free; with \EsixFitPoints{} points
the fit is reported as a consistency check, not as an estimate with inferential force.

\paragraph{Reproducibility.} Every figure and table in this paper is generated from the raw
result files by a script, and every measured number in the text is a macro emitted from those
same files; a checker regenerates the macros, compares them against the file on disk, and fails
the build on any mismatch, on a macro the text uses but does not define, and on a macro defined
but never used. We are explicit about the limit of that guarantee: it certifies that the
\emph{values} match the saved results, not that the sentences around them draw the right
conclusion. Prose claims still need a reader. Commands are in Appendix~\ref{app:repro}.

\paragraph{A note on E6's linear side.} To keep the cost of the largest widths manageable, the
exact linear-energy statistic is computed on every fifth trial rather than every trial, so its
averages rest on fewer samples than the recovery curves (\EsixEnergyTrialsTop{} trials against
\EsixOneZeroTwoFourTrials{} at $d=\EsixDmax$). The recovery probabilities use the full budget.


\section{Earlier campaigns: E1--E6}
\label{app:earlier}

The six campaigns below precede the ones in Section~\ref{sec:campaigns} and are reported here rather
than in the body for two reasons. They establish that the diagnostic measures what the theory says it
measures, which is a precondition for the body's results rather than a result itself; and their
decoder comparisons use a fixed threshold on one side, an asymmetry the body's campaigns were
designed to remove (Appendix~\ref{app:withdrawn}, item 1). Their geometry and frontier numbers are
unaffected by that asymmetry, because neither involves a selected threshold.

One provenance wart applies to E1--E6 and we would rather name it than have it found: they were run
before the run-record policy that stamps every result with the commit and the working-tree status of
the code that produced it, so their records cite the repository's initial commit. Every campaign from
E7 onward carries exact provenance. The saved artefacts of E1--E6 are released unchanged.

\subsection{E1--E3: the diagnostic measures what the theory describes}
\label{sec:res-e1}

Across \EoneConfigs{} configurations, \EoneTrials{} trials each and three readouts ---
\EoneMeasurements{} calibrated interfaces in total --- the floor was violated
\EoneViolations{} times. This is a correctness check on the implementation rather than
evidence for Theorem~\ref{thm:floor}, and we report it as such.

The attainment ratio of the tied readout falls from \EoneTiedRatioMax{} at the lowest feature
load to \EoneTiedRatioMin{} at the highest (Table~\ref{tab:e1}, Fig.~\ref{fig:overview}a): a
random code is close to Welch-optimal in the mean-square sense once $F/d$ is large, since
$\E\langle\phi_i,\phi_j\rangle^2=1/d$ while $W(F,d)\to1/d$ from below. At $d=\EoneTopD$,
$F=\EoneTopF$ the tied ratio is \EoneTopTiedRatio{} and the pseudoinverse ratio is
\EoneTopPinvRatio, the latter being the closer of the two at high load. In all
\EoneTightFrameConfigs{} configurations the harmonic tight frame attains the floor to within
\EoneTightFrameDev{} of exact equality, confirming Proposition~\ref{prop:tight} at machine
precision. The Gaussian control, by contrast, reaches ratios as large as
\EoneGaussianRatioMax: the bound holds for it too, but nothing forces an arbitrary readout to
be anywhere near the floor, and the diagnostic's attainment ratio is therefore informative
across seven orders of magnitude rather than a formality.

Fig.~\ref{fig:overview}b shows threshold recovery at $F=d^2$. The transition moves right with
$d$, from $s_{95}=\EtwoSixFourSNF$ at $d=64$ to $s_{95}=\EtwoTwoFiveSixSNF$ at $d=256$. At
$d=256$ and $s=s_{95}$, thresholding is exact in at least $95\%$ of trials while the calibrated
linear interface retains a per-coordinate RMS error of \EtwoTwoFiveSixLinRms, which is
\EtwoTwoFiveSixLinRmsOverRoot{} times the $d^{-1/2}$ scale and above its own theoretical floor
of \EtwoTwoFiveSixLinRmsFloor. That pair of numbers is the criterion-separation witness of
Problem~\ref{prob:main}, measured.

The threshold is a hyperparameter, and $\theta=1/2$ is not a privileged choice. At
$d=\EtwoThetaDmax$ the recovery threshold moves from $s_{95}=\EtwoThetaZeroPFourSNF$ at
$\theta=0.4$ to $s_{95}=\EtwoThetaZeroPFiveSNF$ at $\theta=0.5$ and
$s_{95}=\EtwoThetaZeroPSixSNF$ at $\theta=0.6$. Raising it helps here because, with unit-norm codes and Boolean
states, active scores concentrate near one while inactive scores concentrate near zero, so a
higher cut trades false positives for false negatives favourably in this regime; the individual
interference terms are of course signed. Two limits on how far this ablation goes: the three thresholds re-decode
the \emph{same} trials rather than independent replications, so they show the sensitivity of the
decision rule and not of the sampling; and at the smallest width the effect is not merely
numerical --- at $d=64$, $\theta=0.4$ no sparsity reaches the $95\%$ level at all, so no
certificate exists there.

E3 confirms the energy scale (Fig.~\ref{fig:overview}c). Over \EthreeConfigs{} configurations
the ratio $E/(s/d)$ lies in $[\EthreeRatioMin,\EthreeRatioMax]$ --- the $s/d$ reference scale is
essentially exact for random codes --- with \EthreeViolations{} violations of either energy
floor, and the closed forms agree with a \EthreeMcTrials-draw Monte Carlo estimate to within a
relative \EthreeMcMaxRelDiff.

Measured against the \emph{proved} uniform-support floor of
Proposition~\ref{prop:uniform-energy}, however, the ratio lies in
$[\EthreeUniformOverFloorMin,\EthreeUniformOverFloorMax]$: the bound is correct but conservative
by a factor between one and two in this regime. That is expected and worth stating plainly,
since it bounds what the diagnostic can certify. Two steps of the proof give the slack: the
non-negative $q\norm{A\mathbf 1}^2$ term is discarded, and $\norm{A}_F^2$ is replaced by its
worst case $F(F-d)/d$ rather than the value realised by the code at hand. The looser the
feature load, the larger the second effect --- here $F=8d$, so $(F-d)/F=7/8$ alone accounts for
part of the gap. Algorithm~\ref{alg:two} therefore reports the \emph{exact} energy computed
from the code's own frame operator, and uses the floor only as the guarantee that the quantity
can never be zero.

\begin{figure}[t]
\centering
\includegraphics[width=\linewidth]{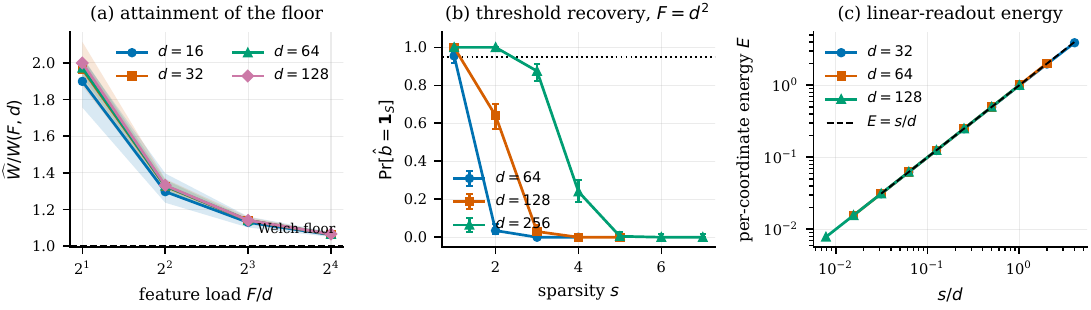}
\caption{\textbf{The Interface Diagnostic on random codes.} (a) Attainment ratio
$\widehat W/W(F,d)$ of the calibrated tied readout versus feature load; shaded bands are the
min--max range over trials, the dashed line is the floor of Theorem~\ref{thm:floor}.
(b) Exact threshold recovery at $F=d^2$ with Wilson $95\%$ intervals; the dotted line marks the
$0.95$ level defining $s_{95}$. (c) Average per-coordinate linear energy under uniformly random
supports against the $s/d$ reference. Generated by \code{scripts/make\_figures.py} from
\code{results/e1}, \code{results/e2} and \code{results/e3}.}
\label{fig:overview}
\end{figure}

\begin{table}[t]
\centering
\caption{E1: attainment ratio $\widehat W/W(F,d)$ of the calibrated tied readout for random
unit-norm codes, averaged over \EoneTrials{} trials. Values above one are required by
Theorem~\ref{thm:floor}; the approach to one with growing feature load is the content of the
measurement.}
\label{tab:e1}
\small
\begin{tabular}{rrrrr}
\toprule
& \multicolumn{4}{c}{$F/d$} \\
\cmidrule(lr){2-5}
$d$ & 2 & 4 & 8 & 16 \\
\midrule
16 & 1.899 & 1.298 & 1.130 & 1.063 \\
32 & 1.966 & 1.327 & 1.141 & 1.065 \\
64 & 1.972 & 1.329 & 1.141 & 1.066 \\
128 & 1.999 & 1.332 & 1.141 & 1.066 \\
\bottomrule
\end{tabular}

\end{table}

\subsection{E4: learned codes attain the floor}
\label{sec:res-e4}

Gradient descent under the unit-diagonal constraint drives the interface to the floor
(Table~\ref{tab:e4}, Fig.~\ref{fig:e4}). Over
\EfourFreeRuns{} runs of the free variant, spanning all \EfourConfigs{} configurations, the
attainment ratio reaches a mean of \EfourFreeRatioMean{} and never exceeds
\EfourFreeRatioMax, with \EfourFreeBelowFloor{} runs below the floor. The tied variant is
sharper still: over \EfourTiedRuns{} runs, mean ratio \EfourTiedRatioMean{} with a maximum of
\EfourTiedRatioMax{} (\EfourTiedBelowFloor{} below the floor), and a relative tight-frame
residual of \EfourTiedFrameResidualMean{} on average (worst case
\EfourTiedFrameResidualMax). In other words, the optimiser recovers the equality case of
Proposition~\ref{prop:tight}: a unit-norm tight frame emerges from gradient descent rather than
being constructed.

The smooth-max variant is the one that does \emph{not} reach its target, and the reason is
structural rather than a tuning failure. Over \EfourSoftmaxRuns{} runs it attains a mean
mean-square ratio of \EfourSoftmaxRatioMean{} (worst case \EfourSoftmaxRatioMax,
\EfourSoftmaxBelowFloor{} below the floor) and drives the maximum off-diagonal entry only to
\EfourSoftmaxRatioMaxMean{} times $\sqrt{W(F,d)}$. Equality in the maximum form requires an
equiangular tight frame, and those exist only for special $(d,F)$ --- none of our
\EfourConfigs{} configurations is guaranteed to admit one. The maximum form of the bound is
therefore approached but not attained here, which is the expected outcome and not evidence
against the bound.

The uncalibrated ablation is the most informative of the four, and its outcome is more extreme
than we expected. Removing the diagonal constraint does not lead the optimiser to a code with
low interference: it leads it to the trivial solution. Across all \EfourUncalRuns{} runs the
\emph{raw} off-diagonal energy collapses to \EfourUncalRawRatioMean{} --- the readout is driven
to zero --- and with it the diagonal gains, whose smallest absolute value reaches
\EfourUncalGainMin{} in single precision. In \EfourUncalDegenerate{} of the
\EfourUncalRuns{} runs the calibration is therefore undefined and Algorithm~\ref{alg:one}
returns \textsc{Degenerate}; in the remaining runs the calibrated interface still respects the
floor, and in no run did it fall below (\EfourUncalBelowFloor{} violations). An unconstrained
optimiser reports spectacular apparent cross-talk while destroying the interface it was
supposed to build. This is Remark~\ref{rem:calib} as an experiment, and it is the reason
step~4 of Algorithm~\ref{alg:one} exists and reports a status rather than a number.

\begin{figure}[t]
\centering
\includegraphics[width=\linewidth]{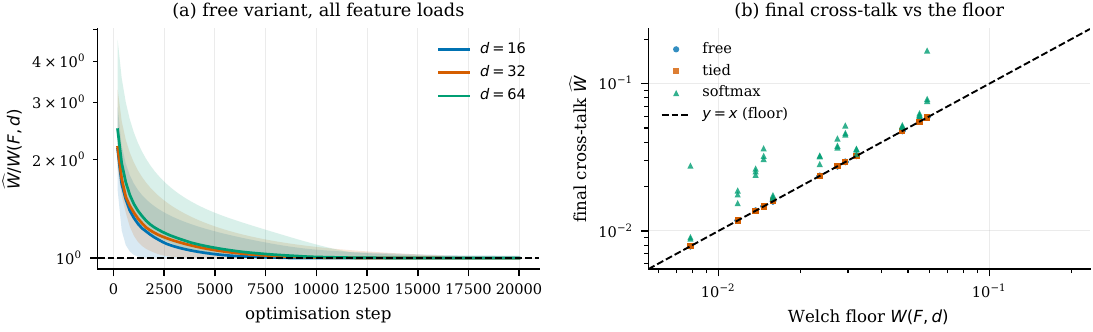}
\caption{\textbf{Optimised codes attain the Welch floor.} (a) Attainment ratio during
optimisation of the free variant; bands are min--max over configurations and seeds.
(b) Final cross-talk against the floor across every configuration and variant; the dashed line
is equality. Generated by \code{scripts/make\_figures.py} from \code{results/e4}.}
\label{fig:e4}
\end{figure}

\begin{table}[t]
\centering
\caption{E4: attainment of the floor by gradient-optimised codes, \EfourRuns{} runs over
\EfourConfigs{} configurations. ``Below floor'' counts runs whose calibrated interface fell
under the theoretical bound; it must be zero. The uncalibrated ablation has no meaningful
calibrated ratio by construction.}
\label{tab:e4}
\small
\begin{tabular}{lrrrrr}
\toprule
variant & runs & mean $\widehat{W}/W$ & max $\widehat{W}/W$ & below floor & extra \\
\midrule
free ($G$, $\Phi$ independent) & 60 & 1.0006 & 1.0176 & 0 & -- \\
tied ($G=\Phi^{\top}$) & 36 & 1.0000 & 1.0000 & 0 & $\rho_{\max}=0.0016$ \\
smooth-max & 36 & 1.5050 & 3.5230 & 0 & $\max$ ratio $=3.12$ \\
uncalibrated (ablation) & 36 & n/a & n/a & 0 & raw ratio $\geq1.5e-16$, $\min|M_{ii}|=0.0e+00$ \\
\bottomrule
\end{tabular}

\end{table}

\subsection{E5: the separation inside a trained network}
\label{sec:res-e5}

Table~\ref{tab:e5} and Fig.~\ref{fig:e5} report the diagnostic applied to the
compressed-computation models. Three findings stand out; one of them we did not anticipate, and
one prediction we had made was not borne out.

\paragraph{The $L^4$-trained network attains the floor.} Its calibrated linear interface has
attainment ratio \EfiveLfourRatioPinvMean{} under the pseudoinverse readout (range
$[\EfiveLfourRatioPinvMin,\EfiveLfourRatioPinvMax]$ over \EfiveSeeds{} seeds) and
\EfiveLfourRatioLsMean{} under the fitted least-squares readout (range
$[\EfiveLfourRatioLsMin,\EfiveLfourRatioLsMax]$), against a floor of $W=\EfiveWelchFloor$. The
$L^2$ baseline sits at \EfiveLtwoRatioLsMean{} and the random-code control at
\EfiveRandRatioLsMean. E4 had shown the floor to be reachable by direct optimisation of the
cross-talk; here it is approached by a network optimising a task loss that never mentions
interference.

Two caveats on the statistics, both cutting against reading too much into them. First, the
comparison with the random code is readout-dependent: $L^4$ is closer to the floor than an
i.i.d.\ code of the same shape under the pseudoinverse (\EfiveLfourRatioPinvMean{} versus
\EfiveRandRatioPinvMean) and under least squares, but its own trained decoder is
\emph{farther} from the floor (\EfiveLfourRatioWoutMean) than either. That last figure has no
counterpart in the control and we do not pair it with one: the random code has no trained
decoder, so its $W_{\mathrm{out}}$ column is by construction the pseudoinverse and would be
compared against a different object. Across the readouts that do admit a matched comparison,
the trained code is comparable to a random one on this axis, not better. Second, every pairwise test returns
$p=\EfiveMwPLfourLtwo$ with $|\delta|=1$, which is exactly the smallest two-sided $p$ an exact
Mann--Whitney can produce at $n=\EfiveSeeds$ per group: within-group variance is negligible, so
the test reports ``no overlap'' and nothing more. The effect sizes are what separate the cases,
and they differ by orders of magnitude --- \EfiveLtwoRatioLsMean{} versus
\EfiveLfourRatioLsMean{} for $L^2$ against $L^4$, a few percent for $L^4$ against random. No
multiplicity correction is applied across the nine tests, and at this $n$ none would survive
one.

\paragraph{The separation is real but much narrower than we predicted, and the $s_{95}$
statistic runs against us.} Our pre-registered expectation was that the network's thresholded
output would remain exact over a strictly wider sparsity range than its best linear readout.
The measurement contradicts that at the sparsities where both work. Ranking the three linear
readouts by exact-recovery rate at each $s$ and taking the best of them, the network is behind
up to $s=4$: at $s=3$ both are exact, and at $s=4$ the network recovers in
\EfiveLfourPmodelFour{} of trials while the calibrated $W_{\mathrm{out}}$ interface recovers in
\EfiveLfourPlinFour. On the $s_{95}$ statistic the linear readout is strictly ahead --
$s_{95}=\EfiveLfourSlinBestMax$ against the network's \EfiveLfourSmodelMax{} on all
\EfiveSeeds{} seeds at $p=\EfivePrimaryP$, and ahead on
\EfiveLfourSlinAheadSecondary{} of \EfiveSeeds{} seeds (tied on the remainder) at
$p=\EfiveSecondaryP$ (Table~\ref{tab:e5}). The preregistered comparison is not supported.

The network only leads further out, where both decoders are already unreliable: at $s=5$,
\EfiveLfourPmodelFive{} versus \EfiveLfourPlinFive, and at $s=6$, \EfiveLfourPmodelSix{} versus
\EfiveLfourPlinSix. E5 therefore shows the network extending the usable sparsity range past the point where the
thresholded linear readouts of its own code have collapsed, not dominating them everywhere.

The floor statement is on firmer ground, but it is worth being exact about what it is. At every
sparsity analog reconstruction carries irreducible error: at $s=3$, where the network is exact in
\EfiveLfourPmodelThree{} of trials, the best linear readout still has per-coordinate RMS error
\EfiveLfourRmsThree{} against its floor of \EfiveLfourFloorThree. That non-zero error is
Theorem~\ref{thm:floor} restated for this code, not an independent measurement --- any
rank-$d$ unit-diagonal interface has it by construction. What E5 measures is its \emph{size} for
a code that gradient descent actually produced, and how far the thresholded decoder gets before
that error stops it.

\paragraph{The $L^2$ model is maintaining a linear interface.} This is the finding we did not
plan for. In the $L^2$ models the ReLU \emph{clips} only \EfiveLtwoNegPre\% of preactivations ---
the macro is the negative fraction, so the unit is very nearly linear --- the learned decoder sits at
relative distance \EfiveLtwoWoutPinvDist{} from the pseudoinverse of the encoder, and the network's
thresholded decision agrees with that of a plain linear readout on \EfiveLtwoAgreeLinear\% of
evaluation states, which the two preceding facts are what make possible. The $L^4$ models are
elsewhere: the ReLU clips \EfiveLfourNegPre\% of preactivations, decoder at relative distance
\EfiveLfourWoutPinvDist,
and agreement with the linear readout down to \EfiveLfourAgreeLinear\%. The diagnostic thus
reproduces, from interface measurements alone and without any mechanistic reverse-engineering,
the distinction that~\cite{bhagat2026cc} and~\cite{ferreiradasilva2026l4} reached by other
means. It also explains why the $L^2$ model is bound by Theorem~\ref{thm:floor} at all: a
network that keeps an effectively linear readout \emph{is} a linear readout for these inputs,
and the best affine readout of that code sits at \EfiveLtwoRatioLsMean{} times the floor. That
number is the fitted least-squares probe's, not the network's own: by the factorisation above the
network's own readout is $\Rgeom\times\Rreadout$, which is a third smaller. Every ratio quoted in
this subsection names its readout, because the two differ by more than the effect under discussion.
Its thresholded output recovers even a single active feature only \EfiveLtwoPmodelOne{} of the
time.

\paragraph{The finding replicates at a second training sparsity.} Everything above is measured
at $p=\EfivePrimaryP$. Repeating the whole campaign at $p=\EfiveSecondaryP$ (ablation A5,
\EfiveSeeds{} fresh seeds per condition) reproduces both halves: the $L^4$ models sit at
\EfiveLfourSecRatioLsMean{} against the floor and agree with a linear readout on only
\EfiveLfourSecAgreeLinear\% of states, while the $L^2$ models sit at
\EfiveLtwoSecRatioLsMean{} and agree \EfiveLtwoSecAgreeLinear\% of the time. The interface
distinction is therefore not an artefact of one input sparsity.

\paragraph{What does not favour our hypothesis.} Three things. The $L^4$ models have a
\emph{worse} mean squared error on the training task than the $L^2$ models (\EfiveLfourMse{}
versus \EfiveLtwoMse), as expected since they optimise a different objective: the interface
advantage is not a task-performance advantage, and we do not claim it is. The $s_{95}$
prediction above failed. And the random-code control is a reminder that a good interface is not
by itself evidence of computation: its \emph{linear} readout recovers exactly with probability
\EfiveRandPlinThree{} at $s=3$, better than the $L^2$ network, simply because an i.i.d.\ code is
already near-Welch-optimal (ratio \EfiveRandRatioLsMean). What the random control lacks is the
nonlinear stage; the ``network'' column for that row is not meaningful and is reported only for
completeness.

\begin{figure}[t]
\centering
\includegraphics[width=\linewidth]{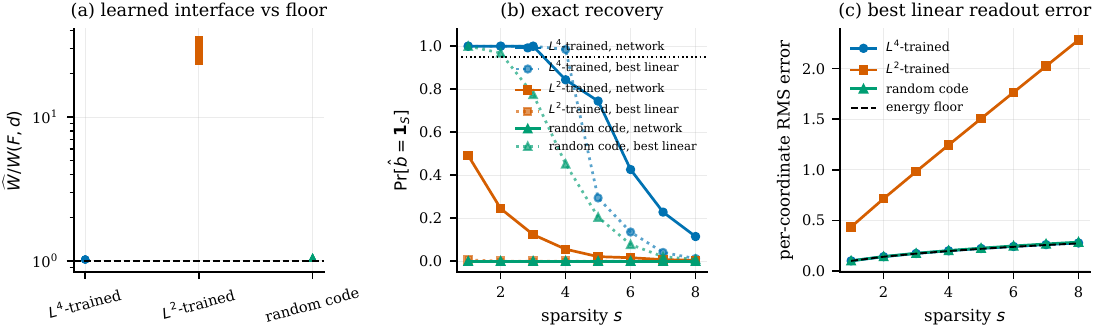}
\caption{\textbf{Interface separation inside a trained toy model of computation in
superposition.} A width-\EfiveD{} network trained to compute \EfiveF{} sparse ReLU features
(setup of~\cite{braun2025apd}; $L^4$ variant of~\cite{ferreiradasilva2026l4}); one point per
seed. (a) Attainment ratio of the learned calibrated interface against the floor of
Theorem~\ref{thm:floor}. (b) Exact recovery by the network's thresholded output (solid) versus
the best linear readout of the same code (dotted). (c) Per-coordinate RMS error of the best
linear readout against the energy floor of Proposition~\ref{prop:uniform-energy}. Generated by
\code{scripts/make\_figures.py} from \code{results/e5}.}
\label{fig:e5}
\end{figure}

\begin{table}[t]
\centering
\caption{E5: the diagnostic applied to trained models, averaged over \EfiveSeeds{} seeds at
training sparsity $p=\EfivePrimaryP$. ``Agrees with linear'' is the fraction of evaluation
states on which the network's thresholded decision coincides with that of a plain linear
readout --- the quantity that identifies which interface the model maintains.}
\label{tab:e5}
\small
\begin{tabular}{lrrrrr}
\toprule
model & $\widehat{W}/W$ & $s_{95}$ (network) & $s_{95}$ (best linear) & ReLU clipped (\%) & agrees with linear (\%) \\
\midrule
$L^4$-trained & 1.022 & 3.0 & 4.0 & 48.6 & 21.8 \\
$L^2$-trained & 28.814 & 0.0 & 0.0 & 4.6 & 99.1 \\
random code & 1.055 & 0.0 & 2.0 & 50.0 & 13.8 \\
\bottomrule
\end{tabular}

\end{table}

\subsubsection*{What the attainment ratio was measuring}
\label{sec:res-g1}

The ratios above are referenced to $W(F,d)$, and Section~\ref{sec:codefloor} showed that such a
ratio conflates two questions. Applying the factorisation~\eqref{eq:decomposition} to the same
models settles which of the two the reported numbers answered, and the answer is uncomfortable:
one of them, entirely.

Under the calibrated pseudoinverse the readout factor is $1$ to six decimals in every cell ---
\GoneLfourRreadoutPinv{} for $L^4$, \GoneLtwoRreadoutPinv{} for $L^2$,
\GoneRandRreadoutPinv{} for the random control --- which is not a measurement but
Theorem~\ref{thm:codefloor} restated, since the calibrated pseudoinverse \emph{is} the
code-specific optimum. The whole ratio is therefore $R_{\mathrm{geom}}$:
\GoneLfourRgeom{} against a reported \GoneLfourRatio{} for $L^4$, \GoneLtwoRgeom{} against
\GoneLtwoRatio{} for $L^2$, \GoneRandRgeom{} against \GoneRandRatio{} for the random code. So
the claim that the $L^4$ model comes ``within a few parts in a thousand of the floor'' is a
statement about the evenness of its leverage profile, and contains no information about its
decoder. We restate it that way.

Asking the decoder question separately reverses the ordering. Measured against the best readout
of its \emph{own} code, the $L^4$ network's trained output layer sits at
$R_{\mathrm{readout}}=\GoneLfourRreadoutWout{}$ while the $L^2$ network's sits at
\GoneLtwoRreadoutWout{} --- so on this criterion the $L^2$ model has the better decoder and the
worse code, the opposite of the reading the single ratio invites. The random control's own
readout gives exactly \GoneRandRreadoutWout{}, which is a consistency check rather than a
finding: that baseline sets $W_{\mathrm{out}}=\Phi^{+}$ by construction, so its ``own decoder''
is the code-specific optimum by definition, and any other value would indicate a bug.

The mechanism is visible in the leverage profiles, and it is the quantity
Corollary~\ref{cor:failthresh} needs. The $L^4$ code spreads leverage almost uniformly, from
\GoneLfourLeverageMin{} to \GoneLfourLeverageMax{} with coefficient of variation
\GoneLfourLeverageCv{}; the random code is comparable at \GoneRandLeverageMin{} to
\GoneRandLeverageMax{} (coefficient of variation \GoneRandLeverageCv{}); the $L^2$ code is not,
running from \GoneLtwoLeverageMin{} to \GoneLtwoLeverageMax{} with coefficient of variation
\GoneLtwoLeverageCv{}. That dispersion is exactly the excess of Proposition~\ref{prop:excess},
and it grows with the training density: at $p=0.02$ the $L^2$ geometry factor rises to
\GoneLtwoRgeomDenseP{}.

Two limits on what this licenses. The factorisation is arithmetic on
Theorem~\ref{thm:codefloor} and needs no new assumption, but it reinterprets an existing
measurement rather than adding evidence --- five seeds at one width remain five seeds at one
width. And \GoneLtwoLeverageMin{} is the number Corollary~\ref{cor:failthresh} consumes: with
$s=2$ the threshold $\min\{\tfrac12,(s-1)^2/((F-1)+(s-1)^2)\}$ is exactly $1/F$, so at
$F=\EfiveF$ it is $0.01$, and the measured minimum leverage of the $L^2$ code falls below it. The
theory therefore predicts that this code loses affine separability at $s=2$ --- from a single
leverage score, with no decoder and no probe in the argument. The corollary is applied at the
minimum leverage, which is where it is strongest and, being below a half, also where its
hypothesis holds. Whether that prediction survives at other widths is the
question Section~\ref{sec:limitations} flags as open.

\subsection{E6: a scalability stress test, not a scaling law}
\label{sec:res-e6}

This campaign was designed to measure how the recovery threshold scales, and it does not have
the resolution to do that. We report it for what it does establish --- that the measurement
itself runs at a width and feature count where the code cannot be materialised --- and withdraw
the scaling-law reading.

Scaling to $d=\EsixDmax$ at $F=\EsixFmax$ features lets the recovery threshold be traced over
a factor of eight in width (Table~\ref{tab:e6}, Fig.~\ref{fig:e6}). On the grid, $s_{95}$ grows
from \EsixOneTwoEightSNF{} at $d=128$ to \EsixOneZeroTwoFourSNF{} at $d=\EsixDmax$.

That grid statistic needs care, because it floors the crossing and floors it unevenly. The
true 95\% crossing lies between grid points, and the gap is proportionally largest where the
transition is sharp relative to the spacing --- at $d=128$ the interpolated crossing is
\EsixInterpOneTwoEight{} against a reported 1. Interpolating instead of flooring changes the
headline: growth over the sweep falls from \EsixGrowthGrid{} to \EsixGrowthInterp, against
\EsixGrowthPredicted{} for an exact $d/\ln d$ law. The floored figure therefore overstates the
growth by roughly a third, purely as an artefact of the grid, and we quote the interpolated
values from here on.

Four widths then cannot identify a law. On the interpolated crossings a one-parameter fit of
$s_{95}=c\,d/\ln d$ gives $c=\EsixInterpC$ with $R^2=\EsixInterpRsqLog$ over
\EsixFitPoints{} points, but a plain linear law $c\,d$ reaches $R^2=\EsixInterpRsqLin$ and
$c\,d/\ln^2 d$ reaches $R^2=\EsixInterpRsqLogSq$ on the same four points. The three are
indistinguishable on this evidence, the base of the logarithm is not identifiable at all, and we
therefore claim no scaling law --- neither for $d/\ln d$ nor against it. A previous version of
this work drew a $d/(16\ln d)$ reference curve here; it has been withdrawn, because its constant
is inherited from the earlier literature rather than derived anywhere in this work, and because
it sat \emph{above} the measured thresholds at every width and so was not the conservative
reference its form suggests. Corollary~\ref{cor:separation} fixes only the shape
$s\le c\,d/\log d$, for an unspecified universal $c$.

A second qualification. $s_{95}$ is a point estimate: at $d=\EsixDmax$ the campaign runs
\EsixOneZeroTwoFourTrials{} trials, so even a flawless \EsixOneZeroTwoFourTrials{}-for-%
\EsixOneZeroTwoFourTrials{} result gives a Wilson lower bound of only
\EsixTopBestWilsonLow. No sparsity at that width can be certified at the 95\% level with 95\%
confidence on this budget; the reported thresholds are best estimates, and a claim at
confidence would need more trials.

The whole campaign took \DurEsixMin{} minutes on CPU, with $d=\EsixDmax$ accounting for
\EsixOneZeroTwoFourDurationMin{} of them over \EsixOneZeroTwoFourTrials{} trials. Materialising
the code at that size would have required $4$~GB; the streaming implementation of
Algorithm~\ref{alg:ensemble} keeps the peak footprint proportional to $Bd$.

\begin{figure}[t]
\centering
\includegraphics[width=\linewidth]{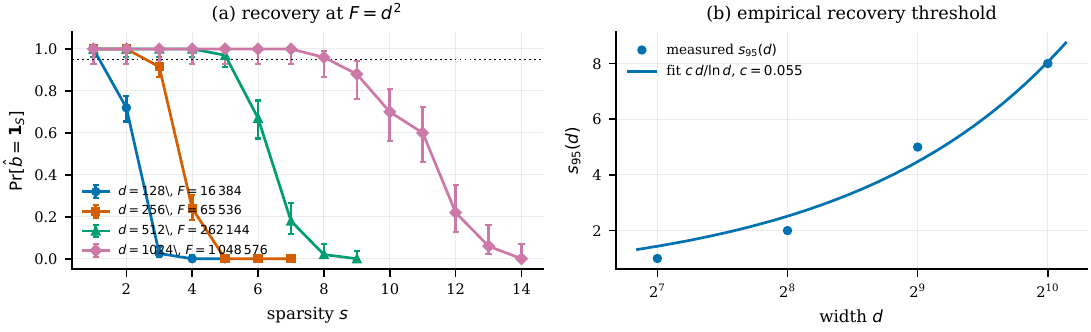}
\caption{\textbf{Threshold recovery at quadratic feature load, up to $d=\EsixDmax$}
($F=\EsixFmax$). (a) Empirical exact-recovery probability with Wilson $95\%$ intervals; a fresh
random code is drawn per trial. (b) The empirical threshold $s_{95}(d)$ with the fitted
$c\,d/\ln d$ curve shown as a fit only: four widths do not distinguish it from $c\,d$ or from
$c\,d/\ln^2d$, and no scaling law is claimed. Generated by
\code{scripts/make\_figures.py} from \code{results/e6}.}
\label{fig:e6}
\end{figure}

\begin{table}[t]
\centering
\caption{E6: recovery threshold and measured cost by width.}
\label{tab:e6}
\small
\begin{tabular}{rrrrrr}
\toprule
$d$ & $F=d^2$ & trials & $s_{95}$ & $d/(16\ln d)$ & wall clock (min) \\
\midrule
128 & 16\,384 & 200 & 1 & 1.65 & 1.5 \\
256 & 65\,536 & 200 & 2 & 2.89 & 2.2 \\
512 & 262\,144 & 100 & 5 & 5.13 & 4.9 \\
1024 & 1\,048\,576 & 50 & 8 & 9.23 & 21.3 \\
\bottomrule
\end{tabular}

\end{table}

\subsection{Cost and traceability}
\label{app:cost}

Every campaign ran on CPU alone; Table~\ref{tab:cost} reports the measured wall clock, which
totals \DurTotalMin{} minutes across the six. Table~\ref{tab:trace} closes the loop between the
claims above and the artefacts behind them: for each experimental claim it names the figure or
table that shows it, the raw result file the numbers come from, and the script and
configuration that produce that file. Together with the automatic check that every measured
number in this manuscript matches the saved results (Appendix~\ref{app:repro}), this is what
makes the reported quantities auditable rather than merely stated.

\begin{table}[t]
\centering
\caption{Measured computational cost, read from the run records. E1--E6 are CPU-only and hide any
accelerator that is present, so they cannot silently become irreproducible on a CPU-only machine.
E7 and E8 train on a GPU; every number this paper draws from them is then recomputed from their
saved weights on a CPU, so no claim here depends on access to an accelerator. \emph{jobs} counts
models for the GPU campaigns and configured jobs for the others. $^{\dagger}$Stage~A was resumed,
and a run record times only the final invocation, so this figure is the last leg rather than the
campaign.}
\label{tab:cost}
\small
\begin{tabular}{llrrlr}
\toprule
campaign & question & jobs & workers & device & wall clock (min) \\
\midrule
E1 & floor on random codes & -- & 1 & CPU & 0.6 \\
E2 & recovery at $F=d^2$ & -- & 1 & CPU & 67.5 \\
E3 & average linear energy & -- & 1 & CPU & 0.2 \\
E4 & optimised codes & 168 & 12 & CPU & 169.6 \\
E5 & trained toy model & 25 & 6 & CPU & 76.1 \\
E6 & scaling to $d=1024$ & -- & 4 & CPU & 29.9 \\
\midrule
E7-A & trained-network grid & 180 & 8 & CUDA & 306.1$^{\dagger}$ \\
E7-B & feature load and sparsity & 100 & 10 & CUDA & 193.5 \\
E7-C & $d=400$ frontier rate & 30 & 10 & CUDA & 103.8 \\
E8 & frozen code vs frozen readout & 60 & 10 & CUDA & 77.0 \\
E8 & the same at $d=200$ & 30 & 8 & CUDA & 123.2 \\
\midrule
\multicolumn{5}{l}{total} & 1147.5 \\
\bottomrule
\end{tabular}

\end{table}

\begin{table}[t]
\centering
\caption{Traceability: each experimental claim, the figure or table that shows it, the raw
result file it comes from, and the script and configuration that produce it. Reproduction
commands are in Appendix~\ref{app:repro}.}
\label{tab:trace}
\small
\begin{tabular}{p{3.5cm}p{2.3cm}p{3.4cm}p{3.6cm}}
\toprule
claim & figure / table & raw result file & script (config) \\
\midrule
The floor is respected by every calibrated interface & Tab.~\ref{tab:e1}, Fig.~\ref{fig:overview}a & \texttt{\small results/e1/raw/e1\_rows.json} & \texttt{\small experiments/e1\_welch\_floor.py} (\texttt{\small configs/e1.json}) \\
\addlinespace[2pt]
Threshold recovery succeeds where linear readout cannot be exact & Fig.~\ref{fig:overview}b & \texttt{\small results/e2/raw/e2\_profiles.json} & \texttt{\small experiments/e2\_threshold\_recovery.py} (\texttt{\small configs/e2.json}) \\
\addlinespace[2pt]
Average linear energy follows the $s/d$ scale & Fig.~\ref{fig:overview}c & \texttt{\small results/e3/raw/e3\_rows.json} & \texttt{\small experiments/e3\_linear\_energy.py} (\texttt{\small configs/e3.json}) \\
\addlinespace[2pt]
Learned codes attain the floor & Tab.~\ref{tab:e4}, Fig.~\ref{fig:e4} & \texttt{\small results/e4/raw/e4\_runs.json} & \texttt{\small experiments/e4\_optimized\_codes.py} (\texttt{\small configs/e4.json}) \\
\addlinespace[2pt]
A trained network attains the floor and separates the two criteria & Tab.~\ref{tab:e5}, Fig.~\ref{fig:e5} & \texttt{\small results/e5/raw/e5\_runs.json} & \texttt{\small experiments/e5\_trained\_toy.py} (\texttt{\small configs/e5.json}) \\
\addlinespace[2pt]
$s_{95}(d)$ grows with the $d/\ln d$ scale & Tab.~\ref{tab:e6}, Fig.~\ref{fig:e6} & \texttt{\small results/e6/raw/e6\_summary.json} & \texttt{\small experiments/e6\_threshold\_scaling.py} (\texttt{\small configs/e6.json}) \\
\addlinespace[2pt]
\bottomrule
\end{tabular}

\end{table}

\FloatBarrier


\section{Readings we withdrew, and what caused each}
\label{app:withdrawn}

This appendix exists because the alternative is worse. Over the course of this work eight readings of
our own measurements turned out not to survive a check, and in every case the reading we had chosen
was the more interesting of the available ones. The last two were found by an external adversarial
review of the submitted manuscript, after the other six had been written up here, and a second round
of that review --- of the fixes for the first two --- found a ninth defect, in the audit tooling
itself. That sequence is the most useful thing this appendix can tell a reader about how much to
trust the rest. Listing them costs us the appearance of a clean
result and buys the reader the only thing that makes the surviving claims worth anything: a record
of how hard they were pushed. The register in \code{docs/known\_defects.md} carries the code-level
detail; what follows is the manuscript-level consequence of each.

\paragraph{1. ``The two estimands disagree'' (withdrawn; cause: a defective threshold search).} The
analysis plan named two co-primary estimands, $s_{95}$ and the recovery AUC. An earlier version of
Section~\ref{sec:res-primary} reported that they disagreed --- $s_{95}$ a tie at every width, the
AUC favouring the network at $d=50$ and the probe at $d\in\{100,200\}$ --- and treated the
disagreement as the finding, on the grounds that a blunt statistic and a fine one seeing different
things is informative.

The cause was in \code{select\_thresholds}. For the \code{global} and \code{per\_feature} policies it
searched a grid of quantiles of the validation scores and returned the best point on that grid,
without ever scoring $\theta_{\mathrm{fixed}}$, the value it was meant to improve on. On these score
distributions the quantile grid could and did miss, so a decoder could be evaluated at a threshold
that loses to the default on the very validation objective the search maximises. It cost the network
up to five sparsity levels at $d=400$. The fix scores $\theta_{\mathrm{fixed}}$ first, adds a uniform
grid to the quantile grid, and asserts the post-condition that the returned threshold is at least as
good as $\theta_{\mathrm{fixed}}$ on validation; \code{tests/test\_threshold\_selection.py}
re-implements the old search and requires it to fail a bimodal case the new one passes. Both sides of
the comparison moved when it was refitted, which is the reason we trust the correction: it was not a
change that could only help one arm.

With the defect fixed the two estimands agree at all four widths and the disagreement was an
artefact. We report the agreement, and this paragraph, rather than quietly substituting one for the
other.

\paragraph{2. ``The ReLU discards affinely decodable information'' (withdrawn; same cause).} A probe
fitted on the preactivation $\Phi b$ once beat the network by a full sparsity level, and we read that
as a measurement of what the nonlinearity throws away --- a claim we liked, because it reframed a
decoder comparison as a statement about representations. It does not survive the same fix. The
pre-ReLU probe's advantage is now \PrimLfourFiftyPreDiff{} at $d=50$
([\PrimLfourFiftyPreCiLow, \PrimLfourFiftyPreCiHigh], \PrimLfourFiftyPreTies{} of \EsevenSeeds{}
seeds tied), \PrimLfourHundredPreDiff{} at $d=100$ and \PrimLfourTwoHundredPreDiff{} at $d=200$, and
at $d=400$ it reverses: the network leads the \emph{pre}-ReLU probe by
\PrimLfourFourHundredPreDiff{} on \PrimLfourFourHundredPreNetWins{} of \ScSeeds{} models. A quarter
of a sparsity level, with most seeds tied and the sign flipping at the largest width, does not
support a claim about what a nonlinearity discards.

What does survive is confined to the untrained arm, where the effect is an order of magnitude larger
and unanimous: the pre-ReLU probe beats the network by \PrimRandFiftyPreDiff{},
\PrimRandHundredPreDiff{}, \PrimRandTwoHundredPreDiff{} and \PrimRandFourHundredPreDiff{} levels at
the four widths, on every one of the \PrimRandFourHundredPreProbeWins{} and \EsevenSeeds{} models per
cell. On an untrained code the thresholded post-ReLU readout is far worse than an affine read of the
preactivation, and training closes almost all of that gap. We state it in that scoped form and note
that it is a comparison across two stages of one pipeline, not a decoder comparison
(Section~\ref{sec:probes}).

\paragraph{3. ``The frozen-code arm is a failed optimisation'' (withdrawn; cause: scoring an arm
under an objective it was not trained on).} In E8 the arm with a frozen random code trains only
$W_{\mathrm{out}}$ and ends at $\Rreadout\approx\EeightFrozenTwoHundredRreadout$, nearly twice the
cross-talk of the best readout of the code it was given. Since $\Phi^{+}$ attains that optimum for
every code, the obvious reading is that gradient descent failed at a convex-looking problem.

We checked it instead of asserting it. Scored under the loss the arm actually trained on, in its own
post-ReLU representation, on fresh draws from a seed no training used, the frozen $W_{\mathrm{out}}$
\emph{beats} $\Phi^{+}$ by a factor of \EeightFrozenTwoHundredTaskGain{} at $d=200$ on
\EeightFrozenTwoHundredTaskWins{} of \EeightSeeds{} models. The cross-talk is bought, not lost, and
the correct reading is about what co-adaptation does to two objectives rather than about optimisation
difficulty. Section~\ref{sec:res-primary} now says that.

\paragraph{4. ``The advantage of training grows with width'' (withdrawn; cause: choosing the summary
after seeing the result).} Both the trained and the untrained frontier grow with width, so at three
widths it was natural to report that training's advantage grows too --- and as a ratio over
$d\in\{50,100,200\}$ it does. The fourth width breaks it. In absolute sparsity levels, the
operationally meaningful unit, the margin is \MarginFiftyGap, \MarginHundredGap,
\MarginTwoHundredGap, \MarginFourHundredGap: it peaks at $d=200$ and falls, by more than the sampling
noise (a two-sample bootstrap over the independently trained arms gives \MarginDrop{} with interval
[\MarginDropCiLow, \MarginDropCiHigh]). As a ratio it falls at every step.

When we first found the fourth width's numbers we reassured ourselves with the ratio, which was the
reading that preserved the claim. That is the same move as picking a summary statistic after seeing
the result, and we record it as such in \code{docs/decision\_log.md} (D20). The manuscript now leads
with the reading that does not survive and states the surviving claim in its narrowed form: training
buys frontier at every width we measured, and the amount it buys stops growing by $d=400$.

\paragraph{5. ``The subset frontier is a justified shortcut'' (withdrawn; cause: a sufficient
condition used as if it were necessary).} Evaluating $\kappa_i$ for every feature costs $F$ linear
programmes, so early campaigns evaluated a subset chosen by Theorem~\ref{thm:arrow} and reported the
minimum over it as an upper bound. The bound is accurate --- it overestimates by
\SubsetFourHundredOver{} on average even at $d=400$ --- but its justification is not: the theorem is
sufficient for a feature to be hard, not necessary, so the subset need not contain the true argmin.
On untrained codes it always does. On trained codes it drifts, and it drifts with width: the median
leverage rank of the true argmin runs \SubsetFiftyRank, \SubsetHundredRank, \SubsetTwoHundredRank,
\SubsetFourHundredRank, and at $d=400$ the subset locates the argmin in \SubsetFourHundredFound{} of
\ScSeeds{} trained models.

The bias is one-sided and asymmetric between the arms, so a margin computed from subset values is
biased in favour of training by an amount growing with the variable under study. Every frontier
number in this paper is therefore the exact minimum over all $F$ features. We would not use the
shortcut past $d=200$ again, and we report the accuracy and the invalidity together because the first
without the second is how a shortcut survives.

\paragraph{6. Three theorem statements, false as written (corrected; cause: a proof establishing less
than its statement claimed).} An external adversarial audit found that
Theorem~\ref{thm:arrow} and Corollary~\ref{cor:failthresh} were stated for all $h_i<1$ while their
shared proof establishes them only for $h_i\le 1/2$, with explicit unit-norm counterexamples at
$d=2$, $F=3$; and that a proposition concluded a feature with $\kappa_i=1$ is separable at $s=1$ when
the frontier statement requires $\kappa_i>s$, so the correct answer is that it is separable at no
positive sparsity. The statements are corrected here and the counterexamples are regression tests
(\code{tests/test\_arrow\_counterexamples.py}).

No measurement was affected: the corollary is only ever applied at $\min_i h_i$, and the largest
$\min_i h_i$ over all \KdThreeModels{} stage A models is \KdThreeMaxLevMin, inside the valid
region. That is why the tests
passed --- they sampled only the region where the statement is true, which is a gap in the tests and
not a false test. An earlier preprint carries the incorrect statements and will be corrected with an
explicit erratum rather than a silent replacement.

\paragraph{7. ``Stage B repeats stage A's protocol exactly, and denser training states buy frontier''
(withdrawn; cause: comparing two cells without checking they had the same shape).} Stage B was
described in the submitted manuscript as an independent replication of two stage A cells plus a
clean one-variable sparsity contrast. Both halves were false. Its $p=0.01$ cells are at
$F=\SbLoadHigh d$, twice stage A's load, so they are not repeats of anything; and the difference in
effect size between the two campaigns, which we attributed to the imprecision of a cell mean over ten
seeds, is a design difference. Worse, the ``cleanest small result in the campaign'' --- that
$\kappa_{\min}$ rises by a third with the training sparsity --- compared an $F=4d$ cell against an
$F=2d$ one. Made at matched load the comparison reverses: \SbMatchPtwoKappa{}
[\SbMatchPtwoKappaLo, \SbMatchPtwoKappaHi] at $p=0.02$ against \SbMatchPoneKappa{}
[\SbMatchPoneKappaLo, \SbMatchPoneKappaHi] at $p=0.01$, per-model ranges disjoint, in the direction
opposite to the claim.

The claim is withdrawn and Section~\ref{sec:stageb} now reports the null. What survives, correctly
described, is a genuine load axis: the loss--geometry separation and the sign structure of the
decoder comparison hold at twice the feature load. The structural cause is worth more than the
correction: neither \code{primary\_comparison.json} nor \code{all\_feature\_frontier.json} recorded
$F$, so the number distinguishing the two designs was absent from the files every table is generated
from, and no check could have caught the error. Both now record it, and Table~\ref{tab:stageb} carries
it as a column.

\paragraph{8. ``The ceiling analysis evaluates the same objective the selection maximised''
(corrected; cause: the seventh instance of the trap in items 1 and 5).} Section~\ref{sec:probes}
scores the network's own map under the campaign's probe-selection objective. The campaign maximises a
normalised trapezoid over the validation recovery curve; the script computed the plain mean over grid
points, and its own documentation asserted that the mean \emph{was} the selection objective. On a
uniform integer grid the two differ by endpoint reweighting of the same order as the smallest gap the
analysis reports.

Recomputed under the correct functional the result survives and strengthens --- the gaps grow, the
count is unchanged --- so nothing in Section~\ref{sec:probes} is withdrawn. Two things move. The
$d=50$ cell is no longer indistinguishable from zero, so it is no longer offered as a consistency
check. And the $L^2$ cells, which under the wrong functional showed a systematic positive gap, sit at
essentially zero, which is the better outcome: the gap is positive exactly where the network leads,
absent where neither decoder recovers anything, and strongly negative where the probe leads. We
record it because a corrected measurement that happens to help is exactly the kind a reader should be
told about.

\paragraph{What the pattern says.} Five of the nine defects behind them are the same trap in five different code paths:
an operating point or an objective chosen by maximising something, with no check that it is the thing
it claims to be. Two are a reading chosen from among several that the data permitted, selected after
the fact for being the stronger claim. One --- item 7 --- is a comparison made without checking that
the two things compared had the same shape, which is the same failure at the level of experimental
design rather than of code. We mention the count because a reader deciding how much weight to put on
this paper's surviving claims should know both the rate at which we found our own errors and the
direction they pointed. Only the first pattern is one that more tests would have caught; the last was
caught by recording the variable that distinguished the designs, which is now done.

\section{Reproducibility statement}
\label{app:reproducibility}

Everything below is anonymised for review. An anonymised archive of the repository accompanies this
submission; the public repository, DOI and citation metadata are supplied at camera-ready.

\paragraph{What is released.} The code that produced every number, the configuration of every
campaign, the raw JSON result of every model, the trained weights of all 400 networks in the
controlled campaigns, and the sources of this manuscript. Nothing in the paper is computed by a
script that is not released, and no measured number in the text is typed by hand: every one is a
macro generated from \code{results/} by \code{scripts/make\_numbers.py}, and
\code{scripts/check\_manuscript\_numbers.py} regenerates every generated table and every macro
definition from \code{results/} and exits non-zero on any mismatch. It is run as part of the
release check and its failure modes are tested: injecting a wrong literal into a generated table, or
an undefined macro of a generated family into the source, both make it fail.

\paragraph{Provenance.} Every campaign from Section~\ref{sec:campaigns} onward writes a run record
carrying the exact command line, the seeds, the resolved configuration, the environment, the
wall-clock duration, the exit status, and the commit hash together with the working-tree status of the
code that produced it --- with the list of modified files when the tree was dirty. The earlier
campaigns of Appendix~\ref{app:earlier} predate that policy and cite the repository's initial commit,
which is stated there rather than hidden. Every third-party dictionary analysed in
Section~\ref{sec:dictionaries} is recorded with the SHA-256 of the checkpoint file actually read.

\paragraph{Cost.} The full pipeline is reproducible on commodity hardware; the released cost table
(Appendix~\ref{app:cost}) gives the measured wall-clock and device of every campaign. The dictionary
measurements of Section~\ref{sec:dictionaries} cost seconds per dictionary and need no accelerator,
which is why a reader can check the paper's main empirical claim without rerunning any training.

\paragraph{What a reader should be able to check independently.} Three things, in increasing cost.
The identity of Proposition~\ref{prop:leverage} against the implementation, from the released
\code{results/sae/derived/sae\_axes.json} alone. The dictionary measurements, by downloading the
listed checkpoints and running one script. And the controlled campaigns, either from the released
weights --- which is how four of the analyses in Section~\ref{sec:campaigns} were computed, months
after the runs --- or by retraining from the recorded seeds. The probe-ceiling analysis of
Section~\ref{sec:probes} is the cheapest of the four and the most useful to re-run adversarially: it
recomputes both sides of the paper's own decoder comparison from the released weights and the
campaign's own seeds, so reproducing Table~\ref{tab:primary} to within \CeilReproMaxDelta{} of a
sparsity level is a check on that comparison as much as on the ceiling.

\paragraph{Known defects.} A register of every defect found during this work, what each invalidated
and what was done about it, is released with the code. The manuscript-level consequences are in
Appendix~\ref{app:withdrawn}. Two of the entries invalidated readings that appear in an earlier
preprint of this work, which will carry an explicit erratum rather than a silent replacement.

\section{Open problem}
\label{sec:open}

Corollary~\ref{cor:impossible} shows that no reset staying inside the small-error
$\varepsilon$-linear class can improve the worst-case cross-talk scale beyond $d^{-1/2}$ when
$F\gg d$, while Theorem~\ref{thm:randomsupport} shows thresholded recovery survives to
$s=O(d/\log d)$ at quadratic load. Any interpolation beyond the H\"{a}nni template must
therefore leave the small-error linear class. This makes one question concrete.

\begin{conjecture}[Generic robust threshold reset]
\label{conj:reset}
Fix $s=O(1)$. There are constants $c,C,K>0$ such that for every $d$ and every
$F\le c\,d^2/\log^C d$ there exist a code $\Phi\in\R^{d\times F}$, an input readout
$R_{\mathrm{in}}$, a width-$\Otilde(d)$ reset module $\mathcal R_\Phi:\R^d\to\R^d$ and a
scoring map $S$ such that, for every $s$-sparse $b$ and every $x$ with
$\norm{R_{\mathrm{in}}x-b}_\infty\le K\,d^{1/2}/(F^{1/2}s^{1/4})$, the output
$z=\mathcal R_\Phi(x)$ is composable with the next stage and satisfies
$b_i=1\Rightarrow(Sz)_i\ge3/4$ and $b_i=0\Rightarrow(Sz)_i\le1/4$.
\end{conjecture}

The incoming-error tolerance is the essential clause: a reset theorem for clean inputs
$x=\Phi b$ alone would not support recursive composition. By
Corollary~\ref{cor:impossible} any proof must use an explicitly nonlinear or thresholded
decoding mechanism, as Adler--Shavit do.

\section{Conclusion}
\label{sec:conclusion}

We proposed the Interface Diagnostic, a computable procedure that decides which decoding
interface an overcomplete code or a trained network maintains, and how far it is from the best
possible one. It rests on a rank--trace floor for unit-diagonal linear readouts, in a
gain-normalised form that bounds the interference-to-gain ratio, and on a separation between the
linear and threshold criteria proved under one and the same sparse-state model. Its mean-square
statistics run in $O(Fd^2)$ time and $O(d^2)$ memory and reach $F\approx10^6$ features on a
laptop CPU.

Two additions make the diagnosis about a code rather than about an ensemble. Computing a code's
own floor in closed form splits any attainment ratio into a geometry term and a decoder term, and
the split is deflationary: under the calibrated pseudoinverse the decoder term is one by
construction, so the ratio a reader is invited to read as a decoder result is nothing of the
kind. The trained $L^4$ network's \EfiveLfourRatioPinvMean{} and the $L^2$ baseline's factor
\EfiveLtwoRatioLsMean{} are statements about leverage profiles. And for support recovery, one
linear programme per feature returns the largest sparsity at which any affine probe can detect
that feature, with a certified margin when it succeeds and a checkable certificate when it fails.
Leverage links the two: low leverage forces the affine frontier down, with no decoder in the
argument, and an explicit code refutes the converse, so the hierarchy holds in one direction
only.

Applied to 180 models, the instrument returns one result that resists summary and two that do not.
The first: given the same representation and the same threshold treatment, the comparison between the
network and a fitted affine probe has no single sign. It is positive under $L^4$ and grows with width
to \ScPrimDiff{} levels in \ScPrimNetWins{} of \ScSeeds{} models at $d=400$; it is exactly zero under
$L^2$, where both decoders fail together; and it is negative by up to
\PrimRandFourHundredDiff{} levels on a code that was never trained. The \PrimTies{} ties out of
\PrimModels{} are therefore not evidence of equivalence --- \PrimLtwoJointFailures{} of them are
joint failures --- and the $L^4$ lead is not evidence of nonlinear expressive power either, since the
probe class contains the network's own decoder (Section~\ref{sec:probes}). What the comparison
identifies is the code, not the decoder. The first positive: what separates
models is the loss, not the architecture --- $L^2$ pushes $R_{\mathrm{geom}}$ to
\SALtwoTwoHundredRgeom{} and collapses its frontier to $1$, while $L^4$ holds
\SALfourTwoHundredRgeom{} and a frontier of \FullLfourTwoHundredKappaMin{}. The second: that
separation is bought by training the code and not the readout, since a frozen random code with a
fully trained readout keeps an untrained code's frontier.

Two things we decline to claim. There is no scaling law here: the threshold-recovery campaign reaches
$d=\EsixDmax$ at about a million features, which establishes that the measurement scales, but four
widths do not distinguish $c\,d/\ln d$ from the alternatives. And proximity to the rank--trace floor
is not evidence of learned structure, because an untrained code already achieves it.

The broader lesson is that capacity questions for overcomplete representations are not settled by
counting features, nor by measuring distance to a bound that random codes already meet. They depend
on which decoding interface a particular code maintains, and on which representation is being read ---
both of which are now measurable, per feature and exactly, for the code in hand.

\section{Reproduction}
\label{app:repro}

From a clean checkout:

\begin{verbatim}
python -m venv .venv && . .venv/bin/activate     # Windows: .venv\Scripts\activate
python -m pip install -r requirements.txt
python -m pytest -q tests/                       # 370 tests
bash scripts/run_all.sh                          # Windows: scripts\run_all.ps1
\end{verbatim}

Individual campaigns, each of which writes \code{results/<id>/run\_record.json} with the exact
command, seeds, environment, duration and exit status:

\begin{verbatim}
python experiments/e1_welch_floor.py
python experiments/e2_threshold_recovery.py
python experiments/e3_linear_energy.py
python experiments/e4_optimized_codes.py
python experiments/e5_trained_toy.py
python experiments/e6_threshold_scaling.py
\end{verbatim}

Adding \code{--smoke} to any of them runs a reduced configuration end to end in seconds.
Figures, tables and the generated numbers are rebuilt with
\code{scripts/make\_figures.py}, \code{scripts/make\_tables.py} and
\code{scripts/make\_numbers.py}; \code{scripts/check\_manuscript\_numbers.py} verifies that
every measured number in this manuscript matches the saved results and exits non-zero
otherwise. E6 writes results per width and per chunk and resumes an interrupted campaign
without recomputing completed trials.

\section{Verification of imported statements}
\label{app:imported}

Every statement imported from the two source frameworks was checked against the published
version. Table~\ref{tab:imported} records the audit.

\begin{table}[h]
\centering
\caption{Audit of imported statements.}
\label{tab:imported}
\small
\begin{tabular}{p{3.1cm}p{4.6cm}p{4.6cm}}
\toprule
source & as printed there & as used here \\
\midrule
\cite[Thm.~11]{hanni2024superposition} &
targeted superpositional AND, precision $\varepsilon_{\mathrm{in}}\to\varepsilon_{\mathrm{out}}$
with $\varepsilon_{\mathrm{out}}=\Otilde(\sqrt{s^2/d})$ &
Lemma~\ref{lem:comp}, $\ecomp(d)=\Otilde(d^{-1/2})$ at $s=O(1)$ \\
\addlinespace[2pt]
\cite[Thm.~14]{hanni2024superposition} &
$\varepsilon^{(1)}=\Otilde(\max\{s\mu,\sqrt{s\varepsilon},\sqrt{s/d}\})$,
$\mu^{(1)}=\Otilde(\max\{\sqrt{1/d},\mu\})$ &
quoted in Sec.~\ref{sec:application} only to note it does not by itself preserve the
$\Otilde(d^{-1/2})$ invariant \\
\addlinespace[2pt]
\cite[Cor.~15]{hanni2024superposition} &
relaxes the near-sphere hypothesis of Thm.~14 at the cost of depth $1\to3$ &
not used for a capacity claim \\
\addlinespace[2pt]
\cite[Thm.~21]{hanni2024superposition} &
tolerance $\varepsilon<K\,d^{1/4}/(m^{1/2}s^{1/4})$, output
$\varepsilon^{(1)}=O(\log(d)\sqrt{s}/\sqrt d)$, $\mu^{(1)}=\Otilde(\sqrt s/\sqrt d)$ &
Lemma~\ref{lem:corr} with $m$ in the role of $F$ \\
\addlinespace[2pt]
\cite[Thm.~8]{hanni2024superposition} &
width $d=\Otilde(m^{2/3}s^2)$, depth $2L$ &
Proposition~\ref{prop:hanni}, inverted at $s=O(1)$ \\
\addlinespace[2pt]
\cite{adler2024superposition} &
$\Omega(\sqrt{m'}\log m')$ neurons and $\Omega(m'\log m')$ parameters; construction with
$O(\sqrt{m'}\log m')$ neurons; at most $O(n^2/\log n)$ features with $n$ neurons &
Sec.~\ref{sec:application}; the parameter-to-neuron conversion is flagged as an architectural
assumption of that model \\
\addlinespace[2pt]
\cite[Thm.~2.3]{strohmer2003grassmannian} &
$\max_{i\neq j}|\langle\phi_i,\phi_j\rangle|\ge\sqrt{(F-d)/(d(F-1))}$ for unit-norm systems &
the Gram case of Theorem~\ref{thm:floor}; the reduction is noted in Sec.~\ref{sec:frames} \\
\bottomrule
\end{tabular}
\end{table}

\section{Auxiliary proofs}
\label{app:proofs}

\subsection{Sub-Gaussian interference}

Let $u,u_1,\dots,u_m$ be independent uniform unit vectors in $\R^d$. Conditionally on $u$, each
$X_j=\langle u,u_j\rangle$ is distributed as the first coordinate of a uniform point on
$S^{d-1}$, hence has mean zero and is sub-Gaussian with variance proxy $C_0/d$ for a universal
$C_0$~\cite[Thm.~3.4.6]{vershynin2018hdp}. The $X_j$ are conditionally independent, so their
sum is sub-Gaussian with proxy $C_0m/d$ and
$\Pr\{|\sum_j X_j|>t\}\le2\exp(-c_0dt^2/m)$ with $c_0=1/(2C_0)$. This is the estimate used in
Theorem~\ref{thm:randomsupport}.

\subsection{The Bernoulli variant of the energy floor}

Under $b_i\sim\mathrm{Bernoulli}(p)$ independently,
$\E[bb^{\top}]=p(1-p)I_F+p^2\mathbf 1\mathbf 1^{\top}$, so
$\E_b\norm{Ab}^2=p(1-p)\norm{A}_F^2+p^2\norm{A\mathbf 1}^2\ge p(1-p)F(F-d)/d$ by
Theorem~\ref{thm:floor}. With $p=s/F$ and $s\le F/2$ this gives
$\frac1F\E_b\norm{Ab}^2\ge\frac{s(F-d)}{2dF}=\Omega(s/d)$, matching
Proposition~\ref{prop:uniform-energy} up to the $(F-1)$ versus $F$ denominator. Both variants
are computed by the released code and both are reported in E3.

\section{The nuclear-norm inequality}
\label{app:nuclear}

For $M\in\R^{F\times F}$ with singular value decomposition
$M=\sum_{k=1}^{r}\sigma_k u_k v_k^{\top}$,
\[
  |\operatorname{tr}(M)|=\Big|\sum_k\sigma_k\,u_k^{\top}v_k\Big|
  \le\sum_k\sigma_k|u_k^{\top}v_k|\le\sum_k\sigma_k=\nucnorm{M},
\]
and by Cauchy--Schwarz on $(\sigma_1,\dots,\sigma_r)$,
$\nucnorm{M}^2\le r\sum_k\sigma_k^2=r\norm{M}_F^2$. With $r\le d$ this gives
$|\operatorname{tr}(M)|^2\le d\norm{M}_F^2$, as used in Theorem~\ref{thm:floor}. Note that the
identity $\operatorname{tr}(M)=\sum_k\sigma_k$ holds only for positive semidefinite matrices and
is \emph{not} used.

\section{Reference scales}
\label{app:scales}

The scales appearing in this literature count different objects and should not be read as a
capacity ordering. Writing $F_{\mathrm{CS}}(d,\alpha)=d\,g(\alpha)$ with
$g(\alpha)=1/((1-\alpha)\ln(1/(1-\alpha)))$ for compressed-sensing-style linearly decodable
storage, $F^{\mathrm{H}}_{\mathrm{cert}}(d)=\Thetatilde(d^{3/2})$ for the H\"{a}nni template
certificate, $L_{\mathrm{AS}}(d)=d^2/\log^2 d$ and $U_{\mathrm{AS}}(d)=d^2/\log d$ for the two
sides of the Adler--Shavit bracket, and $N_{\mathrm{JL}}(d;\varepsilon)=\exp(\Theta(d
\varepsilon^2))$ for Johnson--Lindenstrauss packing~\cite{JL1984}, one has
\[
  F_{\mathrm{CS}}\ll F^{\mathrm{H}}_{\mathrm{cert}}\ll L_{\mathrm{AS}}
  \lesssim\FrecAS\lesssim U_{\mathrm{AS}}\ll N_{\mathrm{JL}}
\]
as $d\to\infty$. This does not mean that recursive computation beats storage:
$N_{\mathrm{JL}}$ counts passively packable states, $F_{\mathrm{CS}}$ counts linearly decodable
features through a compressed bottleneck, and the middle two count recursively computable
Boolean features under two different published frameworks. Michaud et
al.~\cite{michaud2025manifolds} give a complementary reason why an observed dictionary size may
fall below the template scale: if feature manifolds consume capacity, the relevant comparison
is the effective dimension rather than the raw feature count.

\section{Sparse-autoencoder context}
\label{app:sae}

As background only, \citet{borobia2026pruning} train TopK sparse
autoencoders ($k=64$, expansion $8d$) on three language models and observes no dead features at
$F=8d$, placing those dictionaries well below both the $d^{3/2}$ template scale and the
near-quadratic Adler--Shavit scale. Sparse-autoencoder dictionary size is a reconstructive
quantity and is not a measurement of recursive computation capacity; these numbers are not used
in any proof or experiment of the present paper, and are mentioned only to indicate where
current empirical dictionaries sit relative to the scales of Appendix~\ref{app:scales}.
Liu et al.~\cite{liu2025scaling} and Elhage et al.~\cite{elhage2022superposition} provide
further context on why strong superposition regimes are of interest.

\end{document}